%% file: neurips-main.tex
\documentclass{article}

    \PassOptionsToPackage{numbers, compress}{natbib}

    \usepackage[preprint]{neurips_2022}

\usepackage[utf8]{inputenc} 
\usepackage[T1]{fontenc}    
\usepackage{hyperref}       
\usepackage{url}            
\usepackage{booktabs}       
\usepackage{amsfonts}       
\usepackage{nicefrac}       
\usepackage{microtype}      
\usepackage{xcolor}         

\input{preamble}

\usepackage{multirow}
\usepackage{colortbl}

\title{Robust Distillation for Worst-class Performance}

\author{%
  Serena Wang \\
  Google Research \\
  Mountain View, CA, USA  \\
  University of California, Berkeley\\
  Berkeley, CA, USA  \\
  \texttt{serenawang@google.com} \\
   \And
   Harikrishna Narasimhan \\
   Google Research \\
   Mountain View, CA, USA \\
   \texttt{hnarasimhan@google.com} \\
   \And
   Yichen Zhou \\
   Google Research\\
   Mountain View, CA, USA  \\
   \texttt{yichenzhou@google.com} \\
   \And
   Sara Hooker \\
   Cohere For AI \\
   Palo Alto, CA, USA \\
   \texttt{sarahooker@cohere.com} \\
   \And
   Michal Lukasik \\
   Google Research \\
   New York, NY, USA  \\
   \texttt{mlukasik@google.com} \\
   \And
   Aditya Krishna Menon \\
   Google Research \\
   New York, NY, USA  \\
   \texttt{adityakmenon@google.com} \\
}

\begin{document}

\maketitle

\begin{abstract}
Knowledge distillation has proven to be an effective
technique in
improving the performance a student model using predictions
from a teacher model. 
However, recent work has shown that gains in average efficiency are not uniform across subgroups in the data, and in particular can often come at the cost of accuracy on rare subgroups and classes. To preserve strong performance across classes that may follow a long-tailed distribution, 
we develop distillation techniques that are tailored to
improve the student's worst-class performance. Specifically, we introduce robust optimization objectives in different combinations for the teacher and student, and further allow for training with any tradeoff between the overall accuracy and the robust worst-class objective. 
We show empirically that our  robust distillation techniques not only achieve better worst-class performance, but also lead to Pareto improvement in the tradeoff between overall performance and worst-class performance compared to other baseline methods.
Theoretically, we provide insights into what makes a good teacher when the goal is to train a robust student.
\vspace{-5pt}
\end{abstract}

\input{intro}
\input{prelims}
\input{distillation}

\input{algorithms}
\input{theory}
\input{experiments}

\newpage
\clearpage

\bibliographystyle{plainnat}
\bibliography{main}

\newpage
\clearpage
\appendix
\onecolumn
\input{appendix}

\end{document}

%% file: preamble.tex
\newcommand{\X}{\mathcal{X}}

\newcommand{\1}{\mathbf{1}}
\renewcommand{\>}{\rightarrow}
\newcommand{\E}{\mathbb{E}}
\newcommand{\R}{\mathbb{R}}

\renewcommand{\P}{\mathbb{P}}
\newcommand{\cL}{\mathcal{L}}

\newcommand{\cF}{\mathcal{F}}
\newcommand{\bc}{\mathbf{c}}
\newcommand{\boldf}{\mathbf{f}}
\newcommand{\bp}{\mathbf{p}}
\newcommand{\blambda}{{\lambda}}

\newcommand{\bV}{\mathbb{V}}
\newcommand{\cM}{\mathcal{M}}

\newcommand{\Argmax}[1]{\underset{#1}{\operatorname{argmax}}}
\newcommand{\Argmin}[1]{\underset{#1}{\operatorname{argmin}}}
\newcommand{\argmax}[1]{{\operatorname{argmax}}_{#1}}

\newcommand{\val}{\textrm{\textup{val}}}

\newcommand{\xent}{\textrm{\textup{xent}}}

\newcommand{\bal}{\textrm{\textup{bal}}}

\newcommand{\rob}{\textrm{\textup{rob}}}
\newcommand{\robd}{\textrm{\textup{rob-d}}}
\newcommand{\bald}{\textrm{\textup{bal-d}}}
\newcommand{\zo}{\textrm{\textup{0-1}}}
\newcommand{\std}{\textrm{\textup{std}}}
\newcommand{\stdd}{\textrm{\textup{std-d}}}
\newcommand{\tdf}{\textrm{\textup{tdf}}}
\newcommand{\tdfd}{\textrm{\textup{tdf-d}}}
\newcommand{\mar}{\textrm{\textup{mar}}}
\newcommand{\oh}{\textrm{\textup{oh}}}
\newcommand{\softmax}{\textrm{\textup{softmax}}}

\usepackage{enumitem}

\usepackage{algorithm}
\usepackage{algorithmic} 

\usepackage{amsmath}
\usepackage{amssymb}
\usepackage{mathtools}
\usepackage{amsthm}

\theoremstyle{plain}
\newtheorem{theorem}{Theorem}
\newtheorem*{theorem*}{Theorem}
\newtheorem{proposition}[theorem]{Proposition}
\newtheorem*{proposition*}{Proposition}
\newtheorem{lemma}[theorem]{Lemma}

\theoremstyle{definition}

\theoremstyle{remark}

\usepackage[textsize=tiny]{todonotes}

%% file: intro.tex
\section{Introduction}
\label{sec:intro}

Knowledge distillation, wherein one trains a teacher model and uses its predictions to train a student model
of similar or smaller capacity, has proven to be a powerful tool that improves efficiency while achieving state-of-the-art classification accuracies on a variety of problems \citep{Hinton:2015, Radosavovic:2018, pham2021meta,46642,NEURIPS2020_3f5ee243}. While originally devised for model compression, distillation has also been widely applied to improve the performance of 
fixed capacity models and in semi-supervised learning settings \citep{rusu2016policy, Furlanello2018BornAN, yang2019training, xie2020self}. 

Evaluating the trade-offs incurred by distillation techniques has largely focused on canonical measures such as average error over a given data distribution. Along this dimension, distillation has proven to be a remarkably effective 
technique, with the student achieving even better test performance than the teacher (e.g.\ \citet{xie2020self}). 
However, average error may differ significantly from the performance on individual subpopulations in the data, and thus may be an inadequate metric to optimize in real-world settings that also require good performance on different subgroups. For instance, subgroups may be defined by attributes such as country, language, or racial attributes \citep{Hardt:2016, zafar2017fairness, agarwal2018reductions}, in which case performance on the individual subgroups becomes a policy and fairness concern. 

The potential for mismatch between average error and worst-group error is further exacerbated by the fact that many real-world datasets exhibit imbalance in the number of training examples belonging to different subgroups. In a multi-class classification setting where each target label delineates a subgroup, many real-world datasets exhibit a \textit{long-tailed} label distribution \citep{VanHorn:2017,menon2020long,feldman2021does,d2021tale}. Recent work has shown that improved average error of the student model often comes at the expense of poorer performance on the tail classes \cite{lukasik2021teachers,du2021compressed}
or under-represented groups \citep{sagawa2020investigation, menon2020overparameterisation}, and model compression can amplify performance disparities across groups \cite{hooker2020characterising, xu-etal-2021-beyond}. Further hurting test performance is the possibility that the training data may have particularly poor representation of important subgroups that occur more frequently at deployment or test time.

To mitigate the disparity between average error and subgroup performance, a common approach is to train a model to achieve low worst-group test error. Modified objectives for robust optimization techniques have successfully achieved state-of-the-art worst-class performance with manageable computational overhead \citep{Sagawa2020Distributionally, sohoni2020no, narasimhan2021training}. However, an evaluation of these techniques has disproportionately focused on a standard single model training setting. 
In this work, we focus on applying robust optimization to a multi-step training setting typical of distillation. We take a wider view of the optimization process, and ask what combination of student and teacher objectives achieves low worst-group test error \emph{and} minimizes difference between average and worst case error. 

To integrate robust optimization objectives into the distillation procedure, we introduce a robust objective for either the teacher, the student, or both. We also quantify the trade-off between average error across all subgroups and the worst-group error by mapping out Pareto fronts for each distillation procedure.
Theoretically, we ask the question, \textit{what makes for a ``good'' teacher} when training a robust student? That is, what properties are important in a teacher so that a student achieves good worst-class performance?

In summary, we make the following contributions:
\begin{enumerate}[label=(\roman*),itemsep=0pt,topsep=0pt,leftmargin=16pt,nolistsep]
\item We present distillation techniques integrating distributionally robust optimization (DRO) algorithms with the goal of achieving a robust student, and explore different combinations of modifications to the teacher and student objectives. 
\item In both self-distillation and compression settings, we empirically demonstrate that the proposed robust distillation algorithms can train student models that yield better worst-class accuracies than the teacher and other recent baselines. We also demonstrate gains in the trade-off between overall and worst-class performance.
\item We derive robustness guarantees for the student under  different algorithmic design choices, and 
provide  insights into when the student yields better worst-class robustness than the teacher.
\end{enumerate}

\textbf{Related Work.}\
\textit{Fairness in distillation:}\
\citet{lukasik2021teachers} study how distillation impacts worst-group performance and observe that the errors that the teacher makes on smaller subgroups get amplified and transferred to the student. They also propose simple modifications to the student's objective to control the strength of the teacher's labels for different groups. In contrast, in this work we propose a more direct and theoretically-grounded procedure that
seeks to explicitly optimize for the student's worst-case error, and explore modifications to both the teacher and student objectives.

\textit{Worst-group robustness:}\ The goal of achieving good worst-case performance across subgroups can be framed as a 
(group) distributionally robust optimization (DRO) problem, 
and can be solved by iteratively updating costs on the individual groups and minimizing the resulting cost-weighted loss \citep{chen2017robust}. Recent variants of this approach have sought to avoid over-fitting through
group-specific regularization \citep{Sagawa2020Distributionally} or
margin-based losses \cite{narasimhan2021training, kini2021labelimbalanced},
to handle unknown subgroups \citep{sohoni2020no}, and to balance between average and worst-case performance \cite{piratla2021focus}. Other approaches include the use of Conditional Value at Risk for worst-class robustness \citep{xu2020class}.

Among these, the work that most closely relates to our paper
is the margin-based DRO algorithm \citep{narasimhan2021training}, 
which includes preliminary distillation experiments in support of training the teacher with
standard ERM and the student with a robust objective. 
However, this and other prior work have only explored modifications to the student loss \cite{lukasik2021teachers,narasimhan2021training}, while training the teacher using a standard procedure.
{Our robust distillation proposals build on this method,
but carry out a more extensive analysis, exploring different combinations of teacher-student objectives 
and different trade-offs between average and worst-class performance}. Additionally,
we provide robustness guarantees for the student, equip the DRO algorithms to achieve different trade-offs between overall and worst-case error, and provide a rigorous analysis of different design choices.

%% file: prelims.tex
\section{Problem Setup}
\label{sec:prelims}

We are interested in a multi-class classification problem with instance space $\X$
and output space $[m] = \{1, \ldots, m\}$. Let $D$ denote the underlying data distribution over $\X \times [m]$, and 
$D_{\X}$ denote the marginal distribution over $\X$.
Let $\Delta_m$ denote the $(m-1)$-dimensional probability simplex over $m$ classes. 
We define the
conditional-class probability as $\eta_y(x) = \P(Y=y|X=x)$ and the class priors $\pi_y = \P(Y=y)$.
Note that $\pi_y = \E_{X \sim D_\X}\left[ \eta_y(X) \right]$. 

\textbf{Learning objectives.}
Our goal is to learn a multiclass classifier $h: \X \> [m]$ that maps an instance $x \in \X$ to one of $m$ classes.
We will do so by first learning a scoring function $f: \X \> \R^m$ that assigns scores $[f_1(x), \ldots, f_m(x)] \in \R^m$  to a given instance $x$,
and construct the classifier by predicting the class with the highest score: $h(x) = \argmax{j \in [m]}\,f_j(x)$. We will 
denote a softmax transformation of $f$ by $\softmax_y(f(x)) = \frac{\exp( f_y(x) ) }{\sum_j \exp( f_j(x) )}$,
and use the notation  $\softmax_y(f(x)) \propto z_y$ to indicate that $\softmax_y(f(x)) = \frac{z_y}{\sum_{j=1}^m z_j}$.

We measure the efficacy of the scoring function $f$ using a loss function $\ell: [m] \times \R^m \> \R_+$
that assigns a penalty $\ell(y, z)$ for predicting score vector $z \in \R^m$ for true label $y$. Examples of loss functions include the 
0-1 loss: $\ell^\zo(y, z) = \1\left(z \ne \argmax{j}\,f_j(x) \right)$,
and the softmax cross-entropy loss: $\ell^\xent(y, z) = \textstyle -f_y(x) + \log\big( \sum_{j \in [m]} \exp\left( f_j(x) \right) \big)$.

\textit{Standard objective:} A standard machine learning goal entails minimizing the overall expected risk: 
\begin{equation}
L^\std(f) = \E\left[\ell(Y, f(X))\right].
\label{eq:std}
\end{equation}
%
\textit{Balanced objective:} In applications where the classes are severely imbalanced, i.e.,
the class priors $\pi_y$ are non-uniform and significantly skewed,
one may wish to instead optimize a \emph{balanced} version of the above objective,
where we average over the conditional loss for each class. Notice that  the
conditional loss for class $y$ is weighted by the inverse of its prior:
\allowdisplaybreaks
\begin{align}
    L^\bal(f) 
            &= \frac{1}{m}\sum_{y \in [m]} \E\left[\ell(y, f(X))\,|\, Y = y \right]
         ~=~ \frac{1}{m}\sum_{y \in [m]}  \frac{1}{\pi_y}\E_X\left[ \eta_y(X)\,\ell(y, f(X)) \right].
    \label{eq:balanced}
\end{align}
\textit{Robust objective:} A more stringent objective would be to focus on the worst-performing class, and 
minimize a \emph{robust} version of
 \eqref{eq:std} that computes
the worst among the $m$
conditional losses:
\begin{align}
    L^\rob(f) 
         &= \max_{y \in [m]} \frac{1}{\pi_y}\E\left[ \eta_y(X)\,\ell(y, f(X)) \right].
    \label{eq:robust}
\end{align}
In practice, focusing solely on either the average or the worst-case performance may not 
be an acceptable solution, and therefore, in this paper, we will additionally seek to characterize the
trade-off between the balanced and robust objectives. One way to achieve this trade-off is to minimize the robust objective, while constraining the balanced objective to
be within an acceptable range. This constrained optimization can be equivalently formulated as optimizing a convex combination of the balanced and robust objectives, for trade-off 
$\alpha \in [0,1]$:
\begin{align}
    L^\tdf(f)
         &= (1 - \alpha) L^\bal(f) + \alpha L^\rob(f).
    \label{eq:trade-off}
\end{align}
 A similar trade-off can also be specified between the standard and robust objectives.
%
To better understand the differences between the standard, balanced and robust objectives in \eqref{eq:std}--\eqref{eq:trade-off},
we look at the optimal scoring function for each given a cross-entropy loss:
\begin{theorem}[\textbf{Bayes-optimal scorers}]
\label{thm:bayes}
When $\ell$ is 
the cross-entropy loss $\ell^\xent$, 
the minimizers of \eqref{eq:std}--\eqref{eq:robust} over all measurable functions $f: \X \> \R^m$ are given
by:

\begin{tabular}{llll}
    \textit{(i)} $L^\std(f)$: &  $\softmax_y( f^{*}(x) ) \,=\, \eta_y(x)$ & \textit{(ii)} $L^\bal(f)$: &  $\softmax_y( f^{*}(x) ) \,\propto\, \frac{1}{\pi_y}\eta_y(x)$ \\
    \textit{(iii)} $L^\rob(f)$: & $\softmax_y( f^{*}(x) ) \,\propto\, \frac{\lambda_y}{\pi_y}\eta_y(x)$ & \textit{(iv)} $L^\tdf(f)$: & $\softmax_y( f^{*}(x) ) \,\propto\, \frac{(1 - \alpha)\frac{1}{m} + \alpha\lambda'_y}{\pi_y}\eta_y(x),$
\end{tabular}


for class-specific constants $\lambda, \lambda' \in \R^m_+$ that depend on distribution $D$.
\vspace{-5pt}
\end{theorem}

All proofs are provided in Appendix \ref{app:proofs}.
Interestingly, the optimal scorers for all four objectives involve a simple scaling of the conditional-class probabilities $\eta_y(x)$.


%% file: distillation.tex
\section{Distillation for Worst-class Performance}
\label{sec:distillation}

We adopt the common
practice of training both the teacher and student on the same dataset. Specifically, given a training sample $S = \{(x_1, y_1), \ldots, (x_n, y_n)\}$
drawn from $D$, we first train a teacher model $p^t: \X \> \Delta_m$, and use it to generate a student dataset
$S' = \{(x_1, p^t(x_1)), \ldots, (x_n, p^t(x_n))\}$
by replacing the original labels with the teacher's predictions. 
We then train a student scorer $f^s: \X \> [m]$ using the re-labeled dataset,
and use it to construct the final classifier.

In a typical setting, both the teacher and student
are trained to optimize a version of the standard objective in \eqref{eq:std}, i.e.,
the teacher is trained to minimize the average loss against the original training labels,
and the student is trained to minimize an average loss against the teacher's predictions:
\begin{align}
        \text{Teacher: } & \hat{L}^\std(f^t) = \frac{1}{n}\sum_{i=1}^n \ell\left( y_i, f^t(x_i) \right); \hspace{10pt}
        \text{Student: } ~ \hat{L}^\stdd(f^s) = \frac{1}{n}\sum_{i=1}^n \sum_{y=1}^m p_y^t(x_i)\, \ell\left(y , f(x_i) \right),\nonumber
    \\[-10pt]\label{eq:standard-objectives}
\end{align}
~\\[-13pt]
where $p^t(x) = \softmax(f^t(x))$. It is also common to have the student use a mixture of the teacher and one-hot labels. For concreteness, we consider a simpler distillation setup without this mixture, though extensions with this mixture would be straightforward to add.
This work takes a wider view and explores \textit{what combinations of student and teacher objectives} facilitate better worst-group performance for the student. Our experiments evaluate all \emph{nine} combinations of standard, balanced, and robust teacher objectives, paired with standard, balanced, and robust student objectives.

Given the choice of teacher objective, the student will either optimize a distilled version of the balanced objective in \eqref{eq:balanced}:
\begin{align}
    {\hat{L}^\bald(f^s)}
    &= \frac{1}{m}\sum_{y \in [m]} \frac{1}{\hat{\pi}^t_y}\frac{1}{n}\sum_{i=1}^n  p_y^t(x_i)\, \ell\left( y , f^s(x_i) \right),
    \label{eq:balanced-distilled-empirical}
\end{align}
or a distilled version of the robust objective in \eqref{eq:robust}:
\begin{align}
    {\hat{L}^\robd(f^s)}
    &= \max_{y \in [m]} \frac{1}{\hat{\pi}^t_y}\frac{1}{n}\sum_{i=1}^n  p_y^t(x_i)\, \ell\left( y , f^s(x_i) \right),
    \label{eq:robust-distilled-empirical}
\end{align}
In practice, the teacher's predictions may have a different marginal distribution from the underlying class priors, particularly when temperature scaling is applied to the teacher's logits to soften the predicted probabilities \cite{narasimhan2021training}.  To address this, in both  \eqref{eq:balanced-distilled-empirical} and \eqref{eq:robust-distilled-empirical} we have replaced the class priors $\pi_y$ with the marginal distribution 
 $\hat{\pi}^t_y = \frac{1}{n}\sum_{i=1}^n p^t_y(x_i)$ from the teacher's predictions. 


In addition to exploring the combination of objectives that facilitates better worst-group performance for the student, we evaluate a more flexible approach -- have both the teachers and the students trade-off between the balanced and robust objectives:
\begin{align}\label{eq:trade-off-distilled}
        \text{Teacher: } & \hat{L}^\tdf(f^t) = (1-\alpha^t)\hat{L}^\bal(f^t) + \alpha^t\hat{L}^\rob(f^t) \\
        \text{Student: } & \hat{L}^\tdfd(f^s) = (1-\alpha^s)\hat{L}^\bald(f^s) + \alpha^s\hat{L}^\robd(f^s),\nonumber
\end{align}
where $\hat{L}^\bal(f^t)$ and $\hat{L}^\rob(f^t)$ are the respective empirical estimates of
\eqref{eq:balanced} and \eqref{eq:robust} from the training sample, and $\alpha^t, \alpha^s \in [0,1]$
are the respective tradeoff parameters for the teacher and student. 
We are thus able to evaluate the Pareto-frontier of balanced and worst-case accuracies, obtained from  different combinations of the teachers and students, and 
trained with different trade-off parameters.

%% file: algorithms.tex
\section{Robust Distillation Algorithms}
\label{sec:algorithms}
The different objectives we consider -- standard, balanced and robust -- entail different loss objectives to ensure efficient optimization during training.
For example, while training
the standard teacher and student in \eqref{eq:standard-objectives},
we take $\ell$ to be the softmax cross-entropy loss,
and optimize it using SGD. For the balanced and robust models,
we employ the margin-based surrogates that we detail below, which have shown to be more effective
in training over-parameterized networks \citep{cao2019learning, menon2020long, kini2021labelimbalanced}. 
Across all objectives, at evaluation we take the loss $\ell$ 
in the student and teacher objectives
to be the 0-1 loss.

\textbf{Margin-based surrogate for balanced objective.}\
When the teacher or student model being trained is over-parameterized,
i.e., has sufficient capacity to correctly classify all examples in the training set,
the use of an outer weighting term 
in the objective
(such as the inverse class marginals in \eqref{eq:balanced-distilled-empirical})
can be ineffective. In other words, a model that yields zero training objective
would do so irrespective of what outer weights we choose. 
To remedy this problem,
we make use of the margin-based surrogate of \citet{menon2020long},
and incorporate the outer weights as margin terms within the loss. 
For
the balanced student objective in \eqref{eq:balanced-distilled-empirical}, this would look like:
\begin{align}
    {\widetilde{L}^\bald(f^s)}
    &= \frac{1}{n}\sum_{i=1}^n \cL^\mar\left(p^t(x_i), f^s(x_i); \1 / \hat{\pi}^t \right),
    \label{eq:margin-la}
\end{align}
~\\[-25pt]
\begin{align}
\text{where}~~\cL^\mar\left(\bp, \boldf; \bc \right) &=
\frac{1}{m}\sum_{y \in [m]} p_y \log\bigg(1 + \sum_{j \ne y}\exp\left(\log(c_y / c_{j}) \,-\, (f_y - f_j) \right) \bigg),
\nonumber
\end{align}
for teacher probabilities $\bp \in \Delta_m$, student scores $\boldf \in \R^m$,
and per-class costs $\bc \in \R^m_+$. 
 For the balanced teacher, the margin-based objective would take a similar form, but 
 with one-hot labels.

We include a proof in Appendix \ref{app:calibration-mar} showing that
a scoring function that minimizes this surrogate objective
also minimizes the
the balanced objective in \eqref{eq:balanced-distilled-empirical}
(when $\ell$ is the cross-entropy loss, and the student is chosen from a sufficiently flexible function class). In practice, the margin term $\log(c_y / c_{j})$ encourages a larger margin of separation for classes $y$ for which the cost $c_y$ is relatively higher.
 
\begin{figure}[t]
\vspace{-5pt}
\begin{algorithm}[H]
\caption{Distilled Margin-based DRO}
\label{algo:dro}
\begin{algorithmic}
\STATE \textbf{Inputs:} Teacher $p^t$, Student hypothesis class $\cF$, Training set $S$, Validation set $S^\val$, Step-size $\gamma \in \R_+$,
Number of iterations $K$, Loss $\ell$,
Initial student $f^0 \in \cF$, Initial multipliers $\blambda^0 \in \Delta_m$
\STATE Compute $\hat{\pi}^{t}_j = \frac{1}{n}\sum_{(x, y) \in S} p_j^t(x),~ \forall j \in [m]$
\STATE Compute $\hat{\pi}^{t,\val}_j = \frac{1}{n^\val}\sum_{(x, y) \in S^\val} p_j^t(x),~ \forall j \in [m]$
\STATE \textbf{For}~{$k = 0 $ to $K-1$}
\STATE ~~~$\tilde{\lambda}^{k+1}_j \,=\, \lambda^k_j\exp\big( \gamma \hat{R}_j \big), \forall j \in [m]$
~~~~~\text{where} $\hat{R}_j =$ $\frac{1}{n^\val}\frac{1}{\hat{\pi}^{t,\val}_j}\underset{(x, y) \in S^\val}{\sum} p_j^t(x)\, \ell( j , f^k(x) )$\\[-20pt]
\STATE ~~~$\lambda^{k+1}_y \,=\, \frac{\tilde{\lambda}^{k+1}_y}{\sum_{j=1}^m \tilde{\lambda}^{k+1}_j}, \forall y$
\STATE ~~~$f^{k+1} \,\in\, \Argmin{f \in \cF}\, \frac{1}{n}\sum_{i=1}^n \cL^\mar\left(p^t(x_i), f(x_i); \frac{\lambda^{k+1}}{\hat{\pi}^t} \right)$
~~~~~~// Replaced with a few steps of SGD
\STATE \textbf{End For}
\STATE \textbf{Output:} $\bar{f}^{s}: x \mapsto \frac{1}{K}\sum_{k =1}^K f^k(x)$
\end{algorithmic}
\end{algorithm}
\vspace{-22pt}
\end{figure}

\textbf{Margin-based DRO for robust objective.}\
Minimizing the robust objective with plain SGD can be difficult due to the presence of
the outer ``max'' over $m$ classes. The key difficulty is in computing reliable stochastic gradients for
the max objective, especially given a small batch size. The
standard approach is to instead use a (group) distributionally-robust optimization (DRO) procedure,
which comes in multiple flavors \cite{chen2017robust, Sagawa2020Distributionally, kini2021labelimbalanced}.
We employ  the margin-based variant of group DRO  \citep{narasimhan2021training} as it naturally extends the 
margin-based objective used in the balanced setting.

We illustrate below how this applies to the robust student objective in \eqref{eq:robust-distilled-empirical}. The procedure for the robust teacher is similar, but involves one-hot labels.
For a student hypothesis class $\cF$,
we
first re-write the minimization in \eqref{eq:robust-distilled-empirical} 
over $f \in \cF$ into an equivalent min-max optimization using per-class multipliers $\lambda \in \Delta_m$:
\vspace{-5pt}
\[
\min_{f \in \cF}\max_{\lambda \in \Delta_m} \sum_{y \in [m]}\frac{\lambda_y}{\hat{\pi}^t_y}\frac{1}{n}\sum_{i=1}^n  p_y^t(x_i)\, \ell\left( y , f(x_i) \right),
\]
and then
 maximize over $\lambda$ for fixed $f$, and minimize over $f$ for fixed $\lambda$:
\begin{align*}
    \lambda^{k+1}_y &\propto\, 
    \lambda^k_y\exp\bigg( \gamma\frac{1}{n\hat{\pi}^t_y}\sum_{i=1}^n  p_y^t(x_i)\, \ell\left( y , f^k(x_i) \right) \bigg), \forall y\\[-4pt]
    f^{k+1} &\in\, \Argmin{f \in \cF} \sum_{y \in [m]}\frac{\lambda^{k+1}_y}{n\hat{\pi}^t_y}\sum_{i=1}^n  p_y^t(x_i)\, \ell\left( y , f(x_i) \right),
\end{align*}
where $\gamma > 0$ is a step-size parameter. The updates on $\lambda$ 
implement exponentiated gradient (EG) ascent to maximize over the simplex \citep{shalev2011online}.

Following \citet{narasimhan2021training}, we make two modifications to the above updates 
when used to train over-parameterized networks that can fit the training set perfectly.
First, we perform the updates on $\lambda$ using a
small held-out validation set $S^\val \,=\, \{(x_1,y_1), \ldots, (x_{n^\val}, y_{n^\val})\}$, 
instead of the training set,
so that the $\lambda$s reflect how well the model generalizes out-of-sample.
Second, in keeping with the balanced objective, we modify the weighted objective
in the $f$-minimization step to include a margin-based surrogate. 
Algorithm \ref{algo:dro}
provides a summary of these steps
and returns a scorer
that averages over the $K$ iterates: $\bar{f}^s(x) = \frac{1}{K}\sum_{k=1}^K f^k(x)$.
While the averaging is needed for our theoretical analysis,
in practice, we find it sufficient to return the last scorer $f^{K}$. 
In Appendix \ref{app:dro-general-algo},
we describe how Algorithm \ref{algo:dro-general} can be easily modified to  trade-off
between the balanced and robust objectives, as shown in \eqref{eq:trade-off-distilled}.

\textbf{To distill the validation set or not?}\
The updates on $\lambda$
in Algorithm \ref{algo:dro} use a validation set labeled by the teacher. 
One could instead perform these updates with a curated validation set containing the original one-hot labels.
 Each of these choices presents different merits. The use of a teacher-labeled validation set is useful
 in many real world scenarios where labeled data is hard to obtain, 
 while unlabeled data abounds.
  In contrast, the use of one-hot validation labels, although more expensive to obtain,  may
make the student more immune to errors in the teacher's predictions, 
as the coefficients $\lambda$s are now based on an unbiased estimate of 
the student's performance on each class.
%
%
%
With a one-hot validation set, we update $\lambda$s as follows: 
\begin{equation}
    \lambda^{k+1}_j \,=\, \lambda^k_j\exp\big( \eta \hat{R}_j \big), \forall j \in [m], 
~~~~\text{where}~
\hat{R}_j = \frac{1}{n^\val}\frac{1}{\hat{\pi}_j}\sum_{(x, y) \in S^\val} \1(y=j)\,\ell( j , f^k(x) ),
\label{eq:one-hot-val}
\end{equation}
for estimates $\hat{\pi}_y \approx \pi_y$ of the original class priors.
%
We analyze both the variants in our experiments, and in the next section, discuss robustness guarantees for each.

%% file: theory.tex
\section{Theoretical Analysis}
\label{sec:theory}
To simplify our exposition, we will present  our theoretical analysis 
for a student trained using Algorithm \ref{algo:dro} to yield good worst-class performance. Our results easily extend to the case where the student seeks to trade-off between overall and worst-case performance. 

\textbf{What constitutes a good teacher?}
We would first like to understand what makes a good teacher 
when the student's goal is to minimize the robust population objective $L^\rob(f^s)$ in \eqref{eq:robust}. We also analyze whether the student's ability to perform well on this worst-case objective depends on
the teacher also performing well on the same objective.
As a proxy for $L^\rob(f^s)$, the
student minimizes the distilled objective $\hat{L}^\robd(f^s)$ in \eqref{eq:robust-distilled-empirical} 
with predictions from teacher $p^t$.
We argue that an \emph{ideal} teacher 
in this case
would be one that ensures that
the difference between the two objectives
$
|\hat{L}^\robd(f^s) - L^\rob(f^s)|
$
is as small as possible. Below, we provide a simple bound on this difference:
\begin{theorem}
\label{thm:good-teacher}
Suppose $\ell(y, z) \leq B, \forall x \in \X$ for some $B > 0$. 
Let $\pi^t_y = \E_x\left[ p_y^t(x) \right]$,
and let the following denote the per-class expected and empirical student losses respectively:
\begin{align*}
\phi_y(f^s) = \textstyle\frac{1}{\pi^t_y}\E_x\left[ p_y^t(x)\, \ell\left( y , f^s(x) \right)\right]; \quad \hat{\phi}_y(f^s) = \textstyle\frac{1}{\hat{\pi}^t_y}\frac{1}{n}\sum_{i=1}^n  p_y^t(x_i)\,\ell\left( y , f^s(x_i) \right).
\end{align*}
Then for teacher $p^t$ and student $f^s$:
\begin{align*}
|\hat{L}^\robd(f^s) - L^\rob(f^s)|
&\leq
\underbrace{
    B\max_{y \in [m]} \E_x\left[
        \left| \frac{p^t_y(x)}{\pi^t_y} - \frac{\eta_y(x)}{\pi_y} \right|\right]}_{\text{Approximation error}}
+
\underbrace{
\max_{y \in [m]} \big|\phi_y(f^s) - \hat{\phi}_y(f^s)\big|}_{\text{Estimation error}}.
\end{align*}
\end{theorem}
%
The \emph{approximation error} 
captures 
how well the teacher's predictions mimic the conditional-class distribution $\eta(x) \in \Delta_m$, 
up to per-class normalizations. 
This suggests that even if $p^t$ does not 
achieve good worst-class performance, as long as
it is well calibrated within each class (as measured by the approximation error above), 
it will serve as a good teacher. Indeed when the teacher outputs the conditional-class probabilities, 
i.e.\ $p^t(x) = \eta(x)$, the  approximation error is trivially zero (recall that the normalization term
$\pi^t_y = \pi_y$ in this case). We  know from
Theorem \ref{thm:bayes} that, this would be the case with a teacher
trained to optimize the standard cross-entropy objective (provided we use an unrestricted model class). 
In practice, however, we do not expect the teacher to approximate $\eta(x)$ very well,
and this opens the door for training the teacher with the other objectives described in Section \ref{sec:prelims}, each
of which encourage the teacher to approximate a scaled (normalized) version of $\eta(x)$.

The \emph{estimation error} captures how well the teacher aids in the student's out-of-sample generalization. 
The prior work by \citet{menon2021statistical} studies  this question in detail for 
the standard student objective, and 
provide a bound that depends on the variance induced by the teacher's predictions on the
student's objective: the lower the variance, the better the student's generalization.
 A similar analysis can be carried out with the per-class loss terms in Theorem \ref{thm:good-teacher} (more details in Appendix \ref{app:student-gen-bound}).


\textbf{Robustness guarantee for the student.}\
We next seek to understand if the student can match or outperform the teacher's worst-class performance.
For a fixed teacher $p^t$, we
 consider a \emph{self-distillation} setup where 
the student is chosen from the same function class $\cF$ as the teacher,
and can thus exactly mimic the teacher's predictions.
Under this setup, 
we provide robustness gurarantees
for the student output by Algorithm \ref{algo:dro} 
in terms of the approximation and estimation errors
described above. 

\begin{proposition}
\label{prop:student-form}
Suppose $p^t \in \cF$ and $\cF$ is closed under linear transformations. Let $\bar{\lambda}_y = (\prod_{k=1}^K {\lambda_y^k}/{\pi^t_y})^{1/K}, \forall y$.
Then the scoring function $\bar{f}^s(x) = \frac{1}{K} \sum_{k=1}^K f^{k}(x)$ output by Alg. \ref{algo:dro} is of the form:
    \[
        \softmax_j(\bar{f}^s(x)) \propto \bar{\lambda}_j p_j^t(x),~\forall j \in [m],\, \forall (x, y) \in S.
    \]
\end{proposition}
\begin{theorem}
\label{thm:dro}
Suppose $p^t \in \cF$ and $\cF$ is closed under linear transformations. 
Suppose
$\ell$ is the cross-entropy loss $\ell^{\xent}$,
$\ell(y, z) \leq B$
and $\max_{y \in [m]}\frac{1}{\pi^t_y} \leq Z$,
for some $B, Z > 0$. 
Furthermore, suppose for any $\delta \in (0,1)$, the following bound holds on the
estimation error in Theorem \ref{thm:good-teacher}:
with probability at least $1 - \delta$ (over draw of $S \sim D^n$),
$\forall f \in \cF$, $\textstyle
\max_{y \in [m]} \big|\phi_y(f) - \hat{\phi}_y(f)\big| \leq \Delta(n, \delta),$
for some $\Delta(n, \delta) \in \R_+$ that
is increasing in $1/\delta$, and goes to 0 as $n \> \infty$. 
Then when the step size $\gamma = \frac{1}{2BZ}\sqrt{\frac{\log(m)}{{K}}}$
and $n^\val \geq 8Z\log(2m/\delta)$, we have that with
probability at least $1-\delta$ (over draw of $S \sim D^n$ and $S^\val \sim D^{n^\val}$),
\begin{align*}
\lefteqn{
L^\rob(\bar{f}^s) \,\leq\,
    \min_{f \in \cF}L^\rob(f)}
\\[-10pt]
 &\hspace{1cm}
 +\,
        \underbrace{2\Delta(n^\val, \delta/2) \,+\,  2\Delta(n, \delta/2)}_{\text{Estimation error}}
\,+\,
\underbrace{
2B\max_{y \in [m]} \E_x\left[
        \left| \frac{p^t_y(x)}{\pi^t_y} \,-\, \frac{\eta_y(x)}{\pi_y} \right|\right]}_{\text{Approximation error}}
        \,+\,
        \underbrace{4BZ\sqrt{\frac{\log(m)}{{K}}}}_{\text{EG convergence}}.
\end{align*}
\vspace{-10pt}
\end{theorem}


Proposition \ref{prop:student-form} shows the student 
not only learns to mimic the teacher on the training set, but makes per-class adjustments to its predictions,
and Theorem \ref{thm:dro} shows that these adjustments are chosen to close-in on the
gap to the optimal robust scorer in $\cF$.
The form of the student
 suggests that it can not only match the teacher's performance,
but can potentially improve upon it by making 
adjustments to its scores. 
However, the student's convergence to the optimal scorer in $\cF$
would still be limited by the teacher's approximation error:
even when the sample sizes and number of iterations $n, n^\val, K \> \infty$,
the student's optimality gap may still be non-zero as long as the teacher is sub-optimal. 

\textbf{Connection to post-hoc adjustment.}\
The form of the student 
in Proposition \ref{prop:student-form} raises an interesting question. Instead of training an explicit student model, 
why not directly construct a new scoring model by making post-hoc adjustments
to the teacher's predictions? Specifically, one could optimize over functions of the form $f^s_y(x) = \log(\gamma_y p^t_y(x)),$ where the teacher $p^t$ is fixed, and pick 
the coefficients $\gamma \in \R^m$ so that resulting scoring function yields the best worst-class accuracy on a held-out dataset.  
This simple \emph{post-hoc adjustment} strategy 
may not be feasible if the goal is to distill to a student that is considerably smaller than the teacher. Often, this is the case in settings where distillation is used as a compression technique.
Yet, this post-hoc method 
serves as good baseline to compare with.

\textbf{One-hot validation labels.}\
If Algorithm \ref{algo:dro} used one-hot labels in the validation set instead of the teacher generated labels (as prescribed in \eqref{eq:one-hot-val}),
the form of the student learned remains the same as 
 in Theorem \ref{thm:dro}. However, the coefficients $\bar{\lambda}$ used to make
adjustments to the teacher's predictions enjoy a slightly different guarantee.
As shown in Appendix \ref{app:one-hot-vali}, the approximation error bound now has a weaker dependence on the teacher's predictions (and hence is more immune to the teacher's errors),
while the estimation error bound incurs slower convergence with increase in sample size. 

%% file: experiments.tex
\vspace{-3pt}
\section{Experiments}
\label{sec:expts}
\begin{table*}[t]
\caption{Comparison of \emph{self-distilled} teacher/student combos on test. Worst-class accuracy shown above, standard accuracy shown in parentheses for the top three datasets, and balanced accuracy shown in parenthesis for the bottom two long-tail datasets. Note that $L^{\std} = L^{\bal}$ for the top three datasets with balanced class priors; so here we do not include combinations involving a balanced teacher/student. The combination with the best worst-class accuracy is \textbf{bolded}. Mean and standard error are reported over repeats (10 repeats for CIFAR*, 5 repeats for TinyImageNet). We include results for the robust student using either a teacher labeled validation set (``teacher val''), or true one-hot class labels in the validation set (``one-hot val''), as outlined in Eq.~(\ref{eq:one-hot-val}). We include results using a smaller student architecture (ResNet-56 $\to$ ResNet-32), additional ImageNet results, and additional comparisons to group DRO and the Ada* methods in Appendix \ref{app:experiment_details}.
}
\label{tab:combos}
\begin{center}
\begin{small}
\begin{tabular}{p{0.1cm}c||c|c||c|c||c|c||}
\toprule
& & \multicolumn{2}{c||}{\textbf{CIFAR-10} Teacher Obj.} & \multicolumn{2}{c||}{\textbf{CIFAR-100} Teacher Obj.} & \multicolumn{2}{c||}{\textbf{TinyImageNet} Teacher Obj.} \\
& & $L^{\std}$ & $L^{\rob}$ & $L^{\std}$ & $L^{\rob}$ & $L^{\std}$ & $L^{\rob}$ \\
\midrule
\multirow{10}{*}{\rotatebox{90}{Student Obj.}} 
& none & $86.48$ \tiny{$\pm 0.32$}  & $90.09 $ \tiny{$\pm 0.22$} &  $42.22 $ \tiny{$\pm 0.90$}  & $43.42 $ \tiny{$\pm 1.03$}  & $10.77 $ \tiny{$\pm 2.30$}  & $16.30 $ \tiny{$\pm 1.81$}\\
& & \tiny{($93.74 \pm 0.05$)} & \tiny{($92.67 \pm 0.09$)} & \tiny{$72.42 \pm 0.16$} & \tiny{$68.81 \pm 0.11$} & \tiny{($58.64 \pm 0.16$)} & \tiny{($50.52 \pm 0.22$)}\\
\cline{2-8}
& Post & $88.60 $ \tiny{$\pm 0.35$}  & $87.95 $ \tiny{$\pm 0.42$} & $38.19 $ \tiny{$\pm 1.25$}  &  $37.92 $ \tiny{$\pm 0.94$} & $11.22 $ \tiny{$\pm 0.38$}  & $13.58 $ \tiny{$\pm 0.40$}\\
& shift & \tiny{$92.16 \pm 0.18$} & \tiny{$91.21 \pm 0.18$} & \tiny{$61.22 \pm 1.15$} & \tiny{$61.84 \pm 0.93$} & \tiny{$43.63 \pm 0.65$} & \tiny{$43.23 \pm 0.83$}\\
\cline{2-8}
& $L^{\stdd}$ & $87.66 $ \tiny{$\pm 0.40$}  & $90.12 $ \tiny{$\pm 0.23$}  & $43.81\pm0.58$  & \cellcolor{blue!25}$\mathbf{45.33}$ \tiny{$\pm0.82$} & $4.21 $ \tiny{$\pm 1.76$}  & $10.53 $ \tiny{$\pm 1.48$}\\
& & \tiny{($94.34 \pm 0.07$)} & \tiny{($94.07 \pm 0.07$)} & \tiny{$74.61\pm0.15$} & \cellcolor{blue!25}\tiny{$73.67\pm0.05$} & \tiny{($59.66 \pm 0.18$)} & \tiny{($56.55 \pm 0.15$)}\\
\cline{2-8}
& $L^{\robd}$  & \cellcolor{blue!25} $\mathbf{90.94} $ \tiny{$\pm 0.16$} & $85.14 $ \tiny{$\pm 0.47$} & $42.96$ \tiny{$\pm0.99$}  & $27.59\pm0.86$ & ${17.75} $ \tiny{$\pm 1.19$} & $6.00 $ \tiny{$\pm 1.65$}\\
&\tiny{(teacher val)} & \cellcolor{blue!25}\tiny{($92.54 \pm 0.05$)} & \tiny{($89.58 \pm 0.11$)} & \tiny{($68.71\pm0.15$)} & \tiny{($54.79\pm0.23$)} & \tiny{($47.81 \pm 0.13$)} & \tiny{($39.49 \pm 0.14$)}\\
\cline{2-8}
& $L^{\robd}$ & $89.37 $ \tiny{$\pm 0.17$} & $87.32 $ \tiny{$\pm 0.21$} & $40.36$ \tiny{$\pm 0.72$}  & $42.68\pm0.74$ & $14.70 $ \tiny{$\pm 1.65$} & $16.16 $ \tiny{$\pm 1.42$}\\
&\tiny{(one-hot val)} & \tiny{($91.63 \pm 0.06$)} & \tiny{($91.16 \pm 0.10$)} & \tiny{($61.49\pm0.22$)} & \tiny{($62.03\pm0.24$)} & \tiny{($50.16 \pm 0.18$)} & \tiny{($44.39 \pm 0.23$)} \\
\bottomrule
\end{tabular}

\begin{tabular}{p{0.1cm}c||c|c|c||c|c|c||}
\toprule
& & \multicolumn{3}{c||}{\textbf{CIFAR-10-LT} Teacher Obj.} & \multicolumn{3}{c||}{\textbf{CIFAR-100-LT} Teacher Obj.} \\
& & $L^{\std}$ & $L^{\bal}$ & $L^{\rob}$ & $L^{\std}$ & $L^{\bal}$ & $L^{\rob}$ \\
\midrule
\multirow{14}{*}{\rotatebox{90}{Student Obj.}} 
& none & $57.26 $ \tiny{$\pm 0.55$} & $68.52 $ \tiny{$\pm 0.52$} & $74.8 $ \tiny{$\pm 0.30$} &   $0.00 $ \tiny{$\pm 0.00$} & $3.75 $ \tiny{$\pm 0.62$} & $10.33 $ \tiny{$\pm 0.82$} \\
&& \tiny{($76.27 \pm 0.20$)} & \tiny{($79.85 \pm 0.20$)} & \tiny{($80.29 \pm 0.12$)}  & \tiny{($43.33 \pm 0.16$)} & \tiny{($47.55 \pm 0.17$)} & \tiny{($44.27 \pm 0.13$)}\\
\cline{2-8}
& Post & $74.33 $ \tiny{$\pm 0.27$} & $73.60 $ \tiny{$\pm 0.46$} & $73.93 $ \tiny{$\pm 0.35$} &   $6.28 $ \tiny{$\pm 0.84$} & $8.89 $ \tiny{$\pm 0.69$} & $10.01 $ \tiny{$\pm 0.72$} \\
& shift & \tiny{($78.28 \pm 0.15$)}  & \tiny{($77.92 \pm 0.19$)} & \tiny{($77.93 \pm 0.12$)} &  \tiny{($27.93 \pm 0.46$)} & \tiny{($28.70 \pm 0.38$)} & \tiny{($29.88 \pm 0.61$)}\\
\cline{2-8}
& Ada & $47.52 $ \tiny{$\pm0.95$} & $66.74 $ \tiny{$\pm0.35$} & $70.33 $ \tiny{$\pm0.50$} & $0.00  $ \tiny{$\pm 0.00$} & $0.00  $ \tiny{$\pm 0.00$} & $12.46  $ \tiny{$\pm 0.36$} \\
& Margin & \tiny{($72.69\pm0.24$)} & \tiny{($78.20\pm0.09$)} & \tiny{($78.87\pm0.12$)}  & \tiny{($31.26 \pm 0.21$)} & \tiny{($34.06 \pm 0.12$)} & \tiny{($42.90 \pm 0.07$)}\\
\cline{2-8}
& $L^{\stdd}$ & $36.67  $ \tiny{$\pm 0.28$} & $66.96  $ \tiny{$\pm 0.43$} & $71.15  $ \tiny{$\pm 0.24$} & $0.00  $ \tiny{$\pm 0.00$} & $2.39  $ \tiny{$\pm 0.24$} & $7.32  $ \tiny{$\pm 0.47$} \\
&&\tiny{($69.5 \pm 0.13$)} & \tiny{($79.25 \pm 0.10$)} & \tiny{($80.95 \pm 0.11$)}  &\tiny{($43.86 \pm 0.14$)} & \tiny{($48.95 \pm 0.15$)} & \tiny{($47.93 \pm 0.11$)}\\
\cline{2-8}
& $L^{\bald}$ & $71.23  $ \tiny{$\pm 0.44$} & $70.52  $ \tiny{$\pm 0.20$} & $72.96  $ \tiny{$\pm 0.53$} & $4.39  $ \tiny{$\pm 0.65$} & $7.08  $ \tiny{$\pm 0.80$} & $7.19  $ \tiny{$\pm 0.79$} \\
&& \tiny{($80.5 \pm 0.12$)} & \tiny{($81.12 \pm 0.08$)} & \tiny{($80.71 \pm 0.07$)} & \tiny{($50.4 \pm 0.11$)} & \tiny{($50.1 \pm 0.09$)} & \tiny{($47.51\pm 0.20$)}\\
\cline{2-8}
& $L^{\robd}$  & $63.85  $ \tiny{$\pm 0.21$}  & \cellcolor{blue!25}$\mathbf{75.56} $ \tiny{$\pm 0.19$} & $69.21  $ \tiny{$\pm 0.45$}  & $9.05  $ \tiny{$\pm 0.71$}  & $12.52  $ \tiny{$\pm 0.98$} & $10.32  $ \tiny{$\pm 0.76$} \\
& \tiny{(teacher val)} & \tiny{($76.81 \pm 0.08$)} & \cellcolor{blue!25}\tiny{($80.81\pm 0.08$)} & \tiny{($76.72 \pm 0.19$)} & \tiny{($33.75 \pm 0.10$)} & \tiny{($34.05 \pm 0.09$)} & \tiny{($36.83 \pm 0.15$)}\\
\cline{2-8}
& $L^{\robd}$  & $73.59  $ \tiny{$\pm 0.25$} & $75.43  $ \tiny{$\pm 0.38$} & $74.7  $ \tiny{$\pm 0.19$} & $12.28  $ \tiny{$\pm 0.46$} & $11.94  $ \tiny{$\pm 0.80$} & \cellcolor{blue!25}$\mathbf{13.18}  $ \tiny{$\pm 0.61$} \\
& \tiny{(one-hot val)} & \tiny{($77.92 \pm 0.05$)} & \tiny{($79.02 \pm 0.07$)} & \tiny{($77.99 \pm 0.10$} & \tiny{($30.79 \pm 0.18$)} & \tiny{($29.8 \pm 0.20$)} & \cellcolor{blue!25}\tiny{($31.88 \pm 0.20$}\\
\bottomrule
\end{tabular}
\end{small}
\end{center}
\vskip -0.2in
\end{table*}

\textbf{Datasets.} We evaluate each robust distillation objective across different image dataset benchmarks: (i) CIFAR-10, (ii) CIFAR-100 \cite{Krizhevsky09learningmultiple}, and (iii) TinyImageNet  (a subset of ImageNet with 200 classes; \citet{le2015tiny}). We also include long tailed versions of each dataset \cite{cui2019class}.
Details on sampling the long tailed versions of the datasets 
and additional results on the full ImageNet dataset \citep{ILSVRC15} are given in Appendix \ref{app:experiment_details}.
For all datasets, we randomly split the original default test set in half to create our validation set and test set. We use the same validation and test sets for the long-tailed training sets as we do for the original versions, following the convention in prior work \cite{menon2020long, narasimhan2021training}. 

\textbf{Architectures.}  We evaluate our distillation protocols in both a self-distillation and compression setting. On all CIFAR datasets, all teachers were trained with the ResNet-56 architecture and students were trained with either ResNet-56 or ResNet-32. 
On TinyImageNet, teachers and students were trained with ResNet-18.
On ImageNet, teachers and students were trained with ResNet-34 and ResNet-18. 
This is as done by \citet{lukasik2021teachers} and \citet{he2016deep}, where more details on these architectures can be found (see, e.g., Table 7 in \citet{lukasik2021teachers}).

\textbf{Hyperparameters.} 
We apply temperature scaling to the teacher score distributions, i.e., compute $p^t(x) = \softmax(f^t(x) / \gamma)$, and  vary the temperature parameter $\gamma$  over a range of $\{1, 3, 5\}$.
Unless otherwise specified, the temperature hyperparameters were chosen to achieve the best worst-class accuracy on the validation set. A higher temperature produces a softer probability distribution over classes \citep{Hinton2015DistillingTK}. 
When teacher labels are applied to the validation set (see Section \ref{sec:algorithms}), we additionally include a temperature of 0.1 to approximate a hard thresholding of the teacher probabilities. We closely mimic the learning rate and regularization settings from prior work \cite{menon2020long,narasimhan2021training}. In keeping with the theory, the regularization ensures that the losses are bounded. 
See Appendix \ref{app:experiment_details} for further details.

\textbf{Baselines.}
We compare the robust distillation objectives with the following: 
\textit{(i)} \textbf{No distillation:} Models trained without distillation using each objective: $L^{\std}$, $L^{\bal}$, and $L^{\rob}$. We also include a comparison to group DRO \cite{Sagawa2020Distributionally} without distillation in Appendix \ref{app:experiment_details} which differs from our robust objective $L^{\rob}$ in that we apply a margin-based loss with a validation set.
\textit{(ii)} \textbf{Standard distillation without robustness:} Standard distillation protocol where both the teacher and student are trained with $L^{\std}$, $L^{\stdd}$, respectively.
\textit{(iii)} \textbf{Post-shifting:} 
Following Section \ref{sec:theory}, we evaluate a post-shift approach
that directly constructs a new scoring model by making post-hoc adjustments to the teacher,
so as to maximize the robust accuracy on the validation sample \citep{narasimhan2021training}.
\textit{(iv)} \textbf{AdaMargin and AdaAlpha \citep{lukasik2021teachers}:}
Both Ada* techniques are motivated by the observation that the margin defined for each class $y$ by $\gamma_{\rm avg}( y, p^{\rm t}( x ) ) = p^{\rm t}_y( x ) - \frac{1}{m - 1} \sum_{y' \neq y} p^{\rm t}_{y'}( x )$ 
correlates with whether distillation improves over one-hot training \cite{lukasik2021teachers}. 
AdaMargin uses that quantity as a margin in the distillation loss, whereas AdaAlpha uses it to adaptively mix between the one-hot and distillation losses.

\textbf{Which teacher/student combo is most robust?} 
For each objective and baseline,
we report the \textit{standard accuracy} over the test set (see \eqref{eq:std}), as well as the \textit{worst-class accuracy}, which we define to be the minimum per-class recall over all classes (see \eqref{eq:robust}). 
For the long-tail datasets, we follow the convention in 
\citet{menon2020long} of reporting \textit{balanced accuracy} (see \eqref{eq:balanced}) instead of standard
accuracy. 
In each case, we use the 0-1 loss for evaluation.
Table \ref{tab:combos} shows results for various combinations of the proposed robust objectives and baselines.
Across all datasets, the combination with best worst-class objective had at least one of either the teacher or the student apply the robust objective ($L^{\rob}$ or $L^{\robd}$). These combinations also achieve higher worst-class accuracy compared to the post-shift and AdaMargin techniques; although pairing these techniques with robust teachers could be competitive. Interestingly, post-shift was often seen to over-fit to the validation set, resulting in poorer test performance. 
As another comparison point, prior work by \cite{narasimhan2021training} also perform distillation with a standard teacher ($L^{\std}$) and robust student ($L^{\rob}$) with true one-hot labels on the validation set. In comparison to this, the new set of proposed robust combinations still achieves gains.

An inspection of the first row of Table \ref{tab:combos} reveals counter-intuitively that the teacher's worst-class accuracy is not a direct predictor of the robustness of a subsequent student. This couples with our theoretical understanding in Section \ref{sec:theory}, which showed that the ability of a teacher to train good students is determined by the calibration of scores within each class.
%
Perhaps surprisingly, it did not always benefit the robust student ($L^{\robd}$) to utilize the true one-hot labels in the validation set. Instead, training the robust student with teacher labels on the validation set was sufficient to achieve the best worst-class performance. This is promising from a data efficiency standpoint, since it can be expensive to build up a labeled dataset for validation, especially if the training data is long-tailed.

\begin{figure*}[t]
\vskip -0.2in
\begin{center}
\begin{tabular}{ccc}
    \includegraphics[trim={2.8cm 0 2.8cm 0},clip,width=0.11\textwidth]{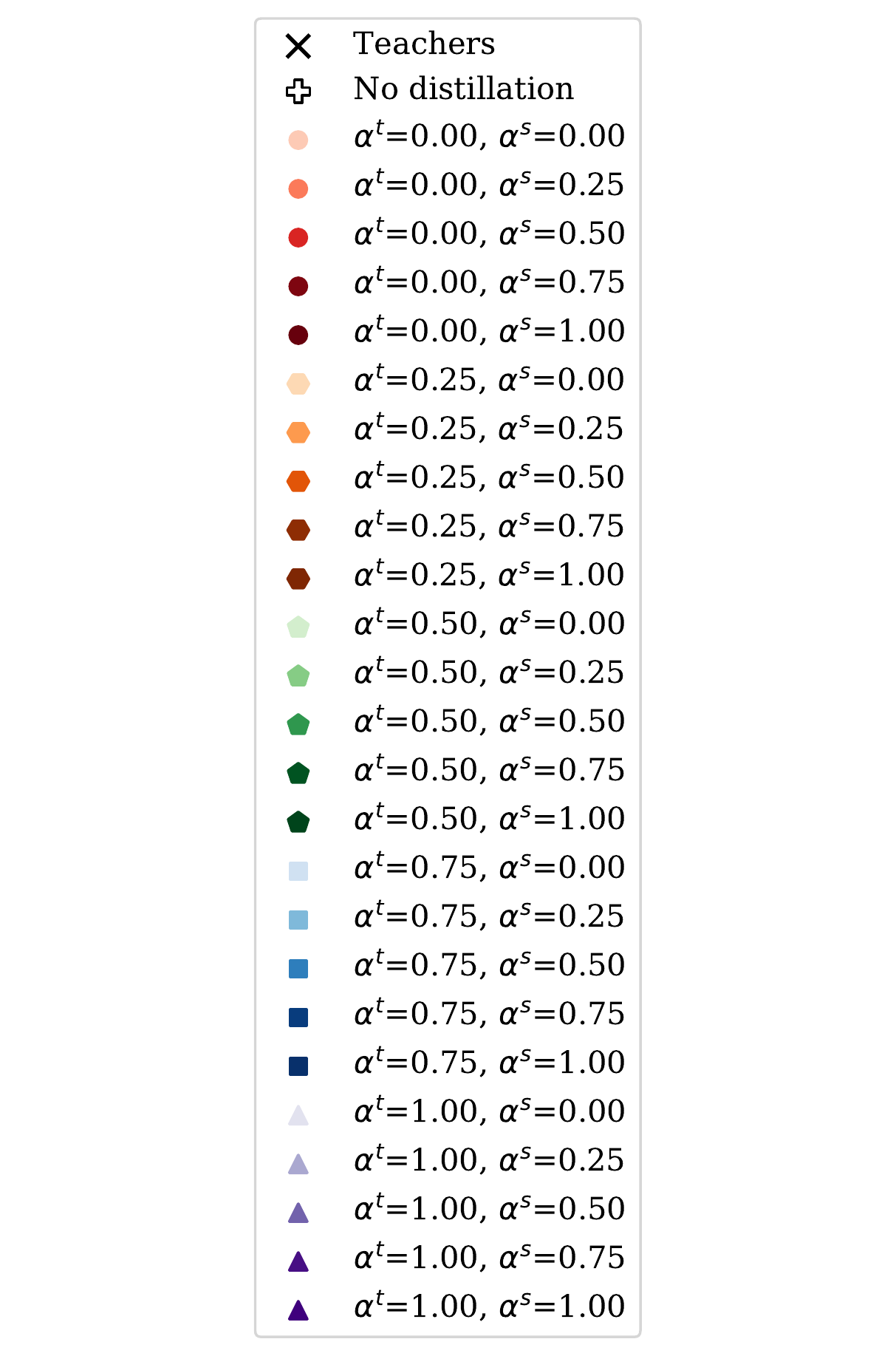} & \includegraphics[trim={0 0.2cm 0 0.1cm},clip,width=0.4\textwidth]{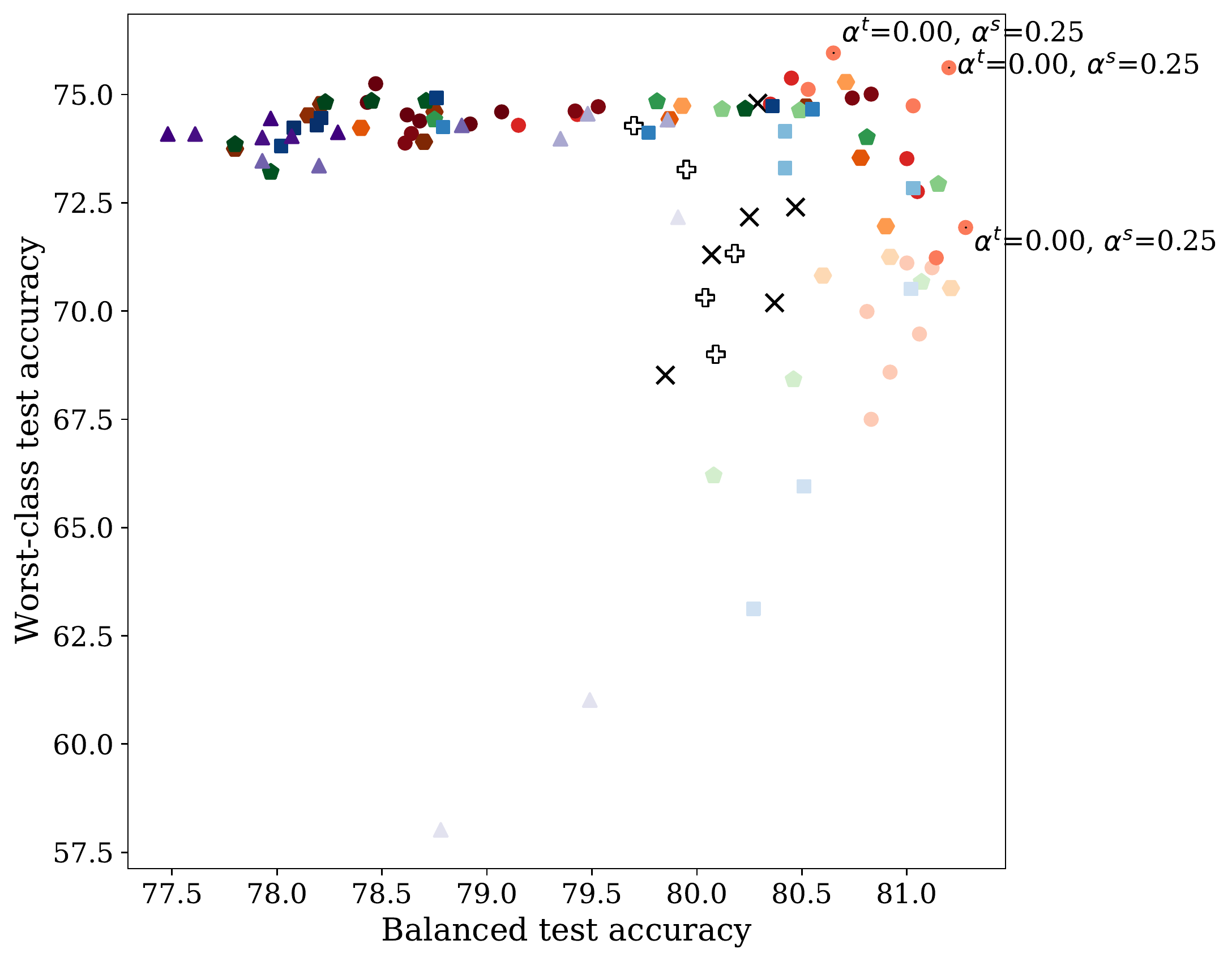} & \includegraphics[trim={0 0.2cm 0 0.1cm},clip,width=0.415\textwidth]{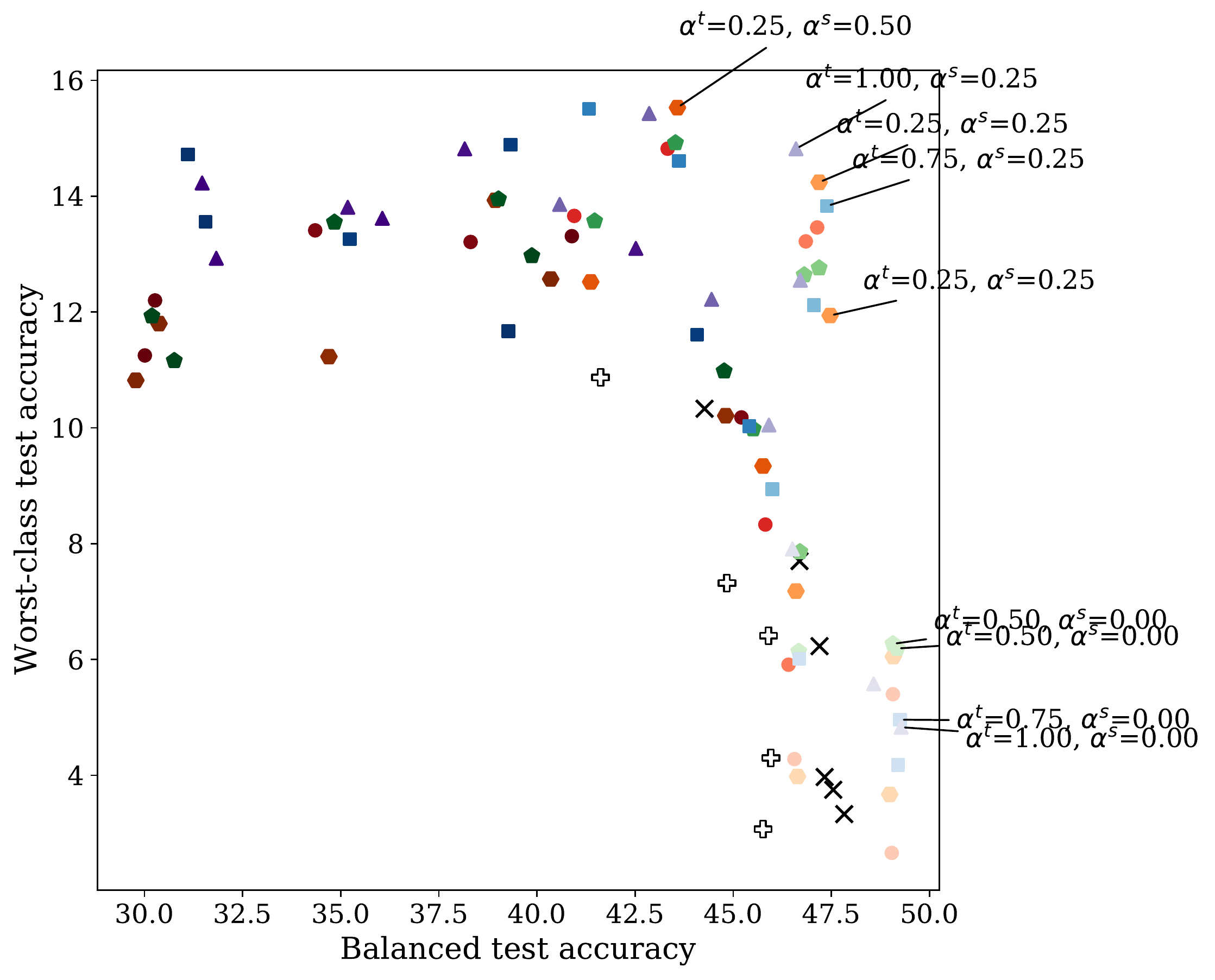}
\end{tabular}
\vspace{-5pt}
\caption{Trade-offs in worst-class test accuracy vs. balanced test accuracy for CIFAR-10-LT \textit{(left)} and CIFAR-100-LT \textit{(right)}, with ResNet-56 teachers and ResNet-32 students trained with one-hot validation labels. 
Robust teacher/student combinations produce points that are more Pareto efficient than the ResNet-56 teachers alone. Additional plots in Appendix \ref{app:tradeoff_plots} show similar results for self distillation and teacher validation labels. Mean test performance over 10 runs is shown.
}
\label{fig:tradeoffs_cf-lt_32_onehot}
\end{center}
\vskip -0.3in
\end{figure*}

\textbf{Trading off balanced vs. worst-class accuracy.} We also evaluate the robust objectives in a setting where distillation is used for efficiency gains by distilling a larger teacher into a smaller student.
In Figure \ref{fig:tradeoffs_cf-lt_32_onehot}, we plot the full trade-off between balanced accuracy and worst-class accuracy for the robust distillation protocols. Each point represents the outcome of a given combination of teacher and student objectives, where the teacher optimizes $L^{\tdf}$ with tradeoff parameter $\alpha^t$, and the student optimizes $L^{\tdfd}$ with tradeoff parameter $\alpha^s$.  
Strikingly, Figure \ref{fig:tradeoffs_cf-lt_32_onehot} shows that ResNet-32 students distilled with robust trade-offs can be more Pareto efficient than even the larger ResNet-56 teacher models. Thus, distillation with combinations of robust losses not only helps worst-case accuracy, but also achieves better trade-offs with balanced accuracy. Similar trends prevail across our experiment setups, including self distillation and the original non-long-tailed datasets (see Appendix \ref{app:experiment_details}).

Overall, we demonstrate empirically and theoretically the value of applying different combinations of teacher-student objectives, not only for improving worst-class accuracy, but also to achieve efficient trade-offs between average and worst-class accuracy.  
Future avenues for exploration include experiments with mixtures of label types and effects of properties of the data on robustness.

%% file: appendix.tex
\section{Proofs}
\label{app:proofs}

\subsection{Proof of Theorem \ref{thm:bayes}}
\label{app:proof-bayes}
(i) The first result follows from the fact that the cross-entropy loss is a proper composite loss \citep{williamson2016composite} with the softmax function as the  associated (inverse) link function.

(ii) For a proof of the second result, please see \citet{menon2020long}.

(iii) Below, we provide a proof for the third result. 

The minimization of the robust objective in \eqref{eq:robust} over $f$ can be re-written as
a min-max optimization problem:
\begin{align}
\min_{f: \X \> \R^m}\,L^\rob(f) = \min_{f: \X \> \R^m}\,\max_{\lambda \in \Delta_m} 
\underbrace{
\sum_{y=1}^m \frac{\lambda_y}{\pi_y}\E\left[ \eta_y(X)\,\ell(y, f(X)) \right]}_{\omega(\lambda, f)}.
\label{eq:min-max}
\end{align}

The min-max objective $\omega(\lambda, f)$ is clearly linear in $\lambda$ (for fixed $f$) and 
with $\ell$ chosen to be the cross-entropy loss, is convex in $f$ (for fixed $\lambda$), i.e.,
$\omega(\lambda, \kappa f_1 + (1-\kappa) f_2) \leq \kappa\omega(\lambda,  f_1) + (1-\kappa) \omega(\lambda, f_2), \,\forall f_1, f_2:\X\>\R^m, \kappa\in [0,1]$. Furthermore, $\Delta_m$ is a convex compact set, while the domain of $f$ is convex. It follows 
from Sion's minimax theorem \citep{sion1958general} that:
\begin{align}
\min_{f: \X \> \R^m}\max_{\lambda \in \Delta_m}\,\omega(\lambda, f)
&=\, \max_{\lambda \in \Delta_m}\min_{f: \X \> \R^m}\,\omega(\lambda, f).
\label{eq:min-max-swap}
\end{align}

Let $(\lambda^*, f^*)$ be such that:
\[
\lambda^* \in \Argmax{\lambda \in \Delta_m}\min_{f: \X \> \R^m}\,\omega(\lambda, f);~~~~~
f^* \in \Argmin{f: \X \> \R^m}\max_{\lambda \in \Delta_m}\,\omega(\lambda, f),
\]
where for any fixed $\lambda \in \Delta_m$, owing to the use of the cross-entropy loss, a minimizer $f^*$ 
always exists for $\omega(\lambda, f)$, and is given by $f^*_y(x) = \log\left(\frac{\lambda_y}{\pi_y}\eta_y(x)\right) + C,$ for some $C\in \R$.

We then have from \eqref{eq:min-max-swap}:
\begin{align*}
\omega(\lambda^*, f^*) &\leq\, \max_{\lambda \in \Delta_m}\, \omega(\lambda, f^*)\\
&=\, \min_{f: \X \> \R^m}\max_{\lambda \in \Delta_m}\,\omega(\lambda, f)
\,=\, \max_{\lambda \in \Delta_m}\min_{f: \X \> \R^m}\,\omega(\lambda, f)\\
&=\,\min_{f: \X \> \R^m}\,\omega(\lambda^*, f)
\,\leq\, \omega(\lambda^*, f^*),
\end{align*}
which tells us that there exists $(\lambda^*, f^*)$ is a saddle-point for
\eqref{eq:min-max}, i.e.,
\begin{align*}
\omega(\lambda^*, f^*) &= \max_{\lambda \in \Delta_m}\, \omega(\lambda, f^*) \,=\, \min_{f: \X \> \R^m}\, \omega(\lambda^*, f).
\label{eq:saddle-point}
\end{align*}
Consequently, we have:
\begin{align*}
L^\rob(f^*) &=\, 
\max_{\lambda \in \Delta_m}\, \omega(\lambda, f^*) \,=\,
\min_{f: \X \> \R^m}\, \omega(\lambda^*, f)\\
&\leq\, \max_{\lambda \in \Delta_m}\min_{f: \X \> \R^m}\, \omega(\lambda, f)
\,=\, \min_{f: \X \> \R^m}\max_{\lambda \in \Delta_m}\, \omega(\lambda, f)\\
&=\, \min_{f: \X \> \R^m}\,L^\rob(f),
\end{align*}
where the last equality follows from \eqref{eq:min-max}.
We thus have that $f^*$ is a minimizer of $L^\rob(f)$

Furthermore, because $f^*$ is a minimizer of $\omega(\lambda^*, f)$ over $f$, i.e.,\
\[
f^* \in \Argmin{f: \X \> \R^m} \sum_{y=1}^m \frac{\lambda_y^*}{\pi_y}\E\left[ \eta_y(X)\,\ell(y, f(X)) \right],
\]
it follows that:
$$
\softmax_y(f^*(x)) \,\propto\, \frac{\lambda^*_y}{\pi_y}\eta_y(x).
$$

(iv) For the fourth result, we expand the traded-off objective,
and re-write it as:
\begin{align*}
    L^\tdf(f)
         &= (1 - \alpha) L^\bal(f) + \alpha L^\rob(f)\\
         &= 
             (1-\alpha)\frac{1}{m}\sum_{y=1}^m \frac{1}{\pi_y}\E\left[ \eta_y(X)\,\ell(y, f(X)) \right]
             \,+\,
             \alpha \max_{\lambda \in \Delta_m}\,\sum_{y=1}^m \frac{\lambda_y}{\pi_y}\E\left[ \eta_y(X)\,\ell(y, f(X)) \right]\\
        &=  \max_{\lambda \in \Delta_m}
        \underbrace{
            \sum_{y=1}^m \left((1-\alpha)\frac{1}{m} + \alpha\lambda_y\right)\frac{1}{\pi_y}\E\left[ \eta_y(X)\,\ell(y, f(X)) \right]}_{\omega(\lambda, f)}.
    \label{eq:trade-off}
\end{align*}
For a fixed $\lambda$, $\omega(\lambda, f)$ is convex in $f$ (as the loss $\ell$ is the cross-entropy loss),
and for a fixed $f$, $\omega(\lambda, f)$ is linear in $\lambda$. Following the same steps as the proof of (iii), we have that there exists $(\lambda^*, f^*)$ such that
\[
L^\tdf(f^*) \,=\, 
\max_{\lambda \in \Delta_m}\, \omega(\lambda, f^*) \,=\,
\min_{f: \X \> \R^m}\,L^\tdf(f),
\]
and
\[
f^* \in \Argmin{f: \X \> \R^m}\,
            \sum_{y=1}^m \left((1-\alpha)\frac{1}{m} + \alpha\lambda^*_y\right)\frac{1}{\pi_y}\E\left[ \eta_y(X)\,\ell(y, f(X)) \right],
\]
which, owing to the properties of the cross-entropy loss, then gives us the desired form for $f^*$.

\subsection{Proof of Theorem \ref{thm:good-teacher}}
\label{app:proof-good-teacher}
\begin{proof}
Expanding the left-hand side, we have:
\begin{align*}
|\hat{L}^\robd(f) - L^\rob(f)|
&\leq |\hat{L}^\robd(f) - {L}^\robd(f) + {L}^\robd(f) - L^\rob(f)|\\
&\leq |\hat{L}^\robd(f) - {L}^\robd(f)| + |{L}^\robd(f) - L^\rob(f)|\\
&= 
|\hat{L}^\robd(f) - {L}^\robd(f)| +  \left|
                        \max_{y \in [m]} \frac{ \E_{x}\left[ p^t_y(x)\,\ell(y, f(x)) \right] }{ \E_x\left[ p^t_y(x) \right] } -
                        \max_{y \in [m]} \frac{ \E_{x}\left[ \eta_y(x)\,\ell(y, f(x)) \right] }{ \pi_y } \right|\\
&\leq 
|\hat{L}^\robd(f) - {L}^\robd(f)| + \max_{y \in [m]} \left|
                        \frac{ \E_{x}\left[ p^t_y(x)\,\ell(y, f(x)) \right] }{ \E_x\left[ p^t_y(x) \right] } -
                        \frac{ \E_{x}\left[ \eta_y(x)\,\ell(y, f(x)) \right] }{ \pi_y } \right|\\
&\leq
|\hat{L}^\robd(f) - {L}^\robd(f)| +
    B\max_{y \in [m]} \E_x\left[
        \left| \frac{p^t_y(x)}{\E_x\left[ p^t_y(x) \right]} \,-\, \frac{\eta_y(x)}{\pi_y} \right|\ell(y, f(x))\right]\\
&\leq
|\hat{L}^\robd(f) - {L}^\robd(f)| +
    B\max_{y \in [m]} \E_x\left[
        \left| \frac{p^t_y(x)}{\E_x\left[ p^t_y(x) \right]} \,-\, \frac{\eta_y(x)}{\pi_y} \right|\right],
\end{align*}
where the second-last step uses Jensen's inequality and the fact that $\ell(y, f(x)) \geq 0$, 
and the last step uses the fact that $\ell(y, f(x)) \leq B$.

Further expanding the first term,
\begin{align*}
|\hat{L}^\robd(f) - L^\rob(f)|
&\leq
    \left| 
        \max_{y \in [m]} \phi_y(f)
        \,-\,
        \max_{y \in [m]} \hat{\phi}_y(f)
    \right|
        +
    B\max_{y \in [m]} \E_x\left[
        \left| \frac{p^t_y(x)}{\E_x\left[ p^t_y(x) \right]} \,-\, \frac{\eta_y(x)}{\pi_y} \right|\right]\\
&\leq
    \max_{y \in [m]}\left| 
         \phi_y(f)
        \,-\,
         \hat{\phi}_y(f)
    \right|
        +
    B\max_{y \in [m]} \E_x\left[
        \left| \frac{p^t_y(x)}{\E_x\left[ p^t_y(x) \right]} \,-\, \frac{\eta_y(x)}{\pi_y} \right|\right],
\end{align*}
as desired.
\end{proof}

\subsection{Calibration of Margin-based Loss $\cL^\mar$}
\label{app:calibration-mar}
To show that minimizer of the margin-based objective in \eqref{eq:margin-la} also minimizes
the balanced objective in \eqref{eq:balanced-distilled-empirical}, we state the following general result:
\begin{lemma}
\label{lem:helper-dro-1}
    Suppose $p^t \in \cF$ and $\cF$ is closed under linear transformations. 
    Let 
    \begin{align}
        \hat{f} &\in\, \displaystyle\Argmin{f \in \cF}\, \frac{1}{n}\sum_{i=1}^n \cL^\mar\left(p^t(x_i), f(x_i); \bc \right)
        \label{eq:margin-empirical}
    \end{align}
    for some cost vector $\bc \in \R_+^m$. Then:
    \[
            \hat{f}_y(x_i) = \log\left(c_y p_y^t(x_i)\right) + C_i, ~~\forall i \in [n],
    \]
for some example-specific constant constants $C_i \in \R, \forall i \in [n]$. Furthermore, 
    for any assignment of example weights of $w \in \R^n_+$,  
    $\hat{f}$ is also the minimizer of 
    the weighted objective:
    \begin{align}
        \hat{f} \,\in\, \displaystyle\Argmin{f \in \cF}\,\frac{1}{n}\sum_{i=1}^n w_i\sum_{y=1}^m c_y\, p_y^t(x_i)\, \ell\left( y , {f}(x_i) \right).
        \label{eq:weighted-loss-empirical}
    \end{align}
\end{lemma}

\begin{proof}
Following \citet{menon2020long} (e.g.\ proof of Theorem 1), we have that
for class probabilities $\bp \in \Delta_m$ and costs $\bc \in \R^m_+$, 
the margin-based loss in \eqref{eq:margin-la} 
\begin{align*}
\cL^\mar\left(\bp, \boldf; \bc \right) 
&=
\frac{1}{m}\sum_{y \in [m]} p_y \log\bigg(1 + \sum_{j \ne y}\exp\left(\log(c_y / c_{j}) \,-\, (f_y - f_j) \right) \bigg).
\end{align*}
is minimized by:
\[
        f^*_y = \log\left(c_y p_y\right) + C,
\]
for any $C > 0$. To see why this is true, note that the above loss can be equivalently written as:
\begin{align*}
\cL^\mar\left(\bp, \boldf; \bc \right) 
&=
-\frac{1}{m}\sum_{y \in [m]} p_y \log\bigg(\frac{ \exp\left(f_y - \log(c_y) \right) }{ \sum_{j=1}^m \exp\left(f_j - \log(c_j) \right) } \bigg).
\end{align*}
This the same as the softmax cross-entropy loss with adjustments made to the logits, the minimizer for which is of the form:
\[
        f^*_y - \log(c_y) = \log\left(p_y\right) + C~~~~\text{or}~~~~f^*_y = \log\left(c_y p_y\right) + C.
\]

It follows that any minimizer $\hat{f}$ of the
average margin-based loss in \eqref{eq:margin-empirical} over sample $S$,
would do so point-wise, and therefore
\begin{align*}
        \hat{f}_y(x_i) = \log\left(c_y p_y^t(x_i)\right) + C_i, ~~\forall i \in [n],
\end{align*}
for some example-specific constant constants $C_i \in \R, \forall i \in [n]$.

To prove the second part, we  note that for the minimizer $\hat{f}$ to also minimize 
the weighted objective:
    \[
        \frac{1}{n}\sum_{i=1}^n w_i\sum_{y=1}^m c_y\, p_y^t(x_i)\, \ell\left( y , {f}(x_i) \right),
    \]
it would also have to do so point-wise for each $i \in [m]$, and so as long the weights $w_i$ are non-negative,
it suffices that
\[
    \hat{f}(x_i) \in \Argmin{\boldf \in \R^m}\, \sum_{y=1}^m c_y\, p_y^t(x_i)\, \ell\left( y , {f}(x_i) \right).
\]
This is indeed the case when $\ell$ is the softmax cross-entropy loss, where
the point-wise minimizer for each $i \in [m]$ would be of the form $\softmax_y(f(x)) = c_y p^t_y(x)$,
which is satisfied by $\hat{f}$.
\end{proof}
A similar result also holds in the population limit, when \eqref{eq:margin-empirical}
and \eqref{eq:weighted-loss-empirical} are computed in expectation, and the per-example weighting in \eqref{eq:weighted-loss-empirical} is
replaced by an arbitrary weighting function $w(x) \in \R_+$. Any scorer of the following form
would then minimize both objectives:
\[
\hat{f}_y(x) = \log\left(c_y p_y^t(x)\right) + C(x), ~~\forall x \in \X,
\]
where $C(x)$ is some  example-specific constant.

\subsection{Proof of Proposition \ref{prop:student-form}}
\label{app:student-form}
\begin{proposition*}[Restated]
Suppose $p^t \in \cF$ and $\cF$ is closed under linear transformations. Then the final scoring function $\bar{f}^s(x) = \frac{1}{K} \sum_{k=1}^K f^{k}(x)$ output by Algorithm \ref{algo:dro} is of the form:
    \[
        \softmax_j(\bar{f}^s(x)) \propto \bar{\lambda}_j p_j^t(x),~~~~\forall j \in [m],~ \forall (x, y) \in S,
    \]
where $\bar{\lambda}_y = \left(\prod_{k=1}^K \lambda_y^k / \pi^t_y\right)^{1/K}$. 
\end{proposition*}
\begin{proof}
The proof follows from Lemma \ref{lem:helper-dro-1} with the costs $\bc$ set to $\lambda^k / \pi^t$ for each iteration $k$. 
The lemma tells us that each $f^k$ is of the form:
\[
        f^k(x') = \log\left(\frac{\lambda^k_y}{\pi^t_y} p_y^t(x')\right) + C(x'), ~~\forall (x', y') \in S,
    \]
for some example-specific constant $C(x') \in \R$. Consequently, we have that:
    \[
        \bar{f}_y^s(x') = \log(\bar{\lambda}_y p_y^t(x')) + \bar{C}(x'), ~~\forall (x', y') \in S,
    \]
    where $\bar{\lambda}_y = \left(\prod_{k=1}^K \lambda_y^k / \pi^t_y\right)^{1/K}$ and $\bar{C}(x') \in \R$. Applying a
    softmax to $\bar{f}^s$ results in the desired form.
\end{proof}

\subsection{Proof of Theorem \ref{thm:dro}}
\label{app:convergence-dro}

\begin{theorem*}[Restated]
Suppose $p^t \in \cF$ and $\cF$ is closed under linear transformations.  Suppose
$\ell$ is the softmax cross-entropy loss $\ell^\xent$,
$\ell(y, z) \leq B$
and $\max_{y \in [m]}\frac{1}{\pi^t_y} \leq Z$, for some $B, Z > 0$.
Furthermore, suppose for any $\delta \in (0,1)$, the following bound holds on the
estimation error in Theorem \ref{thm:good-teacher}:
with probability at least $1 - \delta$ (over draw of $S \sim D^n$),
for all $f \in \cF$,
\[
\max_{y \in [m]} \big|\phi_y(f) - \hat{\phi}_y(f)\big| \leq \Delta(n, \delta),
\]
for some $\Delta(n, \delta) \in \R_+$ that
is increasing in $1/\delta$, and goes to 0 as $n \> \infty$. Fix $\delta \in (0,1)$.
Then when the step size $\gamma = \frac{1}{2BZ}\sqrt{\frac{\log(m)}{{K}}}$
and $n^\val \geq 8Z\log(2m/\delta)$, with
probability at least $1-\delta$ (over draw of $S \sim D^n$ and $S^\val \sim D^{n^\val}$)
\begin{align*}
L^\rob(\bar{f}^s) &\leq\,
\min_{f \in \cF}L^\rob(f)
        \,+\,
    \underbrace{2B\max_{y \in [m]} \E_x\left[
        \left| \frac{p^t_y(x)}{\pi^t_y} \,-\, \frac{\eta_y(x)}{\pi_y} \right|\right]}_{\text{Approximation error}}\\
        &\hspace{4cm}
        \,+\, \underbrace{2\Delta(n^\val, \delta/2) \,+\,  2\Delta(n, \delta/2)}_{\text{Estimation error}}
        \,+\,
        \underbrace{4BZ\sqrt{\frac{\log(m)}{{K}}}}_{\text{EG convergence}}.
\end{align*}
\end{theorem*}

Before proceeding to the proof, we will find it useful to define:
\begin{align*}
\displaystyle
\hat{\phi}^\val_y(f^s) &= \frac{1}{\hat{\pi}^{t,\val}_y}\frac{1}{n^\val}\sum_{(x', y') \in S^{\val}}  p_y^t(x')\,\ell\left( y , f^s(x') \right).
\end{align*}
We then state a useful lemma.


\allowdisplaybreaks

\begin{lemma}
\label{lem:helper-dro-2}
Suppose the conditions in Theorem \ref{thm:dro} hold.
Then with probability $\leq 1 - \delta$ (over draw of $S \sim D^{n}$ and $S^\val \sim D^{n^\val}$), at each iteration $k$, 
\[
\sum_{y=1}^m \lambda^{k+1}_y  \phi_y(f^{k+1})
\,-\,
\min_{f \in \cF}\,\sum_{y=1}^m \lambda^{k+1}_y  \phi_y(f)
\leq 2\Delta(n, \delta);
\]
and for any $\lambda \in \Delta_m$:
\[
\left|\sum_{y=1}^m \lambda_y \hat{\phi}^\val_y(f^{k+1}) \,-\, 
\sum_{y=1}^m \lambda_y  \phi_y(f^{k+1}) \right| ~\leq~ \Delta(n^\val, \delta).
\]
\end{lemma}
\begin{proof}
We first note that by applying Lemma \ref{lem:helper-dro-1}  with $w_i = 1,\forall i$, we have that $f^{k+1}$ is the minimizer
of $\sum_{y=1}^m \lambda^{k+1}_y  \hat{\phi}_y(f)$ over all $f\in \cF$, and therefore:
\begin{equation}
    \sum_{y=1}^m \lambda^{k+1}_y  \hat{\phi}_y(f^{k+1}) \,\leq\, \sum_{y=1}^m \lambda^{k+1}_y  \hat{\phi}_y(f),~\forall f\in \cF.
    \label{eq:f-k+1-minimizer}
\end{equation}

Further, for a fixed iteration $k$, let us denote $\tilde{f} \in \Argmin{f \in \cF}\,  \sum_{y=1}^m\lambda^{k+1}_y  \phi_y({f})$. 
Then for the first part, we have:
\begin{align*}
\lefteqn{
\sum_{y=1}^m \lambda^{k+1}_y  \phi_y(f^{k+1})
\,-\,
\sum_{y=1}^m \lambda^{k+1}_y  \phi_y(\tilde{f})
}\\
&\leq
\sum_{y=1}^m \lambda^{k+1}_y  \phi_y(f^{k+1})
\,-\,
\sum_{y=1}^m \lambda^{k+1}_y \hat{\phi}_y(f^{k+1})
\,+\, \sum_{y=1}^m \lambda^{k+1}_y \hat{\phi}_y(f^{k+1})
\,-\,
\sum_{y=1}^m \lambda^{k+1}_y  \phi_y(\tilde{f})
\\
&\leq 
\sum_{y=1}^m \lambda^{k+1}_y  \phi_y(f^{k+1})
\,-\,
\sum_{y=1}^m \lambda^{k+1}_y \hat{\phi}_y(f^{k+1})
\,+\, \sum_{y=1}^m \lambda^{k+1}_y \hat{\phi}_y(\tilde{f})
\,-\,
\sum_{y=1}^m \lambda^{k+1}_y  \phi_y(\tilde{f})
\\
&\leq
    2\sup_{f \in \cF}\left|
    \sum_{y=1}^m \lambda^{k+1}_y  \hat{\phi}_y(f)
        \,-\,
    \sum_{y=1}^m \lambda^{k+1}_y  \phi_y(f)
    \right|\\
    &\leq
         2\sup_{f \in \cF}\max_{\lambda \in \Delta_m}\, 
        \left|\sum_{y=1}^m \lambda_y  \hat{\phi}_y(f)
        \,-\,
    \sum_{y=1}^m \lambda_y  \phi_y(f)
            \right|\\
    &\leq
        2 \sup_{f \in \cF}\max_{\lambda \in \Delta_m}\, 
        \sum_{y=1}^m {\lambda_y}
        \left|\hat{\phi}_y(f) \,-\, \phi_y(f) \right|
    \\
    &=  2\sup_{f \in \cF}\max_{y \in [m]} \big|\hat{\phi}_y(f) - {\phi}_y(f)\big|.
\end{align*}
where for the second inequality, we use \eqref{eq:f-k+1-minimizer}.
Applying the generalization bound assumed in Theorem \ref{thm:dro},
we have with probability $\leq 1 - \delta$ (over draw of $S \sim D^{n}$), for all iterations $k \in [K]$,
\[
\sum_{y=1}^m \lambda^{k+1}_y  \phi_y(f^{k+1})
\,-\,
\sum_{y=1}^m \lambda^{k+1}_y  \phi_y(\tilde{f})
    \,\leq\, 2\Delta(n, \delta),
\]

For the second part, note that for any $\lambda \in \Delta_m$,
\begin{align*}
    \left|\sum_{y=1}^m \lambda_y \hat{\phi}^\val_y(f^{k+1}) \,-\, 
\sum_{y=1}^m \lambda_y  \phi_y(f^{k+1}) \right| &\leq \sum_{y=1}^m \lambda_y\left| \hat{\phi}^\val_y(f^{k+1}) \,-\, 
\phi_y(f^{k+1})\right|\\
&\leq \max_{y\in[m]}\,\left| \hat{\phi}^\val_y(f^{k+1}) \,-\, 
\phi_y(f^{k+1})\right|\\
&\leq \sup_{f\in \cF}\max_{y\in[m]}\,\left| \hat{\phi}^\val_y(f) \,-\, 
\phi_y(f)\right|.
\end{align*}
An application of the generalization bound assumed in Theorem \ref{thm:dro} to empirical estimates from the validation sample
completes the proof.
\end{proof}

We are now ready to prove Theorem \ref{thm:dro}.

\begin{proof}[Proof of Theorem \ref{thm:dro}]
Note that because $\min_{y \in [m]}\pi^{t}_y \geq \frac{1}{Z}$ and $n^\val \geq 8Z\log(2m/\delta)$,
we have by a direct application of Chernoff's bound (along with a union bound over all $m$ classes) that with
probability at least $1-\delta/2$:
$$
\min_{y \in [m]}\hat{\pi}^{t,\val}_y \geq \frac{1}{2Z}, \forall y \in [m]
$$
and consequently,
$\hat{\phi}_y^\val(f) \leq 2BZ, \forall f \in \cF$. The boundedness of $\hat{\phi}_y^\val$ will then allow us to apply standard convergence guarantees for exponentiated gradient ascent \citep{shalev2011online}. For $\gamma = \frac{1}{2BZ}\sqrt{\frac{\log(m)}{{K}}}$,
the updates on $\lambda$ will give us  with
probability at least $1-\delta/2$:
\begin{equation}
\max_{\lambda \in \Delta_m}\,\frac{1}{K}\sum_{k=1}^K\sum_{y=1}^m \lambda_y \hat{\phi}_y^\val(f^k)
\,\leq\, 
\frac{1}{K}\sum_{k=1}^K\sum_{y=1}^m \lambda_y^k \hat{\phi}_y^\val(f^k)
\,+\,
4BZ\sqrt{\frac{\log(m)}{{K}}}
    \label{eq:phi-upper-bound}
\end{equation}

Applying the second part of Lemma \ref{lem:helper-dro-2} to each iteration $k$, we have
with probability at least $1-\delta$:
\[
\max_{\lambda \in \Delta_m}\,\frac{1}{K}\sum_{k=1}^K\sum_{y=1}^m \lambda_y {\phi}_y(f^k)
\,\leq\, 
\frac{1}{K}\sum_{k=1}^K\sum_{y=1}^m \lambda_y^k {\phi}_y(f^k)
\,+\,
4BZ\sqrt{\frac{\log(m)}{{K}}} \,+\, 2\Delta(n^\val, \delta/2),
\]
and applying the first part of Lemma \ref{lem:helper-dro-2} to the RHS, we have with the same probability:
\begin{align*}
\lefteqn{
\max_{\lambda \in \Delta_m}\,\frac{1}{K}\sum_{k=1}^K\sum_{y=1}^m \lambda_y {\phi}_y(f^k)}\\
&\leq
\frac{1}{K}\sum_{k=1}^K\min_{f \in \cF}\sum_{y=1}^m \lambda_y^k {\phi}_y(f)
\,+\,
4BZ\sqrt{\frac{\log(m)}{{K}}} \,+\, 2\Delta(n^\val, \delta/2) \,+\,  2\Delta(n, \delta/2)\\
&\leq
\min_{f \in \cF}\,\frac{1}{K}\sum_{k=1}^K\sum_{y=1}^m \lambda_y^k {\phi}_y(f)
\,+\,
4BZ\sqrt{\frac{\log(m)}{{K}}} \,+\, 2\Delta(n^\val, \delta/2) \,+\,  2\Delta(n, \delta/2).
\end{align*}
Note that we have taken a union bound over the high probability statement in \eqref{eq:phi-upper-bound} and that in Lemma \ref{lem:helper-dro-2}. 
Using the convexity of $\phi(\cdot)$ in $f(x)$ and Jensen's inequality, 
we have that $\sum_{y=1}^m \lambda_y {\phi}_y(\bar{f}^s) \leq \frac{1}{K}\sum_{k=1}^K\sum_{y=1}^m \lambda_y {\phi}_y(f^k)$.
We use this to further lower bound the LHS 
in terms of the averaged scoring function $\bar{f}^s(x) = \frac{1}{K}\sum_{k=1}^K f^k(x)$:
\begin{align}
\lefteqn{\max_{\lambda \in \Delta_m}\,\sum_{y=1}^m \lambda_y {\phi}_y(\bar{f}^s)}
\nonumber
\\
    &\leq
        \min_{f \in \cF}\,\frac{1}{K}\sum_{k=1}^K\sum_{y=1}^m \lambda_y^k {\phi}_y(f)
        \,+\,
        4BZ\sqrt{\frac{\log(m)}{{K}}} \,+\, 2\Delta(n^\val, \delta/2) \,+\,  2\Delta(n, \delta/2)
        \nonumber\\
    &=
        \min_{f \in \cF}\,\sum_{y=1}^m \tilde{\lambda}_y {\phi}_y(f)
        \,+\,
        4BZ\sqrt{\frac{\log(m)}{{K}}} \,+\, 2\Delta(n^\val, \delta/2) \,+\,  2\Delta(n, \delta/2)
        \nonumber\\
    &\leq
        \max_{\lambda \in \Delta_m}\min_{f \in \cF}\,\sum_{y=1}^m {\lambda}_y {\phi}_y(f)
        \,+\,
        4BZ\sqrt{\frac{\log(m)}{{K}}} \,+\, 2\Delta(n^\val, \delta/2) \,+\,  2\Delta(n, \delta/2)
        \nonumber\\
    &=
        \min_{f \in \cF}\max_{\lambda \in \Delta_m}\,\sum_{y=1}^m {\lambda}_y {\phi}_y(f)
        \,+\,
        4BZ\sqrt{\frac{\log(m)}{{K}}} \,+\, 2\Delta(n^\val, \delta/2) \,+\,  2\Delta(n, \delta/2)
        \nonumber\\
    &=
        \min_{f \in \cF}\max_{y \in [m]}\, {\phi}_y(f)
        \,+\,
        4BZ\sqrt{\frac{\log(m)}{{K}}} \,+\, 2\Delta(n^\val, \delta/2) \,+\,  2\Delta(n, \delta/2),
        \label{eq:dro-final}
\end{align}
where in the second step $\tilde{\lambda}_y = \frac{1}{K}\sum_{k=1}^K \lambda^k_y$;
in the fourth step, we swap the `min' and `max' using  Sion's minimax theorem \citep{sion1958general}. 
We further have from \eqref{eq:dro-final},
\begin{align*}
\max_{y \in [m]}\,{\phi}_y(\bar{f}^s)
    &\leq
        \min_{f \in \cF}\max_{y \in [m]}\, {\phi}_y(f)
        \,+\,
        4BZ\sqrt{\frac{\log(m)}{{K}}} \,+\, 2\Delta(n^\val, \delta/2) \,+\,  2\Delta(n, \delta/2).
\end{align*}
In other words,
\[
    L^\robd(\bar{f}^s)
    \leq
        \min_{f \in \cF}L^\robd(f)
        \,+\,
        4BZ\sqrt{\frac{\log(m)}{{K}}} \,+\, 2\Delta(n^\val, \delta/2) \,+\,  2\Delta(n, \delta/2).
\]

To complete the proof, we need to turn this into a guarantee on the original robust objective $L^\rob$ in \eqref{eq:robust}:
\begin{align*}
    \lefteqn{L^\rob(\bar{f}^s)}\\
    &\leq
        \min_{f \in \cF}L^\rob(f)
        \,+\, 2\max_{f \in \cF}\,\left|L^\rob(f) - L^\robd(f)\right|
        \,+\,
        4BZ\sqrt{\frac{\log(m)}{{K}}} \,+\, 2\Delta(n^\val, \delta/2) \,+\,  2\Delta(n, \delta/2)\\
    &\leq
    \min_{f \in \cF}L^\rob(f)
        \,+\,
    2B\max_{y \in [m]} \E_x\left[
        \left| \frac{p^t_y(x)}{\pi^t_y} \,-\, \frac{\eta_y(x)}{\pi_y} \right|\right]
        \,+\,
        4BZ\sqrt{\frac{\log(m)}{{K}}} \,+\, 2\Delta(n^\val, \delta/2) \,+\,  2\Delta(n, \delta/2),
\end{align*}
where we have used the bound on the approximation error in the proof of Theorem \ref{thm:good-teacher}. 
This completes the proof.
\end{proof}

\section{Student Estimation Error}
\label{app:student-gen-bound}

We now provide a bound on the estimation error in Theorem \ref{thm:dro} using a generalization bound from \citet{menon2021statistical}.
\begin{lemma}
\label{lem:student-gen-bound}
Let $\cF \subseteq \R^\X$ be a given class of scoring functions.
Let $\mathcal{V} \subseteq \R^\X$ denote the class of loss functions $v(x, y) = \ell(y, f(x))$ induced
by scorers $f \in \cF$. Let $\cM_n = \mathcal{N}_\infty(\frac{1}{n}, \mathcal{V}, 2n)$ denote the
uniform $L_\infty$ covering number for $\mathcal{V}$. Fix $\delta \in (0,1)$. 
Suppose $\ell(y, z) \leq B$,
$\pi^t_y \leq \frac{1}{Z}, \forall y \in [m]$, and the
number of samples $n \geq 8Z\log(4m/\delta)$. 
Then with probability $\geq 1 - \delta$
over draw of $S \sim D^n$, for any $f \in \cF$ and $y \in [m]$:
\[
\left| \phi_y(f) - \hat{\phi}_y(f) \right| 
\,\leq\,
CZ\left( \sqrt{\bV_{n,y}(f) \frac{\log(m\cM_n / \delta)}{n}} \,+\, \frac{\log(m\cM_n / \delta)}{n} 
\,+\,
B\sqrt{ \frac{\log(m/\delta)}{n} }\right),
\]
where $\bV_{n,y}(f)$ denotes the empirical variance of the loss values $\{p^t_y(x_i)\cdot\ell(y, f(x_i))\}_{i=1}^n$ for class $y$, and
$C > 0$ is a distribution-independent constant.
\end{lemma}
Notice the dependence on the \emph{variance} that the teacher's prediction induce on the loss. This suggests that the lower the variance in the teacher's predictions, the better is the student's generalization. Similar to \citet{menon2021statistical}, one can further show that when the teacher closely approximates the Bayes-probabilities $\eta(x)$, the distilled loss $p^t_y(x_i)\cdot\ell(y, f(x_i))$ has a lower empirical variance that the loss $\ell(y_i, f(x_i))$ computed from one-hot labels.

\begin{proof}[Proof of Lemma \ref{lem:student-gen-bound}]
We begin by defining the following intermediate term:
\[
\tilde{\phi}_y(f) = 
\frac{1}{\pi^t_y}\frac{1}{n}\sum_{i=1}^n  p_y^t(x_i)\,\ell\left( y , f(x_i) \right).
\]
Then for any $y \in [m]$,
\begin{align}
    \left| \phi_y(f) - \hat{\phi}_y(f) \right|
    &\leq 
    \left| \phi_y(f) - \tilde{\phi}_y(f) \right|
    +
    \left| \tilde{\phi}_y(f) - \hat{\phi}_y(f) \right|.
    \label{eq:genbound-inter}
\end{align}
We next bound each of the terms in \eqref{eq:genbound-inter}, starting with the first term:
\begin{align*}
\left| \phi_y(f) - \tilde{\phi}_y(f) \right|
    &= 
        \frac{1}{\pi^t_y}
        \left|
        \E_x\left[ p_y^t(x)\, \ell\left( y , f(x) \right)\right] \,-\, \frac{1}{n}\sum_{i=1}^n  p_y^t(x_i)\,\ell\left( y , f(x_i) \right)
        \right|
        \\
    &\leq
        Z
        \left|
        \E_x\left[ p_y^t(x)\, \ell\left( y , f(x) \right)\right] \,-\, \frac{1}{n}\sum_{i=1}^n  p_y^t(x_i)\,\ell\left( y , f(x_i) \right)
        \right|,
\end{align*}
where we use the fact that $\pi^t_y \leq \frac{1}{Z}, \forall y$. 
Applying the generalization bound from \citet[Proposition 2]{menon2021statistical}, along with a union bound over all $m$ classes, we have with probability at least $1-\delta/2$ over the draw of $S \sim D^n$, for all $y \in [m]$:
\begin{align}
    \left| \phi_y(f) - \tilde{\phi}_y(f) \right|
    &\leq
    C'Z\left( \sqrt{\bV_{n,y}(f) \frac{\log(m\cM_n / \delta)}{n}} \,+\, \frac{\log(m\cM_n / \delta)}{n} \right),
    \label{eq:genbound-term1}
\end{align}
for a distribution-independent constant $C' > 0$.

We next bound the second term in \eqref{eq:genbound-inter}:
\begin{align*}
\left| \tilde{\phi}_y(f) - \hat{\phi}_y(f) \right|
    &= 
        \left|
        \frac{1}{\pi^t_y}
        \,-\,
        \frac{1}{\hat{\pi}^t_y}
        \right|
        \frac{1}{n}\sum_{i=1}^n  p_y^t(x_i)\cdot\ell\left( y , f(x_i) \right)
        \\
    &\leq
        B\left|
        \frac{1}{\pi^t_y}
        \,-\,
        \frac{1}{\hat{\pi}^t_y}
        \right|
        \\
    &=
    \frac{B}{\pi^t_y\hat{\pi}^t_y}
    \left|
        {\pi}^t_y - \hat{\pi}^t_y
    \right|,
\end{align*}
where in the second step we use the fact that $\ell(y, f(x)) \leq B$ and $p_y^t(x) \leq 1$.

Further note that because $\min_{y \in [m]}\pi^t_y \geq \frac{1}{Z}$ and $n \geq 8Z\log(4m/\delta)$,
we have by a direct application of Chernoff's bound (and a union bound over $m$ classes) that with
probability at least $1-\delta/4$:
\begin{equation}
  \min_{y \in [m]}\hat{\pi}^t_y \geq \frac{1}{2Z}, \forall y \in [m].  
  \label{eq:pi-hat-bound}
\end{equation}

Therefore for any $y \in [m]$:
\begin{align*}
\left| \tilde{\phi}_y(f) - \hat{\phi}_y(f) \right|
    &\leq
    2BZ^2
    \left|
        {\pi}^t_y - \hat{\pi}^t_y
    \right|.
\end{align*}
Conditioned on the above statement, a simple application of Hoeffdings and a union bound over all $y \in [m]$ gives us that with probability at least $1-\delta/4$ over the draw of $S \sim D^n$, for all $y \in [m]$:
\begin{align}
\left| \tilde{\phi}_y(f) - \hat{\phi}_y(f) \right|
    &\leq
    2BZ^2\left(\frac{1}{Z} \sqrt{ 
            \frac{\log(8m/\delta)}{2n} }\right)
    ~= 
    2BZ \sqrt{ \frac{\log(8m/\delta)}{2n} }.
    \label{eq:genbound-term2}
\end{align}

A union bound over the high probability statements in (\ref{eq:genbound-term1}--\ref{eq:genbound-term2}) completes the proof. To see this, note that, for any $\epsilon > 0$ and $y \in [m]$,
\begin{align*}
    \lefteqn{\P\left( \left| {\phi}_y(f) - \hat{\phi}_y(f) \right| \geq \epsilon \right)}
    \\
    &\leq 
      \P\left(
        \left(\left| {\phi}_y(f) - \tilde{\phi}_y(f) \right| \geq \epsilon\right)
        \vee
      \left(\left| \tilde{\phi}_y(f) - \hat{\phi}_y(f) \right| \geq \epsilon\right)
      \right)\\
    &\leq 
      \P\left(\left| {\phi}_y(f) - \tilde{\phi}_y(f) \right| \geq \epsilon\right)
      \,+\,
     \P\left(\left| \tilde{\phi}_y(f) - \hat{\phi}_y(f) \right| \geq \epsilon\right)\\
     &\leq 
      \P\left(\left| {\phi}_y(f) - \tilde{\phi}_y(f) \right| \geq \epsilon\right)
      \,+\,
     \P\left(\hat{\pi}^t_y \leq \frac{1}{Z}\right) \cdot
     \P\left(\left| \tilde{\phi}_y(f) - \hat{\phi}_y(f) \right| \geq \epsilon ~\bigg|~ \hat{\pi}^t_y \leq \frac{1}{Z}\right)
     \\
     &
     \hspace{5cm}\,+\,
     \P\left(\hat{\pi}^t_y \geq \frac{1}{Z}\right) \cdot
     \P\left(\left| \tilde{\phi}_y(f) - \hat{\phi}_y(f) \right| \geq \epsilon ~\bigg|~ \hat{\pi}^t_y \geq \frac{1}{Z}\right)
     \\
    &\leq 
      \P\left(\left| {\phi}_y(f) - \tilde{\phi}_y(f) \right| \geq \epsilon\right)
      \,+\,
     \P\left(\hat{\pi}^t_y \leq \frac{1}{Z}\right)
     \,+\,
     \P\left(\left| \tilde{\phi}_y(f) - \hat{\phi}_y(f) \right| \geq \epsilon ~\bigg|~ \hat{\pi}^t_y \geq \frac{1}{Z}\right),
\end{align*}
which implies that a union bound over (\ref{eq:genbound-term1}--\ref{eq:genbound-term2}) would give us the desired result in Lemma \ref{lem:student-gen-bound}.
\end{proof}

\section{DRO with One-hot Validation Labels}
\label{app:one-hot-vali}
\begin{figure}
\begin{algorithm}[H]
\caption{Distilled Margin-based DRO with One-hot Validation Labels}
\label{algo:dro-val}
\begin{algorithmic}
\STATE \textbf{Inputs:} Teacher $p^t$, Student hypothesis class $\cF$, Training set $S$, Validation set $S^\val$, Step-size $\gamma \in \R_+$,
Number of iterations $K$, Loss $\ell$
\STATE \textbf{Initialize:} Student $f^0 \in \cF$, Multipliers $\blambda^0 \in \Delta_m$
\STATE \textbf{For}~{$k = 0 $ to $K-1$}
\STATE ~~~$\tilde{\lambda}^{k+1}_j \,=\, \lambda^k_j\exp\big( \gamma \hat{R}_j \big), \forall j \in [m]$
\STATE \hspace{2cm}\text{where} $\hat{R}_j =$ $\displaystyle\frac{1}{n^\val}\frac{1}{\hat{\pi}^{\val}_j}\sum_{(x, y) \in S^\val} \ell( y , f^k(x) )$
and $\hat{\pi}^{\val}_j = \displaystyle\frac{1}{n^\val}\sum_{(x, y) \in S^\val} \1(y = j)$
\STATE ~~~$\lambda^{k+1}_y \,=\, \frac{\tilde{\lambda}^{k+1}_y}{\sum_{j=1}^m \tilde{\lambda}^{k+1}_j}, \forall y$
\STATE ~~~$f^{k+1} \,\in\, \displaystyle\Argmin{f \in \cF}\, \frac{1}{n}\sum_{i=1}^n \cL^\mar\left(p^t(x_i), f(x_i); \frac{\lambda^{k+1}}{\hat{\pi}^t} \right)$
~~// Replaced with a few steps of SGD
\STATE \textbf{End For}
\STATE \textbf{Output:} $\bar{f}^{s}: x \mapsto \frac{1}{K}\sum_{k =1}^K f^k(x)$
\end{algorithmic}
\end{algorithm}
\end{figure}

Algorithm \ref{algo:dro-val} contains a version of the margin-based DRO described in Section \ref{sec:algorithms}, where instead of teacher labels the original one-hot labels are used in the validation set. 
Before proceeding to providing a convergence guarantee for this algorithm, we will find it useful
to define the following one-hot metrics:
\begin{align*}
\phi^\oh_y(f^s) &= \frac{1}{\pi_y}\E_x\left[ \eta_y(x)\, \ell\left( y , f^s(x) \right)\right]\\
\hat{\phi}^{\oh,\val}_y(f^s) &= \frac{1}{\hat{\pi}_y}\frac{1}{n^\val}\sum_{(x', y') \in S^\val}\,\1(y' = y)\, \ell\left( y' , f^s(x') \right).
\end{align*}

\begin{theorem}
\label{thm:dro-oh-vali}
Suppose $p^t \in \cF$ and $\cF$ is closed under linear transformations. Then the final scoring function $\bar{f}^s(x) = \frac{1}{K} \sum_{k=1}^K f^{k}(x)$ output by Algorithm \ref{algo:dro-val} is of the form:
    \[
    \softmax_y(\bar{f}^s(x')) \propto \bar{\lambda}_y p_y^t(x'),~~~~\forall (x', y') \in S,
    \]
where $\bar{\lambda}_y = \left(\prod_{k=1}^K \lambda_y^k / \pi^t_y\right)^{1/K}$. 
Furthermore, 
suppose
$\ell$ is the softmax cross-entropy loss in $\ell^\xent$,
$\ell(y, z) \leq B$, for some $B > 0$, and 
$\max_{y \in [m]}\frac{1}{\pi_y} \leq Z$, for some $Z > 0$.
Suppose for any $\delta \in (0,1)$, the following holds:
with probability at least $1 - \delta$ (over draw of $S \sim D^n$),
for all $f \in \cF$,
\[
\max_{y \in [m]} \big|\phi^\oh_y(f) - \hat{\phi}^{\oh}_y(f)\big| \leq \Delta^\oh(n, \delta);
~~~~~~~~
\max_{y \in [m]} \big|\phi_y(f) - \hat{\phi}_y(f)\big| \leq {\Delta}(n, \delta),
\]
for some $\Delta^\oh(n, \delta), {\Delta}(n, \delta) \in \R_+$ that
is increasing in $1/\delta$, and goes to 0 as $n \> \infty$. 
Fix $\delta \in (0,1)$.
Then when the step size $\gamma = \frac{1}{2BZ}\sqrt{\frac{\log(m)}{{K}}}$
and $n^\val \geq 8Z\log(2m/\delta)$, with
probability at least $1-\delta$ (over draw of $S \sim D^n$ and $S^\val \sim D^{n^\val}$), for any $\tau \in \R_+$,
\begin{align*}
L^\rob(\bar{f}^s) &\leq\,
\min_{f \in \cF}L^\rob(f)
        \,+\,
    \underbrace{2B\max_{y \in [m]} \E_x\left[
        \left| \tau\cdot\frac{p^t_y(x)}{\pi^t_y} \,-\, \frac{\eta_y(x)}{\pi_y} \right|\right]}_{\text{Approximation error}}\\
        &\hspace{4cm}
        \,+\, \underbrace{2\tau\cdot\Delta^\oh(n^\val, \delta/2) \,+\,  2\Delta(n, \delta/2)}_{\text{Estimation error}}
        \,+\,
        \underbrace{4BZ\sqrt{\frac{\log(m)}{{K}}}}_{\text{EG convergence}}.
\end{align*}
\end{theorem}

Comparing this  to the bound in Theorem \ref{thm:dro}, we can see that there
is an additional scaling factor $\tau$ against the teacher probabilities $p^t_y(x)$ and in the approximation error. When we set $\tau = 1$, the bound looks very similar to  Theorem \ref{thm:dro}, except that the estimation error term $\Delta^\oh$ now involves one-hot labels. Therefore the estimation error may incur a slower convergence with sample size as it no longer benefits from the lower variance that the teacher predictions  may offer (see Appendix \ref{app:student-gen-bound} for details).

The $\tau$-scaling in the
approximation error also means that the teacher is no longer required to exactly match the (normalized) class probabilities $\eta(x)$. In fact, one can set $\tau$ to a value for which the approximation error is the lowest, and in general to a value that minimizes the upper bound in Theorem \ref{thm:dro-oh-vali}, potentially providing us with a tighter convergence rate than Theorem \ref{thm:dro}.



The proof of Theorem \ref{thm:dro-oh-vali} is similar to that of Theorem \ref{thm:dro}, but requires a modified version of Lemma \ref{lem:helper-dro-2}: 
\begin{lemma}
\label{lem:helper-dro-3}
Suppose the conditions in Theorem \ref{thm:dro} hold.
With probability $\leq 1 - \delta$ (over draw of $S \sim D^{n}$ and $S^\val \sim D^{n^\val}$), 
at each iteration $k$ and for any $\tau \in \R_+$,
\begin{align*}
{
\sum_{y=1}^m \lambda^{k+1}_y  \phi^\oh_y(f^{k+1})
\,-\,
\min_{f \in \cF}\,\sum_{y=1}^m \lambda^{k+1}_y  \phi^\oh_y(f)}
&\leq 
2\tau\cdot{\Delta}(n, \delta) \,+\,
2B\max_{y \in [m]} \E_x\left[
        \left| \tau\frac{p^t_y(x)}{\pi^t_y} \,-\, \frac{\eta_y(x)}{\pi_y} \right|\right].
\end{align*}
Furthermore, with the same probability, for any $\lambda \in \Delta_m$:
\[
\left|\sum_{y=1}^m \lambda_y \hat{\phi}^{\oh,\val}_y(f^{k+1}) \,-\, 
\sum_{y=1}^m \lambda_y  \phi^\oh_y(f^{k+1}) \right| ~\leq~ \Delta^\oh(n^\val, \delta).
\]
\end{lemma}
\begin{proof}
We first note from Lemma \ref{lem:helper-dro-1} that
because $f^{k+1} \,\in\, \displaystyle\Argmin{f \in \cF}\, \frac{1}{n}\sum_{i=1}^n \cL^\mar\Big(p^t(x_i), f(x_i); \frac{\lambda^{k+1}}{\hat{\pi}} \Big)$, we have for the example-weighting
 $w_i = \tau, \forall i$:
\begin{equation}
    \tau\sum_{y=1}^m \lambda^{k+1}_y  \hat{\phi}_y(f^{k+1}) \,\leq\, \tau\sum_{y=1}^m \lambda^{k+1}_y  \hat{\phi}_y(f),~\forall f\in \cF.
    \label{eq:f-k+1-minimizer-2}
\end{equation}

For a fixed iteration $k$, let us denote $\tilde{f} \in \Argmin{f \in \cF}\,  \sum_{y=1}^m\lambda^{k+1}_y  \phi_y({f})$. Then for the first part,
we have for any $\tau \in \R_+$:
\begin{align*}
\lefteqn{
\sum_{y=1}^m \lambda^{k+1}_y  \phi^\oh_y(f^{k+1})
\,-\,
\sum_{y=1}^m \lambda^{k+1}_y  \phi^\oh_y(\tilde{f})
\hspace{10cm}}\\
&\leq
\tau\left(
\sum_{y=1}^m \lambda^{k+1}_y  \phi_y(f^{k+1})
\,-\,
\sum_{y=1}^m \lambda^{k+1}_y  \phi_y(\tilde{f})
\right)
\,+\,
\sum_{y=1}^m \lambda^{k+1}_y  \left|\phi^\oh_y(f^{k+1}) 
- \tau\phi_y(f^{k+1})\right| 
\\
&\hspace{9cm}
\,+\,
\sum_{y=1}^m \lambda^{k+1}_y  \left|\phi^\oh_y(\tilde{f}) 
- \tau\phi_y(\tilde{f})\right| 
\\
&\leq
\tau\left(
\sum_{y=1}^m \lambda^{k+1}_y  \phi_y(f^{k+1})
\,-\,
\sum_{y=1}^m \lambda^{k+1}_y  \phi_y(\tilde{f})
\right)
\,+\,
2\max_{f \in \cF}\,
\sum_{y=1}^m \lambda^{k+1}_y  \left|\phi^\oh_y(f) 
- \tau\phi_y(f)\right| 
\\
&\leq
\tau\left(
\sum_{y=1}^m \lambda^{k+1}_y  \phi_y(f^{k+1})
\,-\,
\sum_{y=1}^m \lambda^{k+1}_y  \phi_y(\tilde{f})
\right)
\,+\,
2\max_{f \in \cF}\max_{\lambda \in \Delta_m}\,
\sum_{y=1}^m \lambda \left|\phi^\oh_y(f) 
- \tau\phi_y(f)\right| 
\\
&\leq
\tau\left(
\sum_{y=1}^m \lambda^{k+1}_y  \phi_y(f^{k+1})
\,-\,
\sum_{y=1}^m \lambda^{k+1}_y  \phi_y(\tilde{f})
\right)
\,+\,
2\max_{f \in \cF}\,
\max_{y\in[m]} \left|\phi^\oh_y(f) 
- \tau\phi_y(f)\right|
\\
&\leq
2\tau\sup_{f \in \cF}\max_{y \in [m]} \big|\hat{\phi}_y(f) - {\phi}_y(f)\big|
\,+\,
2\max_{f \in \cF}\,
\max_{y\in[m]} \left|\phi^\oh_y(f) 
- \tau\phi_y(f)\right|. 
\end{align*}
where the last inequality re-traces the steps in Lemma \ref{lem:helper-dro-2}. Further applying the generalization bound assumed in Theorem \ref{thm:dro},
we have with probability $\leq 1 - \delta$ (over draw of $S \sim D^{n}$), for all iterations $k \in [K]$ and any $\tau \in \R_+$,
\begin{equation}
    \sum_{y=1}^m \lambda^{k+1}_y  \phi^\oh_y(f^{k+1})
    \,-\,
    \sum_{y=1}^m \lambda^{k+1}_y  \phi^\oh_y(\tilde{f})
        \,\leq\, 2\tau\Delta(n, \delta) + 
        2\max_{f \in \cF}\,
    \max_{y\in[m]} \left|\phi^\oh_y(f) 
    - \tau\phi_y(f)\right|.
    \label{eq:dro-vali-inter}
\end{equation}

All that remains is to bound the second term in \eqref{eq:dro-vali-inter}. For any $f \in \cF$ and $y \in [m]$,
\begin{align*}
    \left|\phi^\oh_y(f) 
    - \tau\phi_y(f)\right|
    &\leq
    \left|
    \frac{1}{\pi_y}\E_x\left[ \eta_y(x)\, \ell\left( y , f(x) \right)\right]
    \,-\,
    \frac{\tau}{\pi^t_y}\E_x\left[ p^t_y(x)\, \ell\left( y , f(x) \right)\right]
    \right|\\
    &\leq
    \E_x\left[
    \left|
    \frac{1}{\pi_y} \eta_y(x)\, \ell\left( y , f(x) \right)
    \,-\,
    \frac{\tau}{\pi^t_y} p^t_y(x)\, \ell\left( y , f(x) \right)
    \right|\right]
    \\
    &=
    \E_x\left[
    \left|
    \frac{1}{\pi_y} \eta_y(x)
    \,-\,
    \frac{\tau}{\pi^t_y} p^t_y(x)
    \right|\ell\left( y , f^s(x) \right)\right]
    \\
    &\leq
    B\E_x\left[
    \left|
    \frac{\eta_y(x)}{\pi_y} 
    \,-\,
    \tau\frac{p^t_y(x)}{\pi^t_y} 
    \right|\right],
\end{align*}
where we use Jensen's inequality in the second step,
the fact  that $\ell(y, z) \leq B$ is non-negative in the second step, 
and the fact  that $\ell(y, z) \leq B$ in the last step. Substituting this upper bound back into \eqref{eq:dro-vali-inter} completes the proof of the first part.

The second part follows from a direct application of the bound on the per-class estimation error $\max_{y \in [m]} \big|\phi^\oh_y(f) - \hat{\phi}^{\oh,\val}_y(f)\big|$.
\end{proof}

\begin{proof}[Proof of Theorem \ref{thm:dro-oh-vali}]
The proof traces the same steps as 
Proposition \ref{prop:student-form} and
Theorem \ref{thm:dro}, except that 
it applies Lemma \ref{lem:helper-dro-3} instead of Lemma \ref{lem:helper-dro-2}.

Note that because $\min_{y \in [m]}\pi_y \geq \frac{1}{Z}$ and $n^\val \geq 8Z\log(2m/\delta)$,
we have by a direct application of Chernoff's bound (along with a union bound over all $m$ classes) that with
probability at least $1-\delta/2$:
$$
\min_{y \in [m]}\hat{\pi}^{\oh,\val}_y \geq \frac{1}{2Z}, \forall y \in [m],
$$
and consequently,
$\hat{\phi}_y^{\oh,\val}(f) \leq 2BZ, \forall f \in \cF$. The boundedness of $\hat{\phi}_y^{\oh,\val}$ will then allow us to apply standard convergence guarantees for exponentiated gradient ascent \citep{shalev2011online}. For $\gamma = \frac{1}{2BZ}\sqrt{\frac{\log(m)}{{K}}}$,
the updates on $\lambda$ will give us:
\[
\max_{\lambda \in \Delta_m}\,\frac{1}{K}\sum_{k=1}^K\sum_{y=1}^m \lambda_y \hat{\phi}_y^{\oh,\val}(f^k)
\,\leq\, 
\frac{1}{K}\sum_{k=1}^K\sum_{y=1}^m \lambda_y^k \hat{\phi}_y^{\oh,\val}(f^k)
\,+\,
4BZ\sqrt{\frac{\log(m)}{{K}}}
\]

Applying the second part of Lemma \ref{lem:helper-dro-2} to each iteration $k$, we have
with probability at least $1-\delta$:
\[
\max_{\lambda \in \Delta_m}\,\frac{1}{K}\sum_{k=1}^K\sum_{y=1}^m \lambda_y {\phi}^\oh_y(f^k)
\,\leq\, 
\frac{1}{K}\sum_{k=1}^K\sum_{y=1}^m \lambda_y^k {\phi}^\oh_y(f^k)
\,+\,
4BZ\sqrt{\frac{\log(m)}{{K}}} \,+\, 2\Delta^\oh(n^\val, \delta/2),
\]
and applying the first part of Lemma \ref{lem:helper-dro-2} to the RHS, we have with the same probability, for any $\tau \in \R_+$:
\begin{align*}
{
\max_{\lambda \in \Delta_m}\,\frac{1}{K}\sum_{k=1}^K\sum_{y=1}^m \lambda_y {\phi}^\oh_y(f^k)}
&\leq
\frac{1}{K}\sum_{k=1}^K\min_{f \in \cF}\sum_{y=1}^m \lambda_y^k {\phi}^\oh_y(f)
\,+\,
4BZ\sqrt{\frac{\log(m)}{{K}}} \,+\, 2\Delta^\oh(n^\val, \delta/2) 
\\
&
\hspace{1cm}
\,+\,  2\tau\Delta(n, \delta/2)
\,+\, 2B\max_{y \in [m]} \E_x\left[
        \left| \tau\frac{p^t_y(x)}{\pi^t_y} \,-\, \frac{\eta_y(x)}{\pi_y} \right|\right]
\\
&\leq
\min_{f \in \cF}\,\frac{1}{K}\sum_{k=1}^K\sum_{y=1}^m \lambda_y^k {\phi}^\oh_y(f)
\,+\,
4BZ\sqrt{\frac{\log(m)}{{K}}} \,+\, 2\Delta^\oh(n^\val, \delta/2)
\\
&
\hspace{1cm}
\,+\,  2\tau\Delta(n, \delta/2)
\,+\, 2B\max_{y \in [m]} \E_x\left[
        \left| \tau\frac{p^t_y(x)}{\pi^t_y} \,-\, \frac{\eta_y(x)}{\pi_y} \right|\right].
\end{align*}
Using the convexity of $\phi(\cdot)$ in $f(x)$ and Jensen's inequality, 
we have that $\sum_{y=1}^m \lambda_y {\phi}_y(\bar{f}^s) \leq \frac{1}{K}\sum_{k=1}^K\sum_{y=1}^m \lambda_y {\phi}_y(f^k)$.
We use this to further lower bound the LHS 
in terms of the averaged scoring function $\bar{f}^s(x) = \frac{1}{K}\sum_{k=1}^K f^k(x)$, and re-trace the steps in Theorem \ref{thm:dro} to get"
\begin{align*}
{\max_{y \in [m]}\, {\phi}^\oh_y(\bar{f}^s)}
    &\leq
        \min_{f \in \cF}\max_{y \in [m]}\, {\phi}^\oh_y(f)
        \,+\,
        4BZ\sqrt{\frac{\log(m)}{{K}}} \,+\, 2\Delta^\oh(n^\val, \delta/2)
        \nonumber
        \\
    &\hspace{1cm}
        \,+\,  2\tau\Delta(n, \delta/2)
\,+\, 2B\max_{y \in [m]} \E_x\left[
        \left| \tau\frac{p^t_y(x)}{\pi^t_y} \,-\, \frac{\eta_y(x)}{\pi_y} \right|\right].
\end{align*}
Noting that $L^\rob(f) = \max_{y \in [m]}\, {\phi}^\oh_y(f)$ completes the proof.
\end{proof}

\section{DRO for Traded-off Objective}
\label{app:dro-general-algo}
\begin{figure}
\begin{algorithm}[H]
\caption{Distilled Margin-based DRO for Traded-off Objective}
\label{algo:dro-general}
\begin{algorithmic}
\STATE \textbf{Inputs:} Teacher $p^t$, Student hypothesis class $\cF$, Training set $S$, Validation set $S^\val$, Step-size $\gamma \in \R_+$,
Number of iterations $K$, Loss $\ell$, Trade-off parameter $\alpha$
\STATE \textbf{Initialize:} Student $f^0 \in \cF$, Multipliers $\blambda^0 \in \Delta_m$
\STATE \textbf{For}~{$k = 0 $ to $K-1$}
\STATE ~~~$\tilde{\lambda}^{k+1}_j \,=\, \lambda^k_j\exp\big( \gamma \alpha \hat{R}_j \big), \forall j \in [m]$
~\text{where} $\hat{R}_j =$ $\displaystyle\frac{1}{n^\val}\frac{1}{\hat{\pi}^{t,\val}_j}\sum_{(x, y) \in S^\val} p_j^t(x_i)\, \ell( j , f^k(x) )$
\STATE ~~~$\lambda^{k+1}_y \,=\, \frac{\tilde{\lambda}^{k+1}_y}{\sum_{j=1}^m \tilde{\lambda}^{k+1}_j}, \forall y$
\STATE ~~~
    $\beta^{k+1}_y =\,(1 - \alpha)\frac{1}{m} \,+\, \alpha\lambda^{k+1}_y$
\STATE ~~~$f^{k+1} \,\in\, \displaystyle\Argmin{f \in \cF}\, \frac{1}{n}\sum_{i=1}^n \cL^\mar\left(p^t(x_i), f(x_i); \frac{\beta^{k+1}}{\hat{\pi}^t} \right)$
~~// Replaced with a few steps of SGD
\STATE \textbf{End For}
\STATE \textbf{Output:} $\bar{f}^{s}: x \mapsto \frac{1}{K}\sum_{k =1}^K f^k(x)$
\end{algorithmic}
\end{algorithm}
\end{figure}

We present a variant of the margin-based DRO algorithm described in Section \ref{sec:algorithms}
that seeks to minimize a trade-off between the balanced and robust student objectives:
$$\displaystyle \hat{L}^\tdfd(f^s) = (1-\alpha)\hat{L}^\bald(f^s) + \alpha\hat{L}^\robd(f^s),$$
for some $\alpha \in [0,1]$. 

Expanding this, we have:
\begin{align*}
    L^\tdfd(f)
         &= 
             (1-\alpha)\frac{1}{m}\sum_{y=1}^m \frac{1}{\hat{\pi}^t_y}\frac{1}{n}\sum_{i=1}^n p^t_y(x_i)\,\ell(y, f(x_i))
             \,+\,
             \alpha \max_{y \in [m]}\,\sum_{y=1}^m \frac{1}{\hat{\pi}^t_y}\frac{1}{n}\sum_{i=1}^n p^t_y(x_i)\,\ell(y, f(x_i))\\
        &= 
             (1-\alpha)\frac{1}{m}\sum_{y=1}^m \frac{1}{\hat{\pi}^t_y}\frac{1}{n}\sum_{i=1}^n p^t_y(x_i)\,\ell(y, f(x_i))
             \,+\,
             \alpha \max_{\lambda \in \Delta_m}\,\sum_{y=1}^m \frac{\lambda_y}{\hat{\pi}^t_y}\frac{1}{n}\sum_{i=1}^n p^t_y(x_i)\,\ell(y, f(x_i))\\
        &= 
             \max_{\lambda \in \Delta_m}\sum_{y=1}^m\left((1-\alpha)\frac{1}{m} + \alpha\lambda_y\right) \frac{1}{\hat{\pi}^t_y}\frac{1}{n}\sum_{i=1}^n p^t_y(x_i)\,\ell(y, f(x_i)).
\end{align*}
The minimization of $ L^\tdfd(f)$ over $f$ can then be a cast as a min-max problem:
\begin{align*}
    \min_{f:\X\>\R^m}\,L^\tdfd(f)
         &= 
             \min_{f:\X\>\R^m}\,\max_{\lambda \in \Delta_m}\sum_{y=1}^m\left((1-\alpha)\frac{1}{m} + \alpha\lambda_y\right) \frac{1}{\hat{\pi}^t_y}\frac{1}{n}\sum_{i=1}^n p^t_y(x_i)\,\ell(y, f(x_i)).
\end{align*}
Retracing the steps in the derivation of Algorithm \ref{algo:dro} in Section \ref{sec:algorithms}, we
have the following updates on $\lambda$ and $f$ to solve the above min-max problem:
\begin{align*}
    \tilde{\lambda}^{k+1}_y &= \lambda^k_y\exp\bigg( \gamma\alpha\frac{1}{n\hat{\pi}^t_y}\sum_{i=1}^n  p_y^t(x_i)\, \ell\left( y , f^k(x_i) \right) \bigg), \forall y\\
    \lambda^{k+1}_y &=\, \frac{\tilde{\lambda}^{k+1}_y}{\sum_{j=1}^m \tilde{\lambda}^{k+1}_j}, \forall y\\
    \beta^{k+1}_y &=\,(1 - \alpha)\frac{1}{m} \,+\, \alpha\lambda^{k+1}_y\\
    f^{k+1} &\in\, \Argmin{f \in \cF} \sum_{y \in [m]}\frac{\beta^{k+1}_y}{n\hat{\pi}^t_y}\sum_{i=1}^n  p_y^t(x_i)\, \ell\left( y , f(x_i) \right),
\end{align*}
for step-size parameter $\gamma > 0$. 
To better handle training of over-parameterized students, we will perform the updates on $\lambda$ using a held-out validation set,
and employ a margin-based surrogate for performing the minimization over $f$.  This procedure is outlined in Algorithm \ref{algo:dro-general}.

\section{Further experiment details}\label{app:experiment_details}
This section contains further experiment details and additional results. 
\begin{itemize}
\item Appendices \ref{app:datasets} through \ref{app:baselines} contain additional details about the datasets, hyperparameters, and baselines. 
\item Appendices \ref{app:tables} through \ref{app:tradeoff_plots} contain additional experimental comparisons with the AdaMargin and AdaAlpha baselines \cite{lukasik2021teachers} and group DRO \cite{Sagawa2020Distributionally}, and additional experimental results on CIFAR, TinyImageNet and ImageNet, along with additional trade-off plots.
\end{itemize}


\subsection{Additional details about datasets}\label{app:datasets}

\subsubsection{Building long tailed datasets}
The long-tailed datasets were created from the original datasets following \citet{cui2019class} by downsampling examples with an exponential decay in the per-class sizes. As done by \citet{narasimhan2021training}, we set
the imbalance ratio $\frac{\max_i P(y=i)}{\min_i P(y=i)}$ to 100 for CIFAR-10 and CIFAR-100, and to 83 for TinyImageNet
(the slightly smaller ratio is to ensure that the smallest class is of a reasonable size).
We use the long-tail version of ImageNet generated by \citet{liu2017sphereface}.

\subsubsection{Dataset splits}
The original test samples for  CIFAR-10, CIFAR-10-LT, CIFAR-100, CIFAR-100-LT, TinyImageNet (200 classes), TinyImageNet-LT (200 classes), and ImageNet (1000 classes)
are all balanced. Following \citet{narasimhan2021training}, we randomly split them in half and use half the samples as a validation set, and the other half as a test set. For the CIFAR and TineImageNet datasets, this amounts to using a validation set of size 5000. For the ImageNet dataset, we sample a subset of 5000 examples from the validation set each time we update the Lagrange multipliers in Algorithm \ref{algo:dro}.


In keeping with prior work \cite{menon2020long, narasimhan2021training, lukasik2021teachers}, we use the same validation and test sets for the long-tailed training sets as we do for the original versions. For the long tailed training sets, this simulates a scenario where the training data follows a long tailed distribution due to practical data collection limitations, but the test distribution of interest still comes from the original data distribution. In plots, the ``balanced accuracy'' that we report for the long-tail datasets (e.g., CIFAR-10-LT) is actually the standard accuracy calculated over the balanced test set, which is shared with the original balanced dataset (e.g., CIFAR-10).

Both teacher and student were always trained on the same training set. 

The CIFAR
datasets had images of size 32 $\times$ 32, while the TinyImageNet and ImageNet datasets 
dataset 
had images of size 224 $\times$ 224.

These datasets do not contain personally identifiable information or offensive content. The CIFAR-10 and CIFAR-100 datasets are licensed under the MIT License. The terms of access for ImageNet are given at \url{https://www.image-net.org/download.php}.


\subsection{Additional details about training and hyperparameters}
\label{app:setup-details}
\subsubsection{Code}
\label{app:code}
We have made our code available as a part of the supplementary material.

\subsubsection{Training details and hyperparameters}
Unless otherwise specified, the temperature hyperparameters were chosen to achieve the best worst-class accuracy on the validation set. All models were trained using SGD with momentum of 0.9 \cite{lukasik2021teachers, narasimhan2021training}. 

The learning rate schedule were chosen to mimic the settings in prior work \cite{narasimhan2021training, lukasik2021teachers}. 
 %
For CIFAR-10 and CIFAR-100 datasets, we ran the optimizer for 450 epochs, linearly warming up the learning rate till the 15th epoch, and then applied a step-size decay of 0.1 after the 200th, 300th and 400th epochs, as done by \citet{lukasik2021teachers}.
For the long-tail versions of these datasets, we ran trained for 256 epochs,
linearly warming up the learning rate till the 15th epoch, and then applied a step-size decay of 0.1 after the 96th, 192nd and 224th epochs, as done by \citet{narasimhan2021training}. Similarly, for the TinyImageNet datasets, we train for 200 epochs,  linearly warming up the learning rate till the 5th epoch, and then applying a decay of 0.1 after the 75th and 135th epochs, as done by \citet{narasimhan2021training}. For ImageNet, we train for 90 epochs, linearly warming up the learning rate till the 5th epoch, then applying a decay of 0.1 after the 30th, 60th and 80th epochs, as done by \citet{lukasik2021teachers}. 
We used a batch size of 128 for the CIFAR-10 datasets \cite{narasimhan2021training}, and a batch size of 1024 for the other datasets \cite{lukasik2021teachers}.

We apply an $L_2$ weight decay of $10^{-4}$ in all our SGD updates \cite{lukasik2021teachers}. This amounts to applying an \emph{$L_2$ regularization} on the model parameters, and has the effect of keeping the model parameters (and as a result the loss function) bounded. 

When training with the margin-based robust objective (see Algorithm \ref{algo:dro}), a  separate step size $\gamma$ was applied for training the main model function $f$, and for updating the multipliers $\lambda$. 
We set $\gamma$ to 0.1 in all experiments.



\emph{Hardware.} Model training was done using TPUv2. 

\subsubsection{Repeats}
For all comparative baselines without distillation (e.g. the first and second rows of Table \ref{tab:combos}), we provide average results over $m$ retrained models ($m=5$ for TinyImageNet, or $m=10$ for CIFAR datasets). For students on all CIFAR* datasets, unless otherwise specified, we train the teacher once and run the student training 10 times using the same fixed teacher. We compute the mean and standard error of metrics over these $m=10$ runs. For the resource-heavy TinyImageNet and ImageNet students, we reduce the number of repeats to $m=5$. This methodology captures variation in the student retrainings while holding the teacher fixed. To capture the end-to-end variation in both teacher and student training, we include Appendix \ref{app:tables} which contains a rerun of the CIFAR experiments in Table \ref{tab:combos} using a distinct teacher for each student retraining. The overall best teacher/student objective combinations either did not change or were statistically comparable to the previous best combination in Table \ref{tab:combos}.

\subsection{Additional details about algorithms and baselines}\label{app:baselines}

\subsubsection{Practical improvements to Algorithms \ref{algo:dro}--\ref{algo:dro-general}}
\label{app:post-hoc-details}
 Algorithms \ref{algo:dro}--\ref{algo:dro-general} currently return a scorer that averages over
 all $K$ iterates $\bar{f}^s(x) = \frac{1}{K}\sum_{k=1}^K f^k(x)$. While this averaging was required for our theoretical robustness guarantees to hold,
 in our experiments, we find it sufficient to simply return the last model $f^{K}$.
 Another practical improvement that we make  to these algorithms following \citet{cotter2019optimization}, is to employ the 0-1 loss while 
 performing updates on $\lambda$, i.e., set $\ell = \ell^\zo$ in the $\lambda$-update step. We are able to do this because
 the convergence of the exponentiated gradient updates on $\lambda$ does not depend on $\ell$ being differentiable. 
 This modification allows $\lambda$s to better reflect the model's per-class performance on the validation sample. 
 
 \subsubsection{Discussion on post-shifting baseline}
 We implement the post-shifting method in \citet{narasimhan2021training} (Algorithm 3 in their paper),
 which provides for an efficient way to construct a scoring function of the form $f^s_y(x) = \log(\gamma_y p^t_y(x)),$ for a fixed teacher $p^t$, where
the coefficients $\gamma \in \R^m$ are chosen to maximize the worst-class accuracy on the validation dataset. 
Interestingly, in our experiments, we find this approach
to do exceedingly well on the validation sample, but this does not always translate to good worst-class test performance. 
In contrast, some of the teacher-student combinations
that we experiment with were seen to over-fit less to the validation sample, and as a result were
able to generalize better to the test set. This could perhaps indicate that the
teacher labels we use in these combinations benefit the student in a way that it improves its generalization. 
The variance reduction effect that \citet{menon2021statistical} postulate may be one possible explanation 
for why we see this behavior.


\subsection{Different teachers on repeat trainings}\label{app:tables}

Most of the student results in the main paper in Table \ref{tab:combos} use the same teacher for all repeat trainings of the student. This captures the variance in the student training process while omitting the variance in the teacher training process. To capture the variance in the full training pipeline, we ran an additional set of experiments where students were trained on different retrained teachers, rather than on the same teacher. We report results on all CIFAR datasets in Table \ref{tab:teacher_var}. The best teacher/student combinations are identical for all datasets except for CIFAR-10-LT, for which the best teacher/student combinations from Table \ref{tab:teacher_var} were also not statistically significantly different from the best combination in Table \ref{tab:combos} ($L^{\bal}$/$L^{\robd}$ (one-hot val) vs. $L^{\bal}$/$L^{\robd}$ (teacher val) in Table \ref{tab:combos}). Note that the first and second rows of Table \ref{tab:combos} are already averaged over $m$ retrained teachers ($m=5$ for TinyImageNet, or $m=10$ for CIFAR datasets), and those same $m$ teachers are used in the repeat trainings in Table \ref{tab:teacher_var}.

\begin{table}[!h]
\caption{Comparison using different teachers for student retrainings for self-distilled teacher/student combos on test. For each student/teacher objective pair, we train $m=10$ students total on each of $m=10$ distinct retrained teachers. For comparability, the same set of $m$ teachers is used for each student. This differs from Table \ref{tab:combos} in that in Table \ref{tab:combos}, the students are retrained on each repeat using the same teacher (arbitrarily selected). Otherwise, setups are the same as in Table \ref{tab:combos}.
}
\label{tab:teacher_var}
\begin{center}
\begin{small}
\hspace{-2pt}
\begin{tabular}{p{0.1cm}c||c|c||c|c||}
\toprule
& & \multicolumn{2}{c||}{\textbf{CIFAR-10} Teacher Obj.} & \multicolumn{2}{c||}{\textbf{CIFAR-100}Teacher Obj.} \\
& & $L^{\std}$ & $L^{\rob}$ & $L^{\std}$ & $L^{\rob}$ \\
\midrule
\multirow{6}{*}{\rotatebox{90}{Student Obj.}}
& $L^{\stdd}$ & $87.09 \pm 0.51$  & $89.68 \pm 0.20$   & $44.21\pm0.57$  & \cellcolor{blue!25}$\mathbf{47.79}\pm0.82$ \\
& & \tiny{($93.78 \pm 0.22$)} & \tiny{($93.74 \pm 0.07$)} & \tiny{$74.6\pm0.11$} & \cellcolor{blue!25}\tiny{$73.48\pm0.11$}  \\
\cline{2-6}
& $L^{\robd}$  & \cellcolor{blue!25} $\textbf{90.62} \pm 0.19$ & $87.12 \pm 0.38$ & $39.7\pm1.32$  & $31.09\pm1.21$ \\
&\tiny{(teacher val)} & \cellcolor{blue!25}\tiny{($92.58 \pm 0.08$)} & \tiny{($90.46 \pm 0.08$)} & \tiny{($64.28\pm0.41$)} & \tiny{($55.39\pm0.28$)} \\
\cline{2-6}
& $L^{\robd}$ & $88.15 \pm 0.66$ & $86.44 \pm 0.52$ & $39.44\pm0.94$  & $39.65\pm0.59$ \\
&\tiny{(one-hot val)} & \tiny{($91.03 \pm 0.47$)} & \tiny{($90.16 \pm 0.42$)} & \tiny{($61.23\pm0.36$)} & \tiny{($60.89\pm0.29$)} \\
\bottomrule
\vspace{4pt}
\end{tabular}

\begin{tabular}{p{0.1cm}c||c|c|c||c|c|c||}
\toprule
& & \multicolumn{3}{c||}{\textbf{CIFAR-10-LT} Teacher Obj.} & \multicolumn{3}{c||}{\textbf{CIFAR-100-LT} Teacher Obj.} \\
& & $L^{\std}$ & $L^{\bal}$ & $L^{\rob}$ & $L^{\std}$ & $L^{\bal}$ & $L^{\rob}$ \\
\midrule
\multirow{6}{*}{\rotatebox{90}{Student Obj.}}
& $L^{\stdd}$ & $60.12 \pm 0.56$ & $66.13 \pm 0.47$ & $69.75\pm0.52$ & $0.00 \pm 0.00$ & $1.41 \pm 0.41$ & $9.17\pm0.74$ \\
&&\tiny{($77.39 \pm 0.10$)} & \tiny{($79.16 \pm 0.20$)} & \tiny{($80.73\pm0.08$)}  &\tiny{($45.84 \pm 0.13$)} & \tiny{($49.67\pm 0.20$)} & \tiny{($48.55\pm0.14$)}\\
\cline{2-8}
& $L^{\bald}$ & $72.41 \pm 0.52$ & $71.49 \pm 0.30$ & $71.70 \pm 0.33$ & $5.83 \pm 0.54$ & $5.94 \pm 0.50$ & $8.37 \pm 0.72$ \\
&& \tiny{($81.97 \pm 0.11$)} & \tiny{($81.20 \pm 0.15$)} & \tiny{($80.29 \pm 0.11$)} & \tiny{($50.58 \pm 0.15$)} & \tiny{($50.85 \pm 0.14$)} & \tiny{($48.16\pm 0.20$)}\\
\cline{2-8}
& $L^{\robd}$  & $62.77 \pm 0.58$  & $73.09\pm 0.34$ & $68.04 \pm 0.47$  & $10.53 \pm 0.76$  & $12.04 \pm 0.89$ & $9.66 \pm 1.15$ \\
& \tiny{(teacher val)} & \tiny{($77.18 \pm 0.15$)} & \tiny{($80.03\pm 0.22$)} & \tiny{($75.36 \pm 0.25$)} & \tiny{($33.69 \pm 0.14$)} & \tiny{($34.08 \pm 0.12$)} & \tiny{($37.10 \pm 0.15$)}\\
\cline{2-8}
& $L^{\robd}$  & \cellcolor{blue!25}$\mathbf{75.10} \pm 0.36$ & \cellcolor{blue!25}$\mathbf{75.10} \pm 0.50$ & $74.16 \pm 0.34$ & $10.74 \pm 0.44$ & $11.95 \pm 0.69$ & \cellcolor{blue!25} $\mathbf{12.87}\pm 0.81$ \\
& \tiny{(one-hot val)} & \cellcolor{blue!25}\tiny{($79.27 \pm 0.13$)} & \cellcolor{blue!25}\tiny{($79.07 \pm 0.20$)} & \tiny{($78.11 \pm 0.14$} & \tiny{($30.36 \pm 0.39$)} & \tiny{($31.00 \pm 0.16$)} & \cellcolor{blue!25}\tiny{($31.62 \pm 0.34$}\\
\bottomrule
\end{tabular}
\end{small}
\end{center}
\end{table}

\subsection{AdaAlpha and AdaMargin comparisons}

We include and discuss additional comparisons to the AdaMargin and AdaAlpha methods \cite{lukasik2021teachers}. These methods each define additional ways to modify the student's loss function (see Section \ref{sec:expts}). Table \ref{tab:ada*} shows results for these under the same self-distillation setup as in Table \ref{tab:combos}. For the balanced datasets, AdaMargin was competitive with the robust and standard students: on CIFAR-100 and TinyImageNet, AdaMargin combined with the robust teacher and the standard teacher (respectively) achieved worst-class accuracies that are statistically comparable to the best worst-class accuracies in Table \ref{tab:combos} and Table \ref{tab:teacher_var}. However, on the long-tailed datasets CIFAR-10-LT, CIFAR-100-LT, and TinyImageNet-LT (Table \ref{tab:combos_tinlt} below), AdaAlpha and AdaMargin did not achieve worst-class accuracies as high as other teacher/student combinations. This suggests that the AdaMargin method can work well on balanced datasets in combination with a robust teacher, but other combinations of standard/balanced/robust objectives are valuable for long-tailed datasets. Overall, AdaMargin achieved higher worst-class accuracies than AdaAlpha. 

\begin{table*}[!h]
\caption{Results for AdaAlpha and AdaMargin baselines for self-distilled teacher/student combos on test. Worst-class accuracy shown above, standard accuracy shown in parentheses  for the top three datasets, and balanced accuracy shown in parentheses for the bottom two long-tail datasets. Mean and standard error are reported over 3 repeats for all datasets.}
\label{tab:ada*}
\begin{center}
\begin{small}
\begin{tabular}{p{0.1cm}c||c|c||c|c||c|c||}
\toprule
& & \multicolumn{2}{c||}{\textbf{CIFAR-10} Teacher Obj.} & \multicolumn{2}{c||}{\textbf{CIFAR-100} Teacher Obj.} & \multicolumn{2}{c||}{\textbf{TinyImageNet} Teacher Obj.} \\
& & $L^{\std}$ & $L^{\rob}$ & $L^{\std}$ & $L^{\rob}$ & $L^{\std}$ & $L^{\rob}$ \\
\midrule
& Ada & $88.33\pm0.14$  & $89.96\pm0.44$ & $43.50\pm 0.62$  & $45.59\pm 0.82$ & $11.11 \pm 1.29$ & $16.58\pm1.67$ \\
& Alpha & \tiny{($94.31\pm0.01$)} & \tiny{($93.97\pm0.07$)} & \tiny{$73.96\pm 0.09$} & \tiny{$71.42\pm 0.14$} & \tiny{$61.13\pm0.09$} & \tiny{$56.84\pm0.15$} \\
\cline{2-8}
& Ada & $87.36$ \tiny{$\pm0.06$}  & $90.37$\tiny{$\pm0.26$} & $43.91$ \tiny{$\pm 1.11$}  & \cellcolor{blue!25} $\mathbf{47.78}$ \tiny{$\pm 0.96$} & \cellcolor{blue!25}$\mathbf{18.17}$ \tiny{$\pm 3.89$}  & $17.84$ \tiny{$\pm 1.77$}\\
& Margin & \tiny{($94.25\pm0.02$)} & \tiny{($94.02\pm0.12$)} & \tiny{($73.58\pm 0.11$)} & \cellcolor{blue!25}\tiny{($70.92\pm 0.09$)} & \cellcolor{blue!25}\tiny{($61.3 \pm 0.28$)} & \tiny{($55.77 \pm 0.32$)}\\
\bottomrule
\end{tabular}

\begin{tabular}{p{0.1cm}c||c|c|c||c|c|c||}
\toprule
& & \multicolumn{3}{c||}{\textbf{CIFAR-10-LT} Teacher Obj.} & \multicolumn{3}{c||}{\textbf{CIFAR-100-LT} Teacher Obj.} \\
& & $L^{\std}$ & $L^{\bal}$ & $L^{\rob}$ & $L^{\std}$ & $L^{\bal}$ & $L^{\rob}$ \\
\midrule
& Ada & $41.90\pm0.44$ & $66.23\pm0.39$ & $71.17\pm0.32$  & $0.00 \pm 0.00$ & $1.46 \pm 0.61$ & $9.15 \pm 0.54$ \\
& Alpha & \tiny{($71.67\pm0.08$)} & \tiny{($77.87\pm0.16$)} & \tiny{($79.66\pm0.13$)}   & \tiny{($42.52 \pm 0.08$)} & \tiny{($45.44 \pm 0.14$)} & \tiny{($45.64 \pm 0.11$)}\\
\bottomrule
\end{tabular}
\end{small}
\end{center}
\end{table*}

\subsection{TinyImageNet-LT worst-10 accuracy results}
In this section we provide results for TinyImageNet-LT. We report worst-10 accuracies in Table \ref{tab:combos_tinlt}, where the \textbf{worst-$k$ accuracy} is defined to be the average of the $k$ lowest per-class accuracies. We report worst-10 accuracies since the worst-class accuracy (as reported in Table \ref{tab:combos}) for TinyImageNet-LT was usually close to 0.

\begin{table}[!ht]
\caption{TinyImageNet-LT comparison of self-distilled teacher/student combos on test. Average \textbf{worst-10} accuracy shown above, balanced accuracy shown in parentheses below. The combination with the best worst-class accuracy is bolded.}
\label{tab:combos_tinlt}
\begin{center}
\begin{small}
\begin{tabular}{p{0.1cm}c||c|c|c||}
\toprule
& & \multicolumn{3}{c||}{\textbf{TinyImageNet-LT} Teacher Obj.} \\
& & $L^{\std}$ & $L^{\bal}$ & $L^{\rob}$ \\
\midrule
\multirow{10}{*}{\rotatebox{90}{Student Obj.}} 
& none & $0.00 \pm 0.00$ & $0.00 \pm 0.00$ & $1.00 \pm 0.69$ \\
&& \tiny{($26.15 \pm 0.21$)} & \tiny{($29.19 \pm 0.22$)} & \tiny{($22.73 \pm 0.25$)} \\
\cline{2-5}
& Post &  $0.60 \pm 0.28$ & $0.56 \pm 0.32$ &  $0.07 \pm 0.06$ \\
& shift & \tiny{($15.99 \pm 0.39$)} & \tiny{($17.00 \pm 0.23$)} & \tiny{($16.73 \pm 0.46$)} \\
\cline{2-5}
& Ada &  $0.00 \pm 0.00$ & $0.00 \pm 0.00$ &  $0.00 \pm 0.00$ \\
& Alpha & \tiny{($28.14 \pm 0.12$)} & \tiny{($0.50 \pm 0.00$)} & \tiny{($0.50 \pm 0.00$)} \\
\cline{2-5}
& Ada &  $0.00 \pm 0.00$ & $0.00 \pm 0.00$ &  $0.41 \pm 0.17$ \\
& Margin & \tiny{($9.18 \pm 0.09$)} & \tiny{($7.92 \pm 0.10$)} & \tiny{($23.08 \pm 0.15$)} \\
\cline{2-5}
& $L^{\stdd}$ & $0.0 \pm 0.0$ & $0.0 \pm 0.0$ & $1.87 \pm 0.52$ \\
&&\tiny{($26.05 \pm 0.40$)} & \tiny{($28.76	\pm 0.37$)} & \tiny{($25.34 \pm	0.29$)} \\
\cline{2-5}
& $L^{\bald}$ & $0.2 \pm 0.4$ &  $0.71 \pm 0.5$ & $\cellcolor{blue!25}{\bf 2.05} \pm 0.55$ \\
&& \tiny{($30.43 \pm 0.14$)} & \tiny{($30.41 \pm 0.4$)} & \cellcolor{blue!25}\tiny{($26.08 \pm 0.28$)}\\
\cline{2-5}
& $L^{\robd}$  & $0.0 \pm 0.0$  & $0.5 \pm	0.36$ & $0.0 \pm 0.0$\\
& \tiny{(teacher val)} & \tiny{($22.67 \pm 0.23$ )} & \tiny{($23.79 \pm	0.41$)} & \tiny{($17.24	\pm 0.34$)} \\
\cline{2-5}
& $L^{\robd}$  & $1.04 \pm 0.51$ & $1.46 \pm 0.66$ & $1.22 \pm	0.79$\\
& \tiny{(one-hot val)} & \tiny{($16.85 \pm 0.23$)} & \tiny{($19.81 \pm	0.34$)} & \tiny{($16.3	\pm 0.13$)}\\
\bottomrule
\end{tabular}
\end{small}
\end{center}
\vskip -0.1in
\end{table}

\subsection{Group DRO comparison}

\citet{Sagawa2020Distributionally} propose a group DRO algorithm to improve long tail performance without distillation. In this section we present an experimental comparison to Algorithm 1 from \citet{Sagawa2020Distributionally}. This differs from our robust optimization methodology in Section \ref{sec:algorithms} in two key ways: \textit{(i)} we apply a margin-based surrogates of \citet{menon2020long}, and \textit{(ii)} we use a validation set to update the Lagrange multipliers $\lambda$ in Algorithm \ref{algo:dro-general}. Table \ref{tab:groupdro} shows results from running group DRO directly as specified in Algorithm 1 in \citet{Sagawa2020Distributionally}, as well as a variant where we use the validation set to update Lagrange multipliers in group DRO (labeled as ``with vali'' in the table).
This comparison shows that $L^{\rob}$ is comparable to group DRO, and that robust distillation protocols can outperform group DRO alone. 

\begin{table}[!ht]
\caption{Results from comparison to group DRO (Algorithm 1 in \citet{Sagawa2020Distributionally}) without distillation. ``No vali'' uses the training set to update group lagrange multipliers, as done originally by \citet{Sagawa2020Distributionally}. ``With vali'' uses the validation set to compute group Lagrange multipliers as done in all other experiments in our paper. Worst-class accuracy is shown above, and balanced accuracy is shown in parentheses below.}
\label{tab:groupdro}
\begin{center}
\begin{small}
\begin{tabular}{||c|c||c|c||c|c||}
\toprule
\multicolumn{2}{||c||}{\textbf{CIFAR-10} group DRO} & \multicolumn{2}{c||}{\textbf{CIFAR-100} group DRO} & \multicolumn{2}{c||}{\textbf{TinyImageNet} group DRO} \\
 No vali & With vali & No vali & With vali & No vali & With vali \\
\midrule
 $86.65$ \tiny{$\pm 0.49$}  & $89.32 $ \tiny{$\pm 0.21$} &  $40.35 $ \tiny{$\pm 1.18$}  & $43.89 $ \tiny{$\pm 1.12$}  & $0.00 $ \tiny{$\pm 0.00$}  & $9.17 $ \tiny{$\pm 1.55$}\\
 \tiny{($93.61 \pm 0.09$)} & \tiny{($92.34 \pm 0.07$)} & \tiny{$70.25 \pm 0.17$} & \tiny{$65.18 \pm 0.08$} & \tiny{($6.55 \pm 0.41$)} & \tiny{($47.67 \pm 0.22$)}\\
\bottomrule
\end{tabular}

\begin{tabular}{||c|c||c|c||c|c||}
\toprule
\multicolumn{2}{||c||}{\textbf{CIFAR-10-LT} group DRO} & \multicolumn{2}{c||}{\textbf{CIFAR-100-LT} group DRO} & \multicolumn{2}{c||}{\textbf{TinyImageNet-LT} group DRO} \\
No vali & With vali & No vali & With vali & No vali & With vali \\
\midrule
$51.59 $ \tiny{$\pm 2.49$} & $59.93 $ \tiny{$\pm 0.59$} &   $0.00 $ \tiny{$\pm 0.00$} & $0.19 $ \tiny{$\pm 0.17$}  &   $0.00 $ \tiny{$\pm 0.00$} & $0.00 $ \tiny{$\pm 0.00$}\\
\tiny{($71.94 \pm 0.75$)} & \tiny{($74.39 \pm 0.17$)}  & \tiny{($39.81 \pm 0.23$)} & \tiny{($40.47 \pm 0.17$)}& \tiny{($9.79 \pm 0.40$)} & \tiny{($22.49 \pm 0.10$)}\\
\bottomrule
\end{tabular}

\end{small}
\end{center}
\end{table}



\subsection{CIFAR compressed students}
To supplement the self-distilled results in Table \ref{tab:combos}, we provide  Table \ref{tab:combos_32} which gives results when distilling a larger teacher to a smaller student. As in Table \ref{tab:combos}, the combination with the best worst-class accuracy involved a robust student for CIFAR-10, CIFAR-10-LT, and CIFAR-100. 

\begin{table*}[!h]
\caption{Comparison of ResNet-56$\to$ResNet-32 distilled teacher/student combos on test on CIFAR datasets. Compared to the standard distillation baseline of teacher trained with $L^{\std}$ and student trained with $L^{\stdd}$ (top left corner), the introduction of a robust objective achieves better worst-class accuracy. Worst-class accuracy shown above, standard accuracy shown in parentheses  for the top three datasets, and balanced accuracy shown in parenthesis for the bottom two long-tail datasets. The combination with the best worst-class accuracy is bolded. Mean and standard error are reported over repeats (10 repeats for all except CIFAR-100, which has 3 repeats). We include results for the robust student using either a teacher labeled validation set (``teacher val''), or true one-hot class labels in the validation set (``one-hot val''), as outlined in Section \ref{sec:algorithms}.}
\label{tab:combos_32}
\begin{center}
\begin{small}
\begin{tabular}{p{0.1cm}p{0.1cm}c||c|c||c|c||}
\toprule
& & & \multicolumn{2}{c||}{\textbf{CIFAR-10} Teacher Obj.} & \multicolumn{2}{c||}{\textbf{CIFAR-100} Teacher Obj.} \\
& & & $L^{\std}$ & $L^{\rob}$ & $L^{\std}$ & $L^{\rob}$ \\
\midrule
\multirow{6}{*}{\rotatebox{90}{Student Obj.}} &\multirow{6}{*}{\rotatebox{90}{(ResNet-32)}}
& $L^{\stdd}$ & $86.4 \pm 0.27$  & $89.56 \pm 0.20$  & $41.82\pm1.12$  & \cellcolor{blue!25}$\mathbf{45.7}\pm1.13$ \\
&& & \tiny{($93.73 \pm 0.05$)} & \tiny{($93.38 \pm 0.05$)} & \tiny{$73.19\pm0.10$} & \cellcolor{blue!25}\tiny{$71.42\pm0.22$}  \\
\cline{3-7}
&& $L^{\robd}$  & \cellcolor{blue!25} $\mathbf{89.61} \pm 0.27$ & $83.8 \pm 0.95$ & $38.94\pm2.61$  & $19.15\pm0.00$ \\
&&\tiny{(teacher val)} & \cellcolor{blue!25}\tiny{($92.20 \pm 0.08$)} & \tiny{($88.71 \pm 0.24$)} & \tiny{($62.28\pm0.40$)} & \tiny{($52.9\pm0.00$)} \\
\cline{3-7}
&& $L^{\robd}$ & $87.92 \pm 0.23$ & $86.57 \pm 0.24$ & $33.19\pm1.29$  & $41.23\pm0.84$ \\
&&\tiny{(one-hot val)} & \tiny{($90.89 \pm 0.12$)} & \tiny{($90.54 \pm 0.11$)} & \tiny{($57.43\pm0.29$)} & \tiny{($61.14\pm0.24$)} \\
\bottomrule
\end{tabular}

\begin{tabular}{p{0.1cm}p{0.1cm}c||c|c|c||c|c|c||}
\toprule
&& & \multicolumn{3}{c||}{\textbf{CIFAR-10-LT} Teacher Obj.} & \multicolumn{3}{c||}{\textbf{CIFAR-100-LT} Teacher Obj.} \\
&& & $L^{\std}$ & $L^{\bal}$ & $L^{\rob}$ & $L^{\std}$ & $L^{\bal}$ & $L^{\rob}$ \\
\midrule
\multirow{6}{*}{\rotatebox{90}{Student Obj.}} &\multirow{6}{*}{\rotatebox{90}{(ResNet-32)}} 
& $L^{\stdd}$ & $57.23 \pm 0.53$ & $66.80 \pm 0.25$ & $72.36\pm0.39$ & $0.00 \pm 0.00$ & $1.38 \pm 0.39$ & $7.99\pm0.48$ \\
&&&\tiny{($75.76 \pm 0.12$)} & \tiny{($78.99 \pm 0.06$)} & \tiny{($80.74\pm0.09$)}  &\tiny{($44.33 \pm 0.11$)} & \tiny{($47.28\pm0.13$)} & \tiny{($47.34\pm0.08$)}\\
\cline{3-9}
&& $L^{\bald}$ & $71.37 \pm 0.50$ & $71.00 \pm 0.45$ & $72.17 \pm 0.40$ & $3.57 \pm 0.58$ & $4.28 \pm 0.45$ & $5.58 \pm 0.53$ \\
&&& \tiny{($81.13 \pm 0.12$)} & \tiny{($81.12 \pm 0.15$)} & \tiny{($79.91 \pm 0.08$)} & \tiny{($49.21 \pm 0.10$)} & \tiny{($46.56 \pm 0.13$)} & \tiny{($48.58\pm 0.09$)}\\
\cline{3-9}
&& $L^{\robd}$  & $64.1 \pm 0.36$  & $73.51\pm 0.33$ & $69.90 \pm 0.42$  & $10.24 \pm 0.71$  & \cellcolor{blue!25}$\mathbf{13.41} \pm 0.72$ & $11.27 \pm 0.61$ \\
&& \tiny{(teacher val)} & \tiny{($76.34 \pm 0.12$)} & \tiny{($80.10\pm 0.10$)} & \tiny{($76.37 \pm 0.14$)} & \tiny{($33.55 \pm 0.16$)} & \cellcolor{blue!25}\tiny{($33.37 \pm 0.17$)} & \tiny{($36.14 \pm 0.19$)}\\
\cline{3-9}
&& $L^{\robd}$  & $72.65 \pm 0.27$ & $74.39 \pm 0.34$ & \cellcolor{blue!25}$\mathbf{74.45} \pm 0.26$ & $10.93 \pm 0.65$ & $12.2 \pm 0.65$ & $12.93\pm 0.62$ \\
&& \tiny{(one-hot val)} & \tiny{($77.69 \pm 0.11$)} & \tiny{($78.68 \pm 0.16$)} & \cellcolor{blue!25}\tiny{($77.97 \pm 0.10$} & \tiny{($29.48 \pm 0.22$)} & \tiny{($30.27 \pm 0.18$)} & \tiny{($31.83 \pm 0.17$}\\
\bottomrule
\end{tabular}
\end{small}
\end{center}
\vskip -0.12in
\end{table*}

\subsection{ImageNet comparisons}\label{app:imgnet}
We trained ResNet-34 teachers on ImageNet and report results of distillation to ResNet-34 and ResNet-18 students in Table \ref{tab:combos_imagenet}. 
In Table \ref{tab:combos_imagenetlt},
we report results of
distilling a ResNet-34 teacher to a ResNet-34 student and to a ResNet-18 student on ImageNet-LT.

\begin{table}[!ht]
\caption{ImageNet comparison of teacher/student combos on test. The left panel presents the result of \emph{self distillation} and the right panel presents \emph{model compression}. Average worst-50 (worst-5\%) accuracy shown above, standard accuracy shown in parentheses below. The combination with the best worst-class accuracy is bolded. The first two rows are teacher metrics and are the same between the two experiments. Mean and standard error are reported over 5 repeats.}
\label{tab:combos_imagenet}
\begin{center}
\begin{small}
\begin{tabular}{p{0.1cm}c||c|c||}
\multicolumn{4}{c}{ResNet-34 Teacher to ResNet-34 Student} \\
\toprule
& & \multicolumn{2}{c||}{\textbf{ImageNet} Teacher Obj.} \\
&  & $L^{\std}$ & $L^{\rob}$ \\
\midrule
\multirow{8}{*}{\rotatebox{90}{Student Obj.}} 
& none & $28.08$ & $27.94$ \\
&& \tiny{($70.82$)} & \tiny{($66.62$)} \\
\cline{2-4}
& Post &  $26.10$ & $19.68$  \\
& shift & \tiny{($55.85 $)} & \tiny{($45.61 $)}\\
\cline{2-4}
& $L^{\stdd}$ & $24.43 \pm 0.00$ &  $25.23 \pm 0.06$\\
&& \tiny{($67.74 \pm 0.00$)} & \tiny{($67.05 \pm 0.05$)} \\
\cline{2-4}
& $L^{\robd}$  &\cellcolor{blue!25} $\mathbf{26.63} \pm 0.11$  & $24.18 \pm 0.24$\\
& \tiny{(one-hot val)} &\cellcolor{blue!25} \tiny{($63.50 \pm 0.01$ )} & \tiny{($60.68 \pm	0.14$)}\\
\bottomrule
\end{tabular}\quad
\begin{tabular}{p{0.1cm}c||c|c||}
\multicolumn{4}{c}{ResNet-34 Teacher to ResNet-18 Student} \\
\toprule
& & \multicolumn{2}{c||}{\textbf{ImageNet} Teacher Obj.} \\
&  & $L^{\std}$ & $L^{\rob}$ \\
\midrule
\multirow{8}{*}{\rotatebox{90}{Student Obj.}} 
& none & $28.08$ & $27.94$ \\
&& \tiny{($70.82$)} & \tiny{($66.62$)} \\
\cline{2-4}
& Post &  $26.10$ & $19.68$  \\
& shift & \tiny{($55.85 $)} & \tiny{($45.61 $)}\\
\cline{2-4}
& $L^{\stdd}$ & $20.64 \pm 0.24$ &\cellcolor{blue!25}  $\mathbf{23.91} \pm 0.38$\\
&& \tiny{($64.98 \pm 0.06$)} &\cellcolor{blue!25} \tiny{($60.24 \pm 0.07$)} \\
\cline{2-4}
& $L^{\robd}$  & $22.44 \pm 0.20$  & $21.62 \pm 0.22$\\
& \tiny{(one-hot val)} & \tiny{($64.52 \pm 0.10$ )} & \tiny{($57.96 \pm	0.06$)}\\
\bottomrule
\end{tabular}

\end{small}

\end{center}
\end{table}

\begin{table}[!ht]
\caption{ImageNet-LT comparison of teacher/student combos on test. The left panel presents the result of \emph{self distillation} and the right panel presents \emph{model compression}. Average worst-50 (worst-5\%) accuracy shown above, balanced accuracy shown in parentheses below. The combination with the best worst-class accuracy is bolded. The first two rows are teacher metrics and are the same between the two experiments. Mean and standard error are reported over 5 repeats. }
\label{tab:combos_imagenetlt}
\begin{center}
\begin{small}
\begin{tabular}{p{0.1cm}c||c|c||}
\multicolumn{4}{c}{ResNet-34 Teacher to ResNet-34 Student} \\
\toprule
& & \multicolumn{2}{c||}{\textbf{ImageNet-LT} Teacher Obj.} \\
&  & $L^{\bal}$ & $L^{\rob}$ \\
\midrule
\multirow{8}{*}{\rotatebox{90}{Student Obj.}} 
& none & $1.58$ & $2.67$ \\
&& \tiny{($43.26$)} & \tiny{($38.46$)} \\
\cline{2-4}
& Post &  $2.40$ & $0.87$  \\
& shift & \tiny{($19.49$)} & \tiny{($16.52$)}\\
\cline{2-4}
& $L^{\bald}$ & $3.53 \pm 0.11$ &  $4.85 \pm 0.15$\\
&& \tiny{($43.98 \pm 0.10$)} & \tiny{($43.89 \pm 0.03$)} \\
\cline{2-4}
& $L^{\robd}$  &\cellcolor{blue!25} $\mathbf{6.00} \pm 0.06$  & $4.31 \pm 0.08$\\
& \tiny{(one-hot val)} &\cellcolor{blue!25} \tiny{($39.17 \pm 0.14$ )} & \tiny{($36.12 \pm	0.14$)}\\
\bottomrule
\end{tabular}\quad
\begin{tabular}{p{0.1cm}c||c|c||}
\multicolumn{4}{c}{ResNet-34 Teacher to ResNet-18 Student} \\
\toprule
& & \multicolumn{2}{c||}{\textbf{ImageNet-LT} Teacher Obj.} \\
&  & $L^{\bal}$ & $L^{\rob}$ \\
\midrule
\multirow{8}{*}{\rotatebox{90}{Student Obj.}} 
& none & $1.58$ & $2.67$ \\
&& \tiny{($43.26$)} & \tiny{($38.46$)} \\
\cline{2-4}
& Post &  $2.40$ & $0.87$  \\
& shift & \tiny{($19.49$)} & \tiny{($16.52$)}\\
\cline{2-4}
& $L^{\bald}$ & $2.41 \pm 0.09$ &  $3.93 \pm 0.23$\\
&& \tiny{($42.35 \pm 0.07$)} & \tiny{($42.12 \pm 0.05$)} \\
\cline{2-4}
& $L^{\robd}$  &\cellcolor{blue!25} $\mathbf{5.02} \pm 0.20$  & $3.69 \pm 0.14$\\
& \tiny{(one-hot val)} &\cellcolor{blue!25} \tiny{($37.58 \pm 0.11$ )} & \tiny{($35.10 \pm	0.09$)}\\
\bottomrule
\end{tabular}
\end{small}

\end{center}
\vskip -0.1in
\end{table}

\subsection{Additional tradeoff plots}\label{app:tradeoff_plots}

Supplementing the tradeoff plot in the main paper (Figure \ref{fig:tradeoffs_cf-lt_32_onehot}), we provide the following set of plots illustrating the tradeoff between worst-case accuracy and balanced accuracy for different student architectures:
\begin{itemize}
    \item Long-tail datasets: self distillation (Figures \ref{fig:tradeoffs_cf-lt_56_onehot}, \ref{fig:tradeoffs_cf-lt_56_teacher}), and compressed student (Figures \ref{fig:tradeoffs_cf-lt_32_onehot}, \ref{fig:tradeoffs_cf_32_teacher}).
    \item Original balanced datasets: self distillation (Figures \ref{fig:tradeoffs_cf_56_onehot}, \ref{fig:tradeoffs_cf_56_teacher}), and compressed student (Figures \ref{fig:tradeoffs_cf_32_onehot}, \ref{fig:tradeoffs_cf_32_teacher}).
\end{itemize}

\begin{figure*}[!ht]
\begin{center}
\begin{tabular}{ccc}
    & \tiny{CIFAR-10-LT (ResNet-56 $\to$ ResNet-56)} & \tiny{CIFAR-100-LT (ResNet-56 $\to$ ResNet-56)} \\
    \includegraphics[trim={2.8cm 0 2.8cm 0},clip,width=0.15\textwidth]{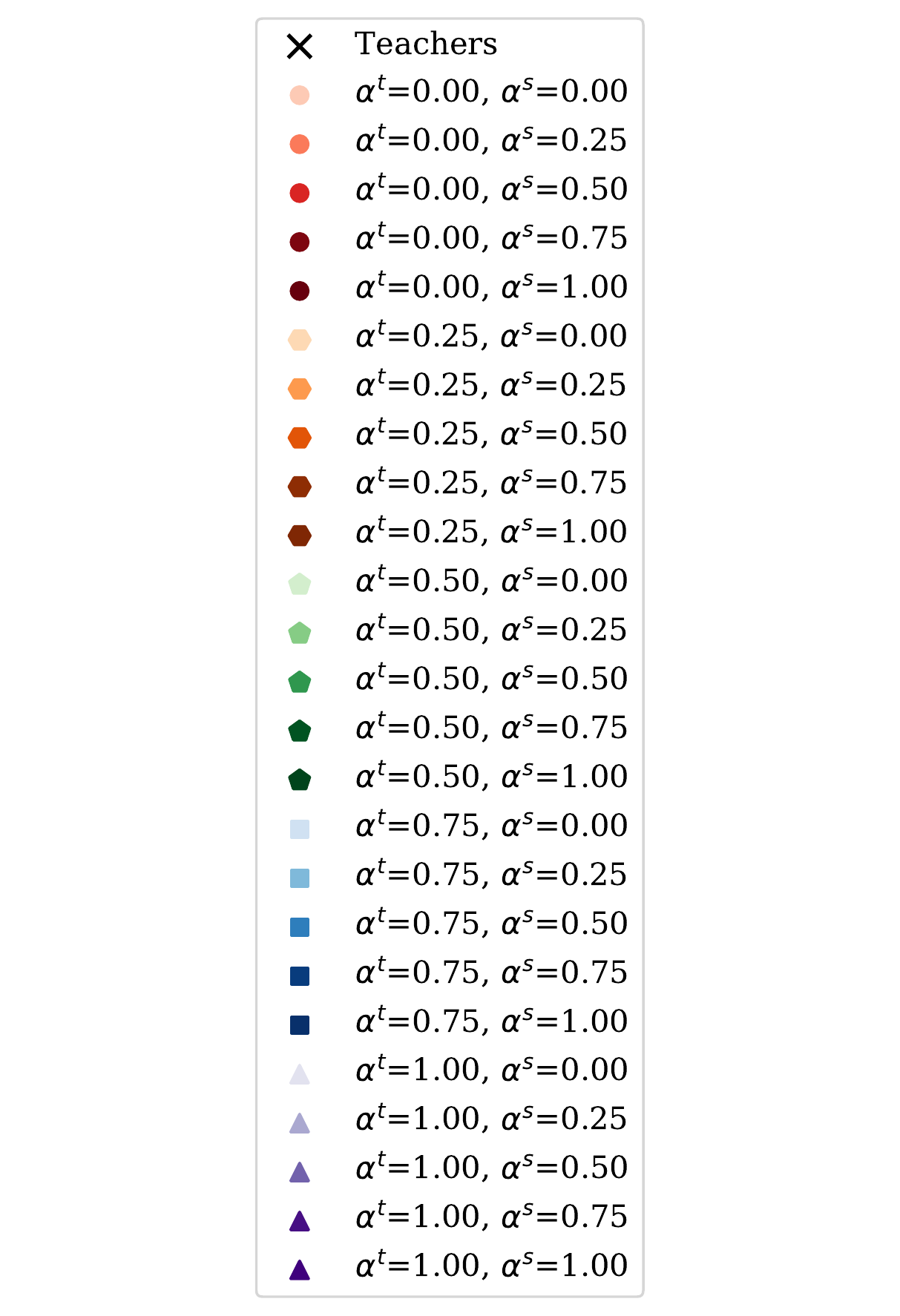} & \includegraphics[width=0.4\textwidth]{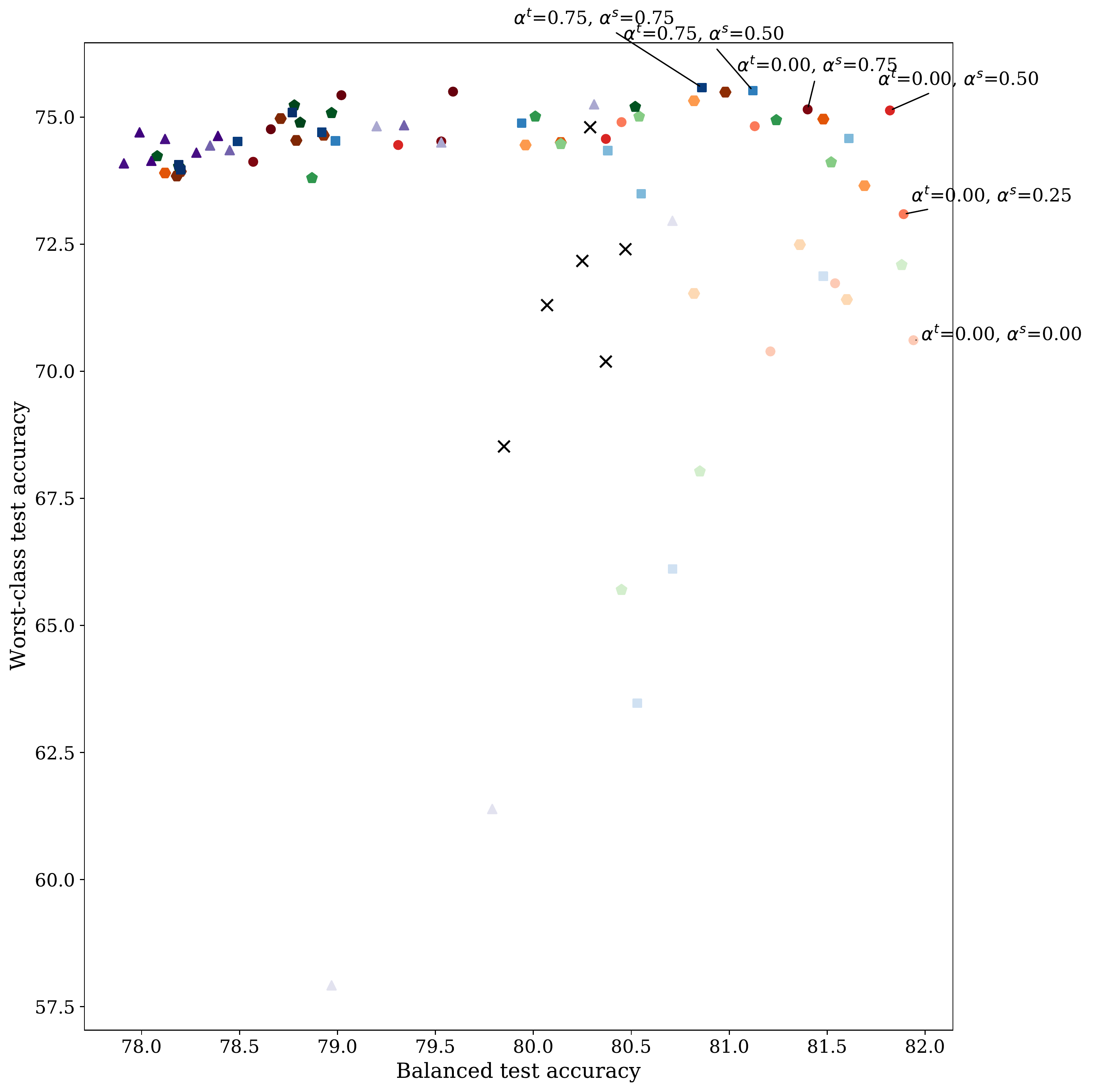} & \includegraphics[width=0.4\textwidth]{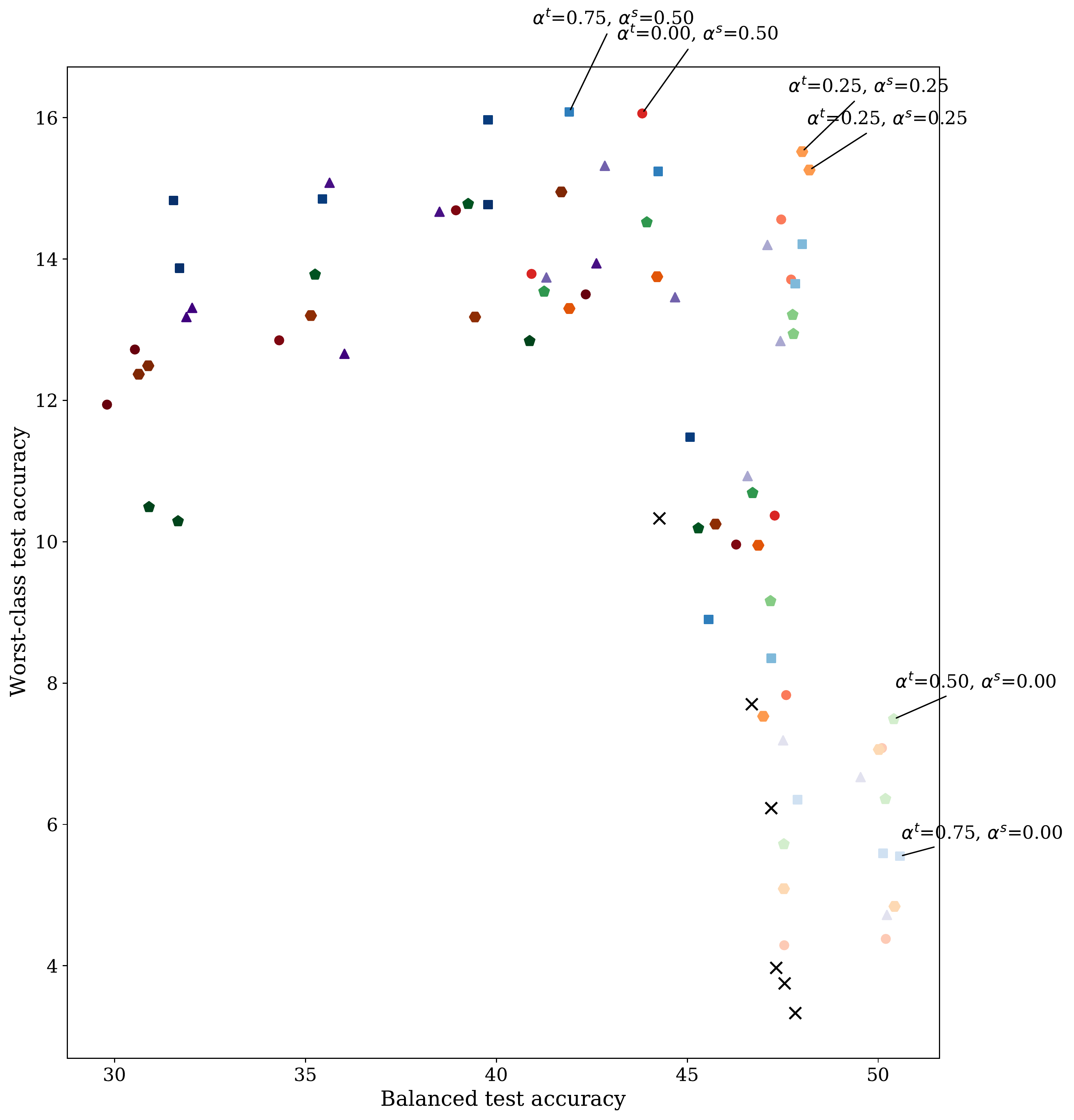}
\end{tabular}
\caption{Tradeoffs in worst-class test accuracy vs. average test accuracy for CIFAR-10-LT and CIFAR-100-LT, with students trained using architecture ResNet-56 and \textbf{one-hot} validation labels. Mean test performance over 10 runs are shown, and all temperatures are included.}
\label{fig:tradeoffs_cf-lt_56_onehot}
\end{center}
\vskip -0.2in
\end{figure*}

\begin{figure*}[!ht]
\begin{center}
\addtolength{\tabcolsep}{-5pt} 
\begin{tabular}{ccc}
    & \tiny{CIFAR-10-LT (ResNet-56 $\to$ ResNet-56)} & \tiny{CIFAR-100-LT (ResNet-56 $\to$ ResNet-56)} \\
    \includegraphics[trim={2.8cm 0 2.8cm 0},clip,width=0.15\textwidth]{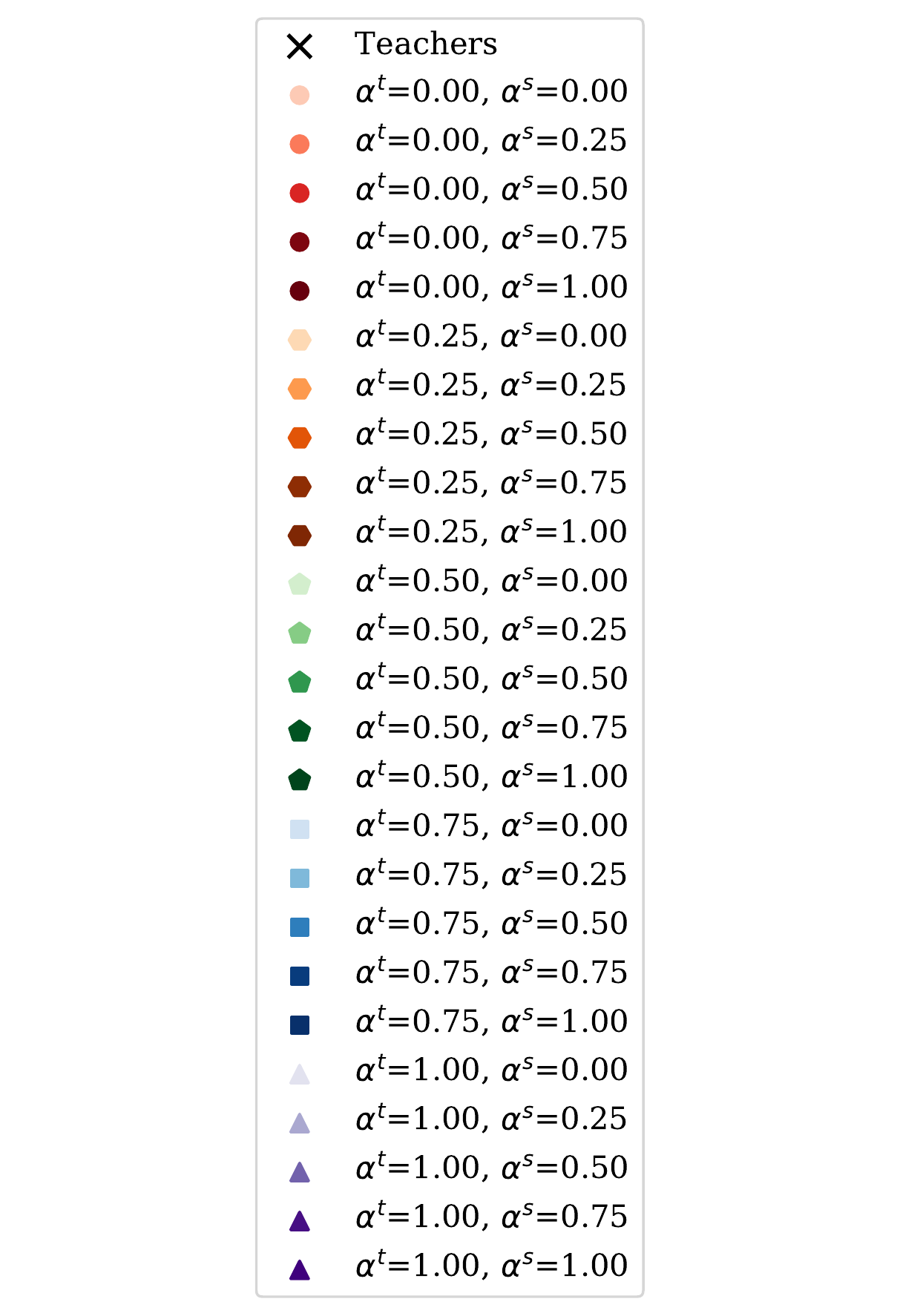} & \includegraphics[width=0.4\textwidth]{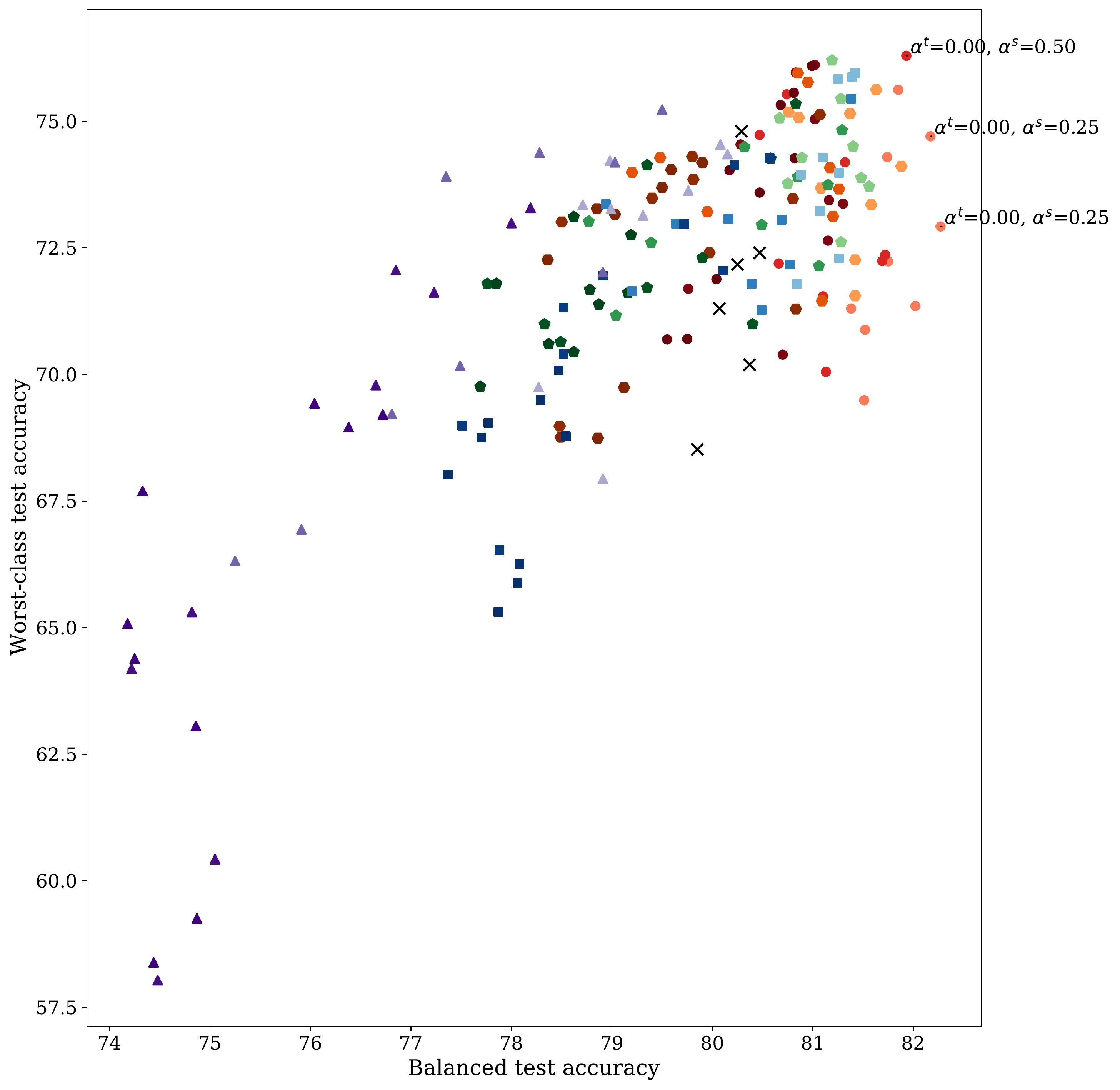} & \includegraphics[width=0.4\textwidth]{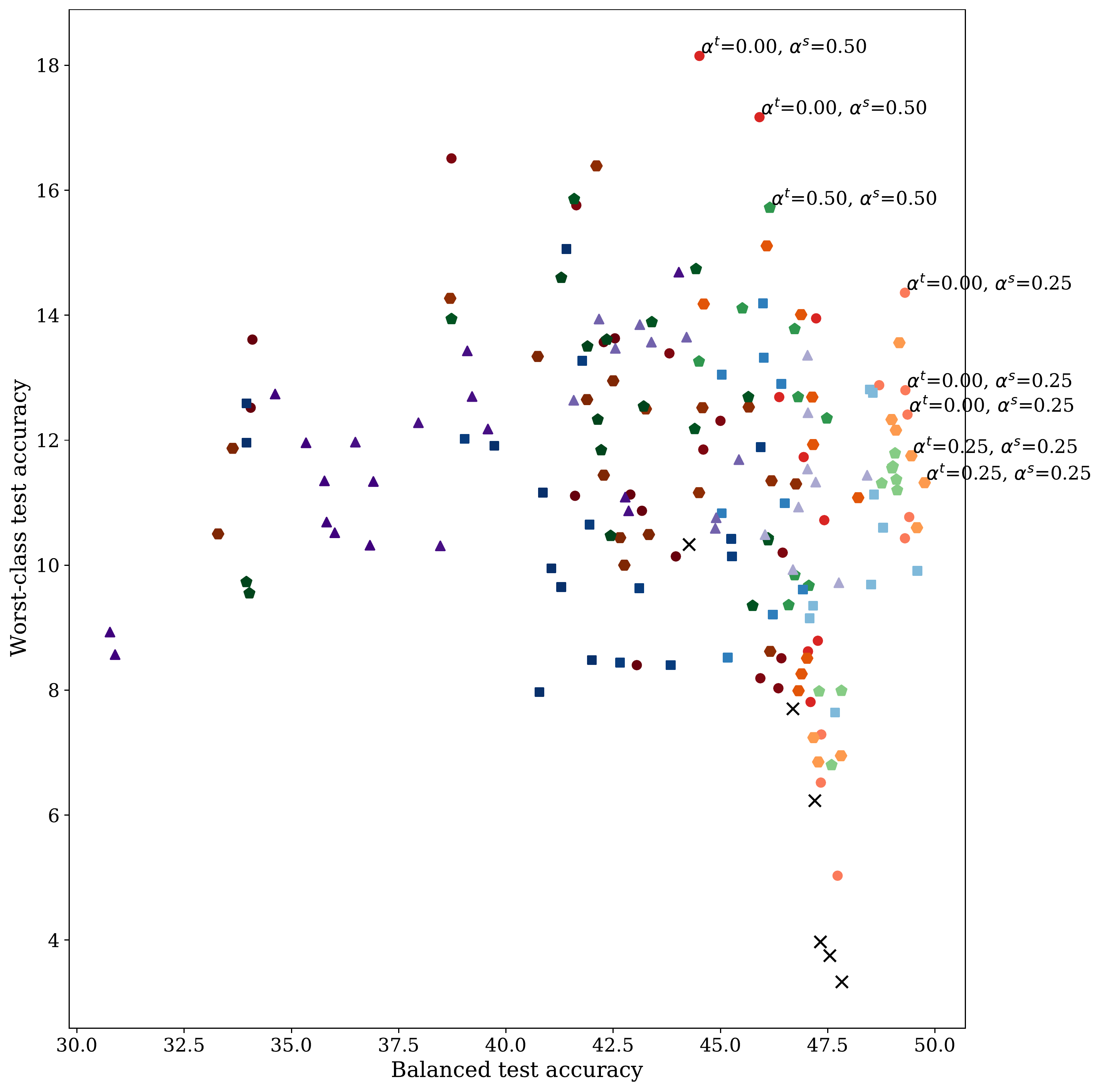}
\end{tabular}
\addtolength{\tabcolsep}{5pt} 
\caption{Tradeoffs in worst-class test accuracy vs. average test accuracy for CIFAR-10-LT and CIFAR-100-LT, with students trained using architecture ResNet-56 and \textbf{teacher} validation labels. Mean test performance over 10 runs are shown, and all temperatures are included.}
\label{fig:tradeoffs_cf-lt_56_teacher}
\end{center}
\vskip -0.2in
\end{figure*}

\begin{figure*}[!ht]
\begin{center}
\addtolength{\tabcolsep}{-5pt} 
\begin{tabular}{ccc}
    & \tiny{CIFAR-10-LT (ResNet-56 $\to$ ResNet-32)} & \tiny{CIFAR-100-LT (ResNet-56 $\to$ ResNet-32)} \\
    \includegraphics[trim={2.8cm 0 2.8cm 0},clip,width=0.15\textwidth]{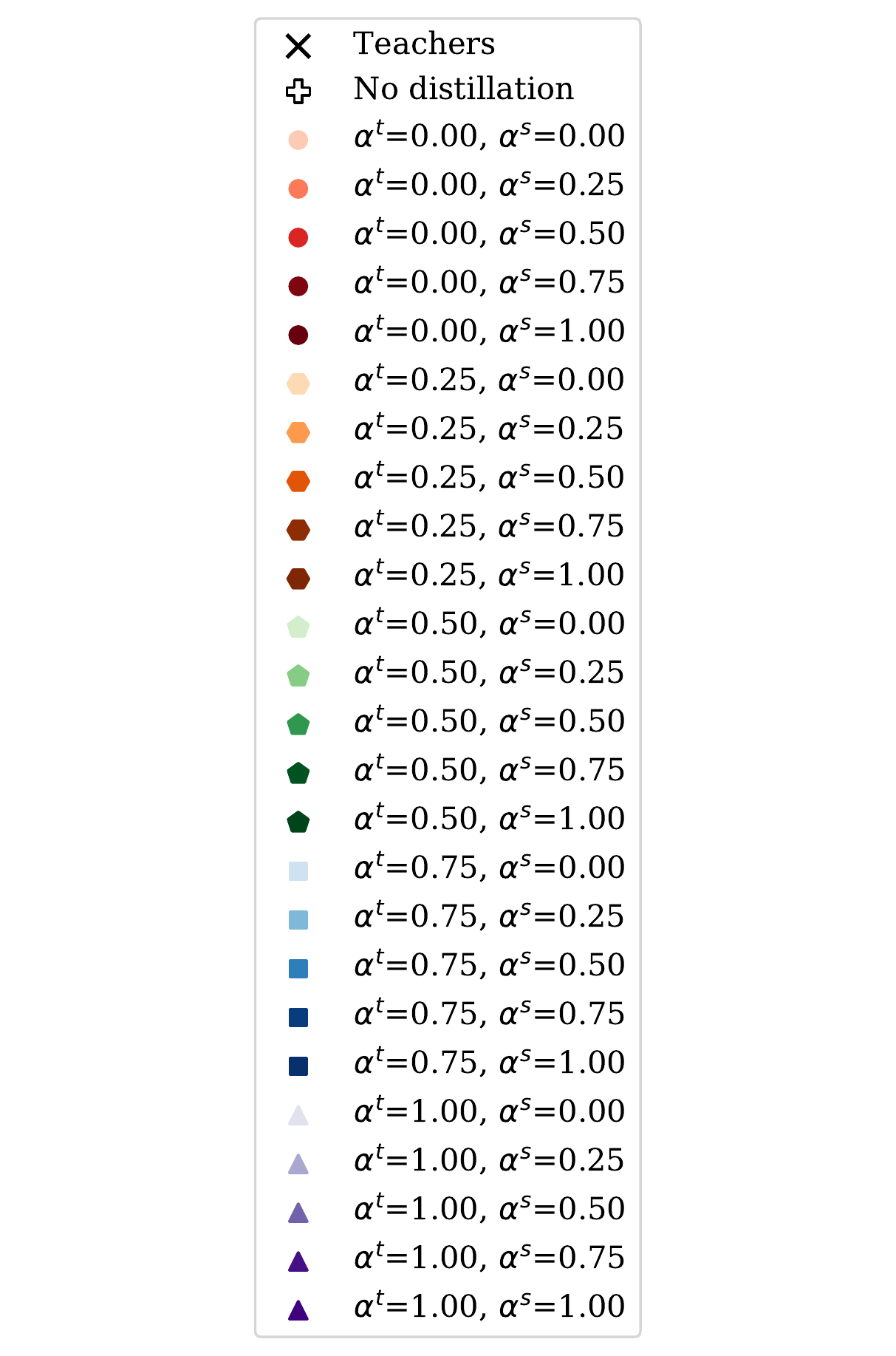} & \includegraphics[width=0.4\textwidth]{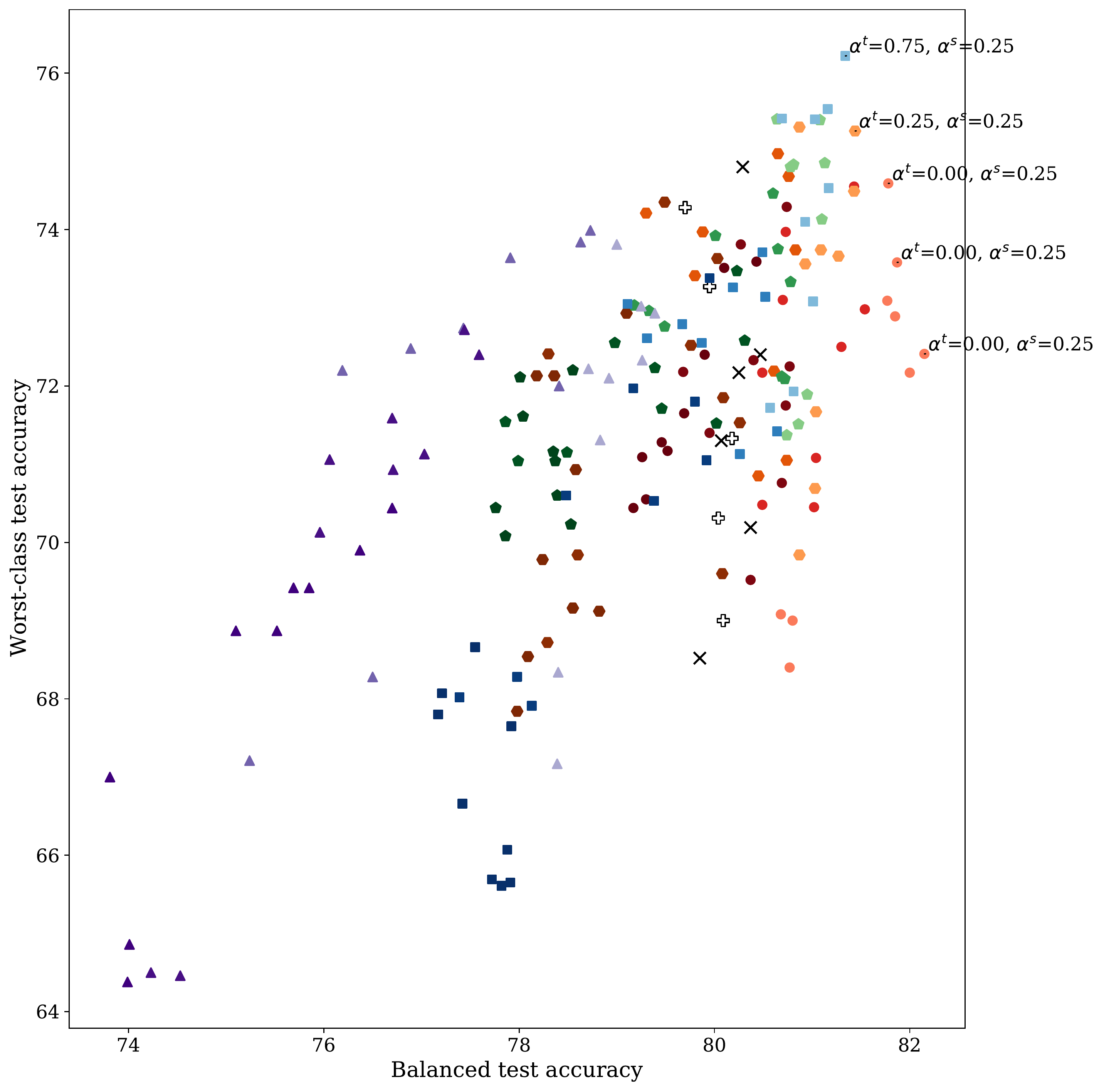} & \includegraphics[width=0.4\textwidth]{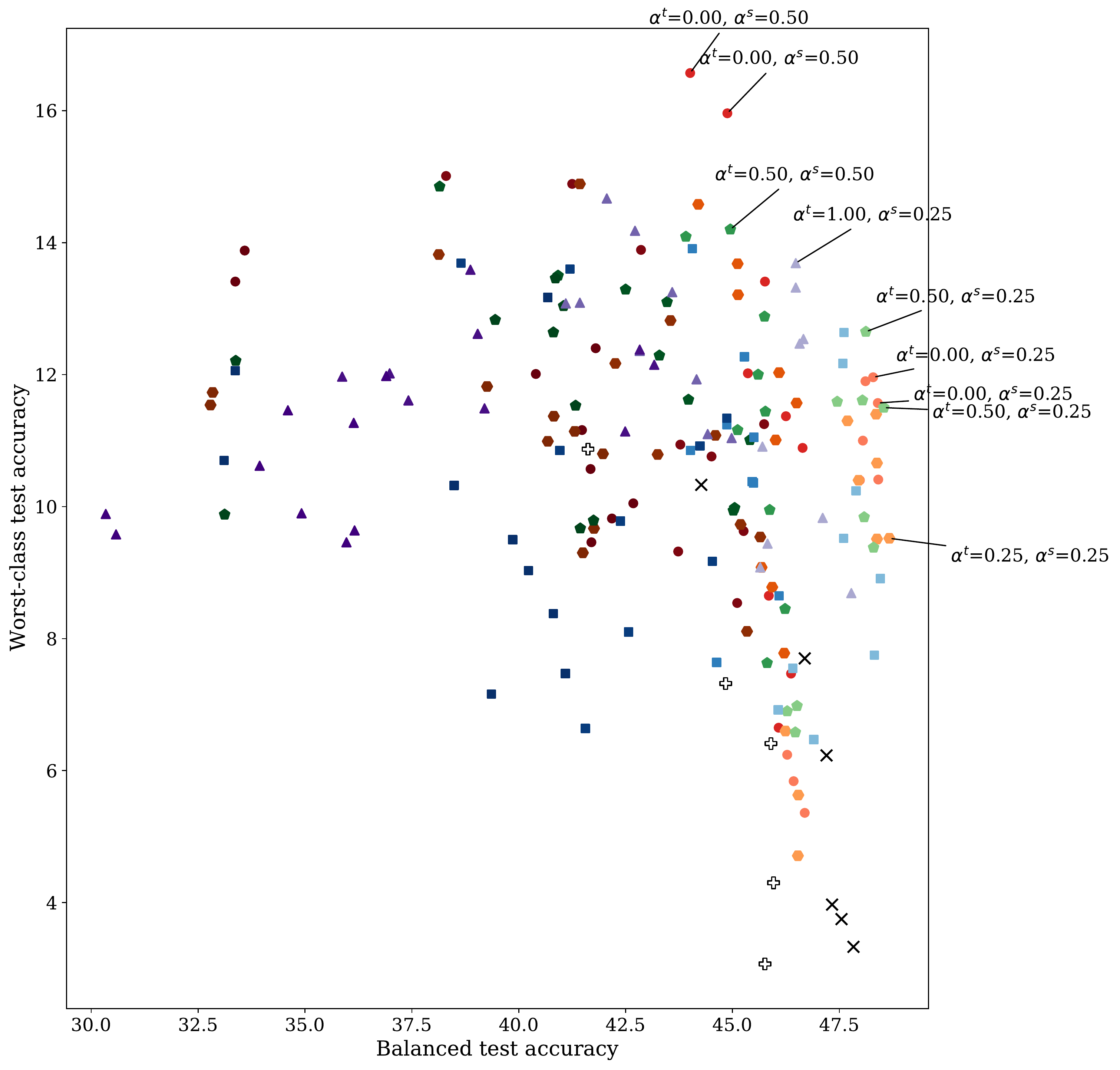}
\end{tabular}
\addtolength{\tabcolsep}{5pt} 
\caption{Tradeoffs in worst-class test accuracy vs. average test accuracy for CIFAR-10-LT and CIFAR-100-LT, with students trained using architecture ResNet-32 and \textbf{teacher} validation labels. Mean test performance over 10 runs are shown, and all temperatures are included.}
\label{fig:tradeoffs_cf-lt_32_teacher}
\end{center}
\vskip -0.2in
\end{figure*}

\begin{figure*}[!ht]
\begin{center}
\addtolength{\tabcolsep}{-5pt} 
\begin{tabular}{ccc}
    & \tiny{CIFAR-10 (ResNet-56 $\to$ ResNet-56)} & \tiny{CIFAR-100 (ResNet-56 $\to$ ResNet-56)} \\
    \includegraphics[trim={2.8cm 0 2.8cm 0},clip,width=0.15\textwidth]{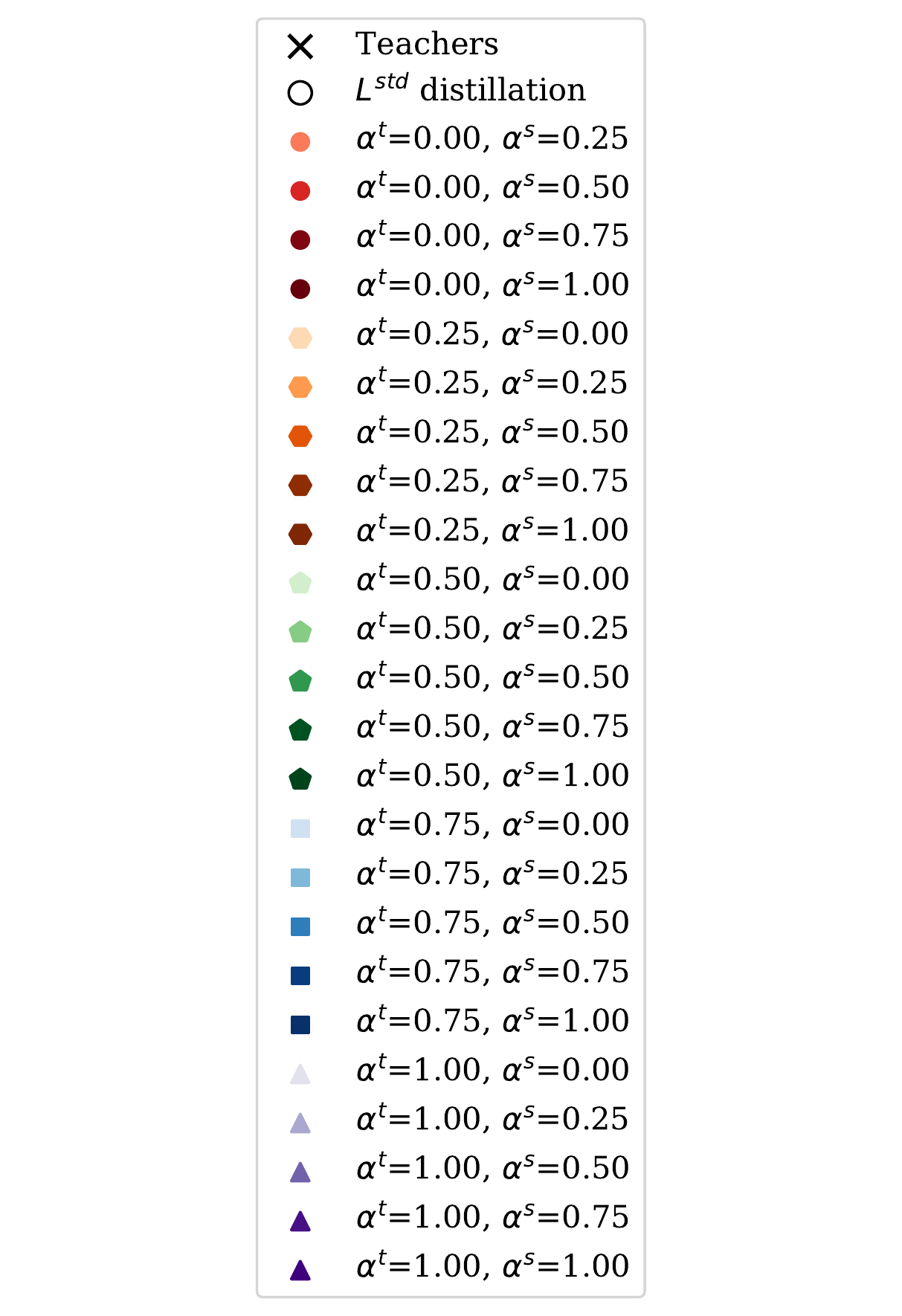} & \includegraphics[width=0.4\textwidth]{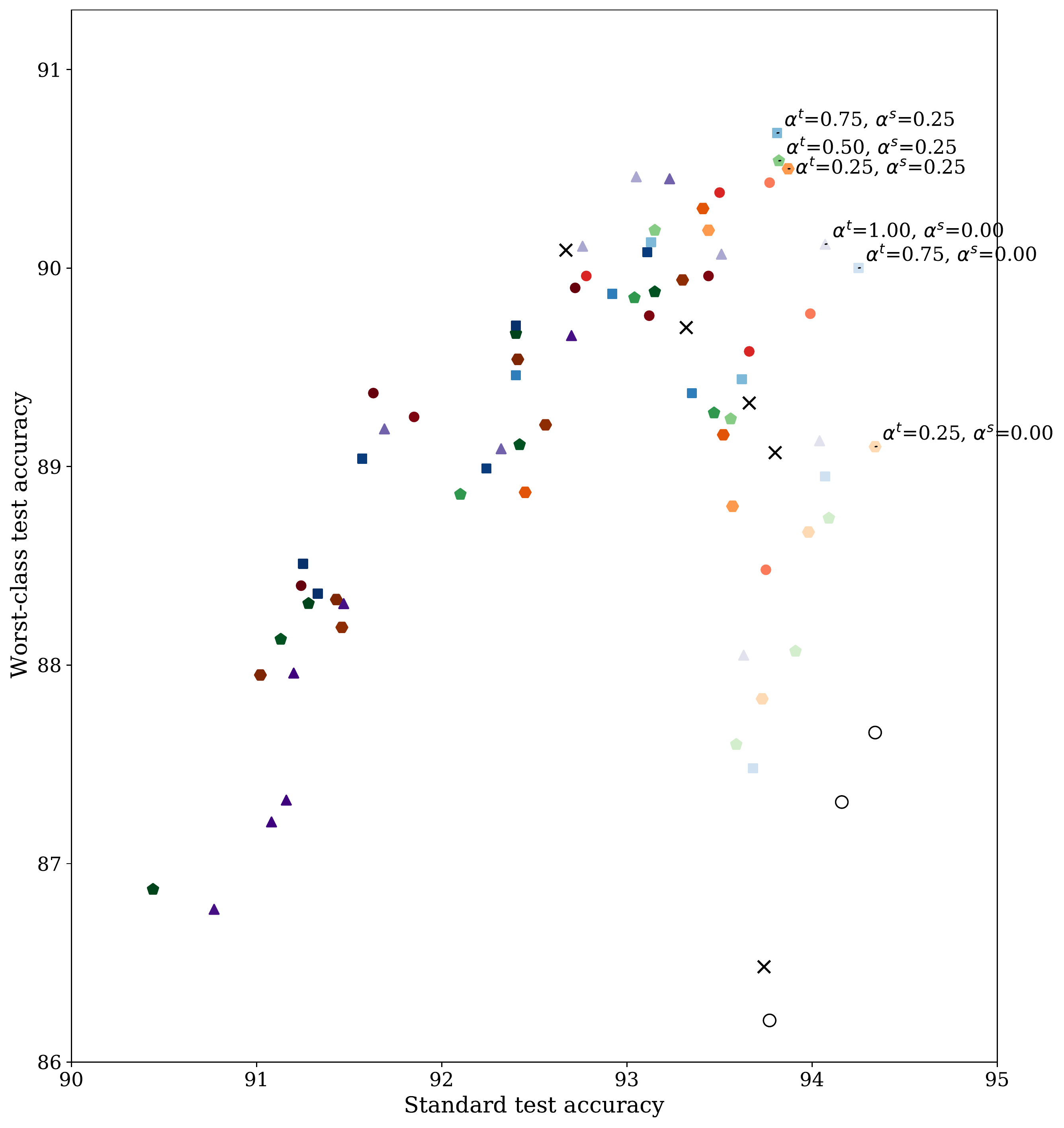} & \includegraphics[width=0.4\textwidth]{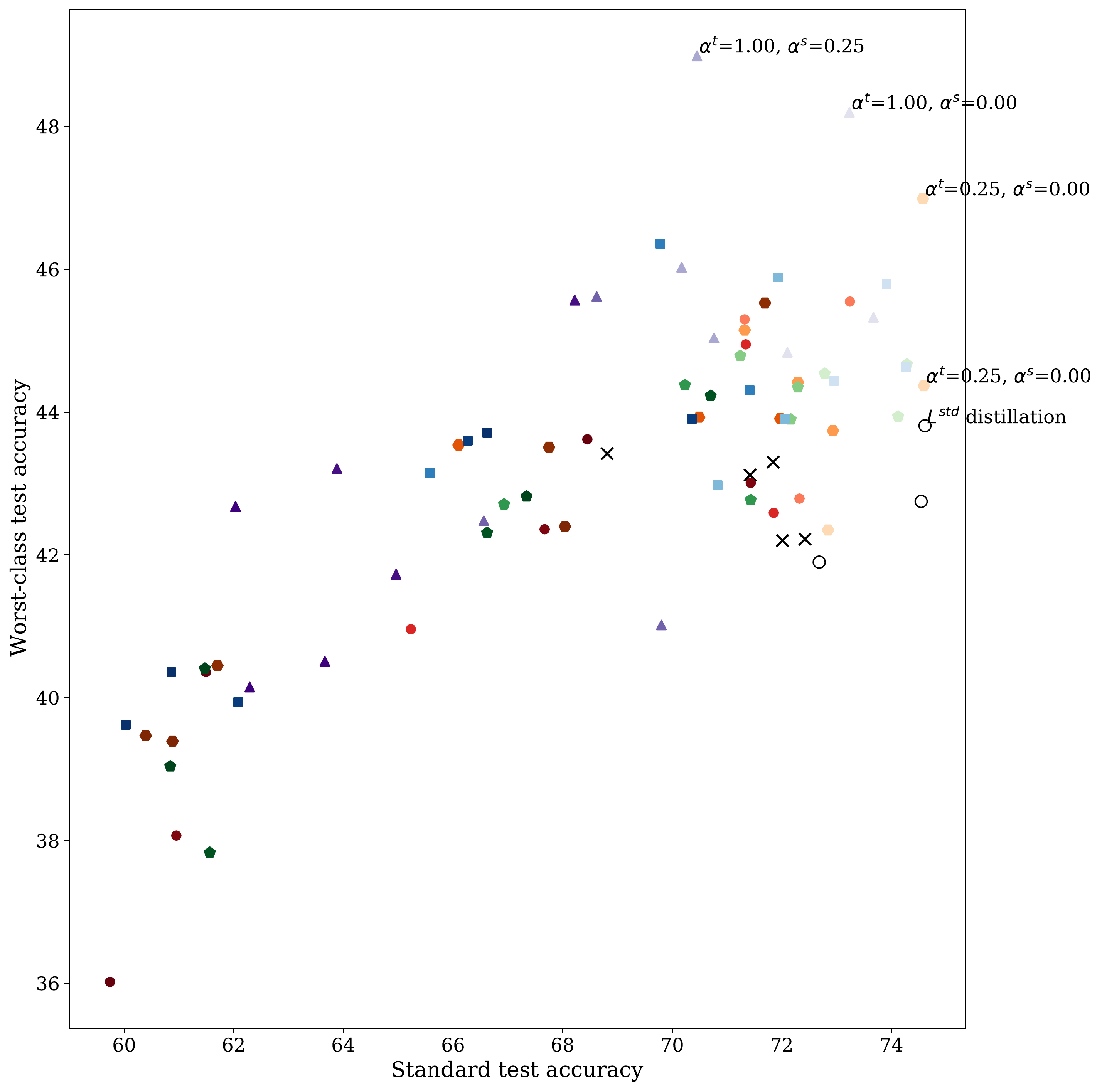}
\end{tabular}
\addtolength{\tabcolsep}{5pt} 
\caption{Tradeoffs in worst-class test accuracy vs. average test accuracy for CIFAR-10 and CIFAR-100, with students trained using architecture ResNet-56 and \textbf{one-hot} validation labels. Mean test performance over 10 runs are shown, and all temperatures are included.}
\label{fig:tradeoffs_cf_56_onehot}
\end{center}
\vskip -0.2in
\end{figure*}

\begin{figure*}[!ht]
\begin{center}
\addtolength{\tabcolsep}{-5pt} 
\begin{tabular}{ccc}
    & \tiny{CIFAR-10 (ResNet-56 $\to$ ResNet-56)} & \tiny{CIFAR-100 (ResNet-56 $\to$ ResNet-56)} \\
    \includegraphics[trim={2.8cm 0 2.8cm 0},clip,width=0.15\textwidth]{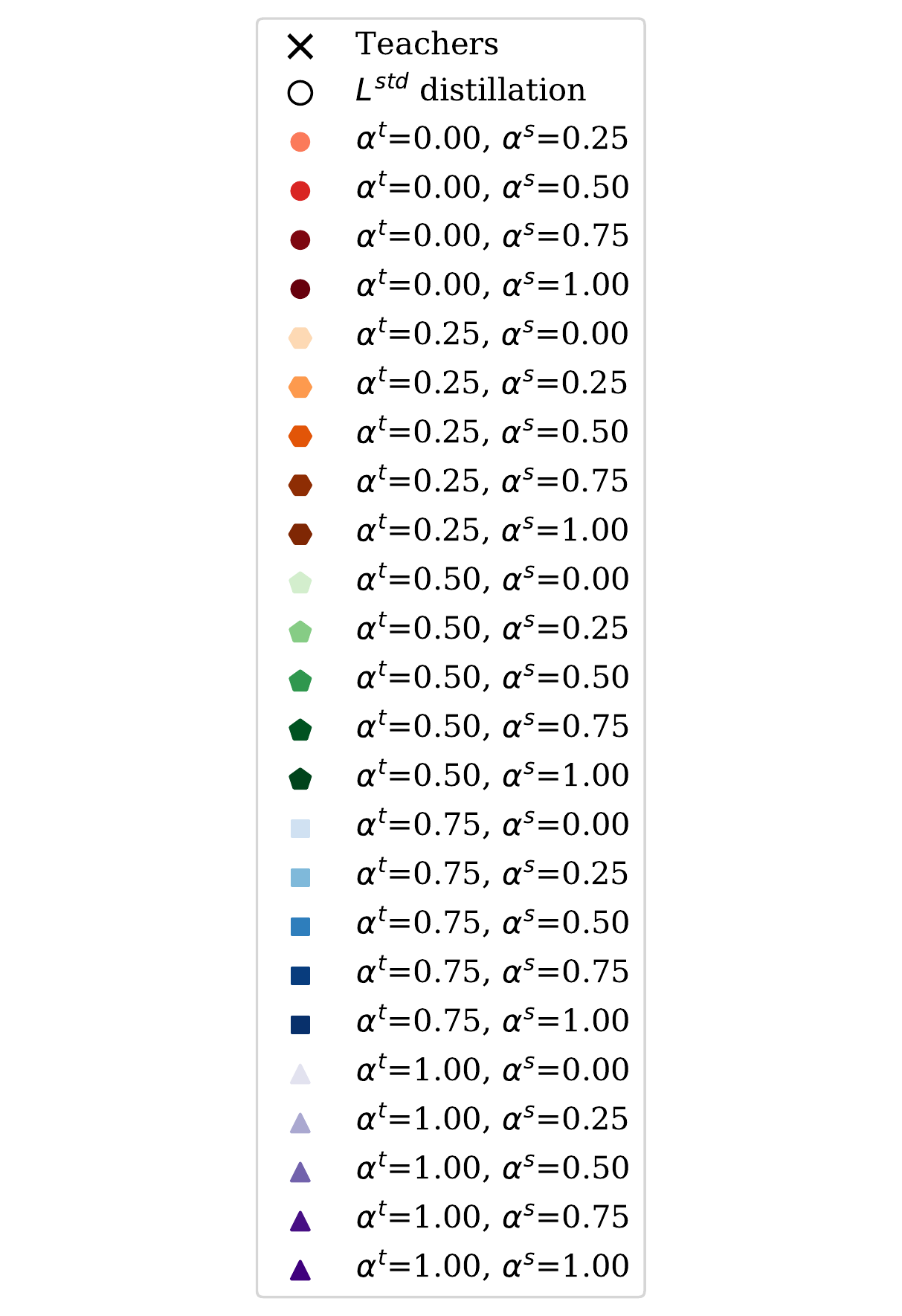} & \includegraphics[width=0.4\textwidth]{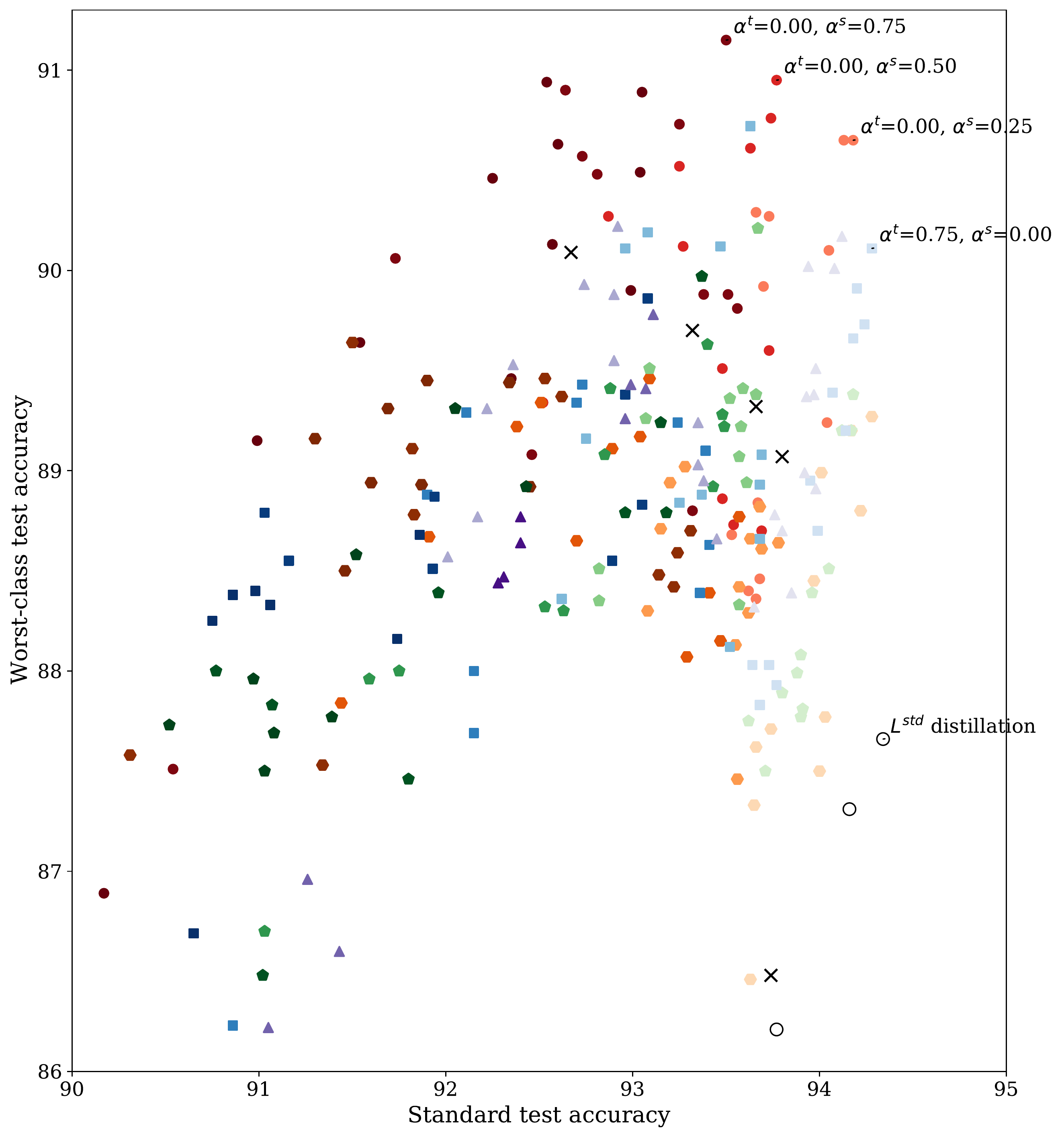} & \includegraphics[width=0.4\textwidth]{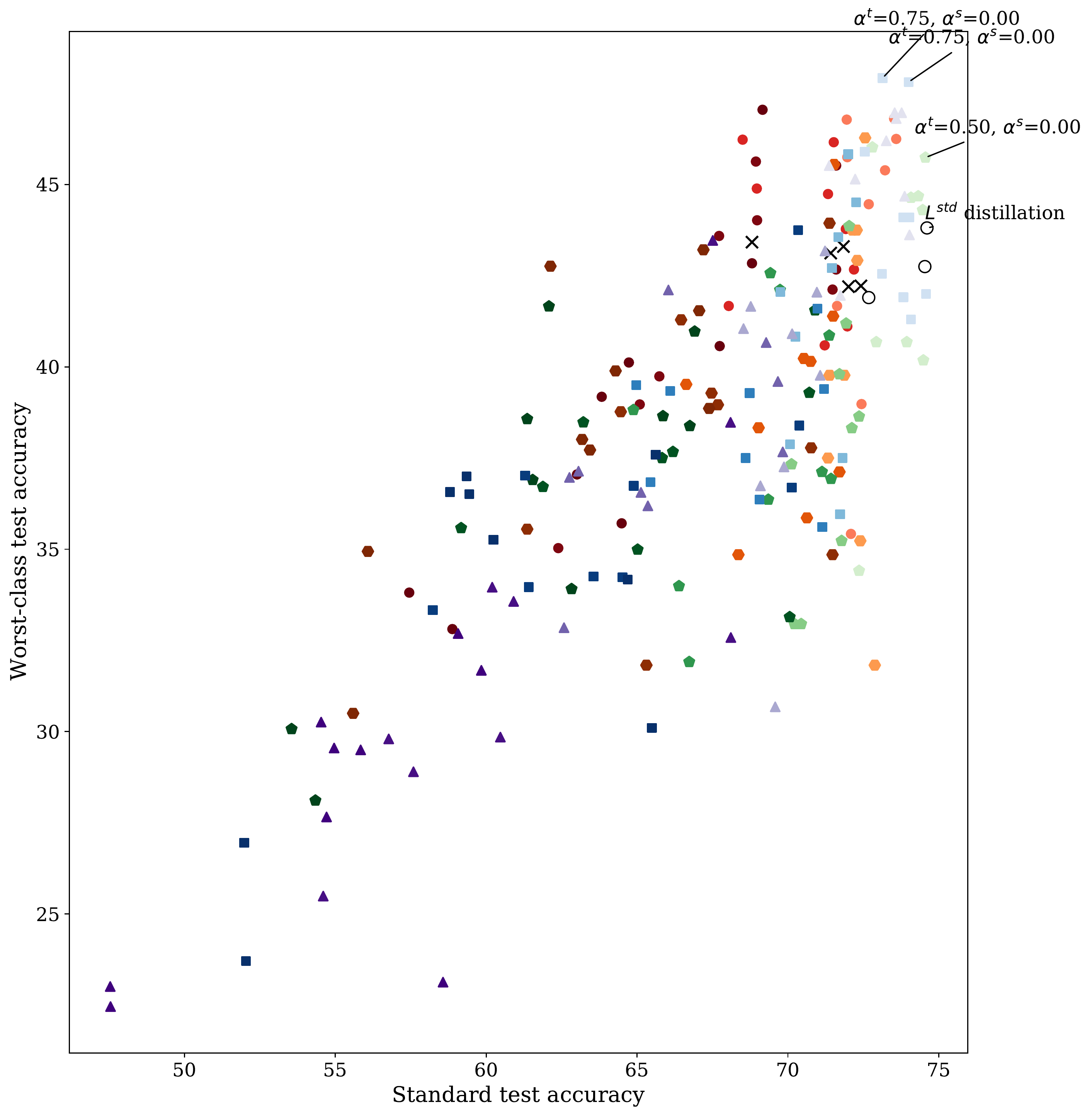}
\end{tabular}
\addtolength{\tabcolsep}{5pt} 
\caption{Tradeoffs in worst-class test accuracy vs. average test accuracy for CIFAR-10 and CIFAR-100, with students trained using architecture ResNet-56 and \textbf{teacher} validation labels. Mean test performance over 10 runs are shown, and all temperatures are included.}
\label{fig:tradeoffs_cf_56_teacher}
\end{center}
\vskip -0.2in
\end{figure*}

\begin{figure*}[!ht]
\begin{center}
\addtolength{\tabcolsep}{-5pt} 
\begin{tabular}{ccc}
    & \tiny{CIFAR-10 (ResNet-56 $\to$ ResNet-32)} & \tiny{CIFAR-100 (ResNet-56 $\to$ ResNet-32)} \\
    \includegraphics[trim={2.8cm 0 2.8cm 0},clip,width=0.15\textwidth]{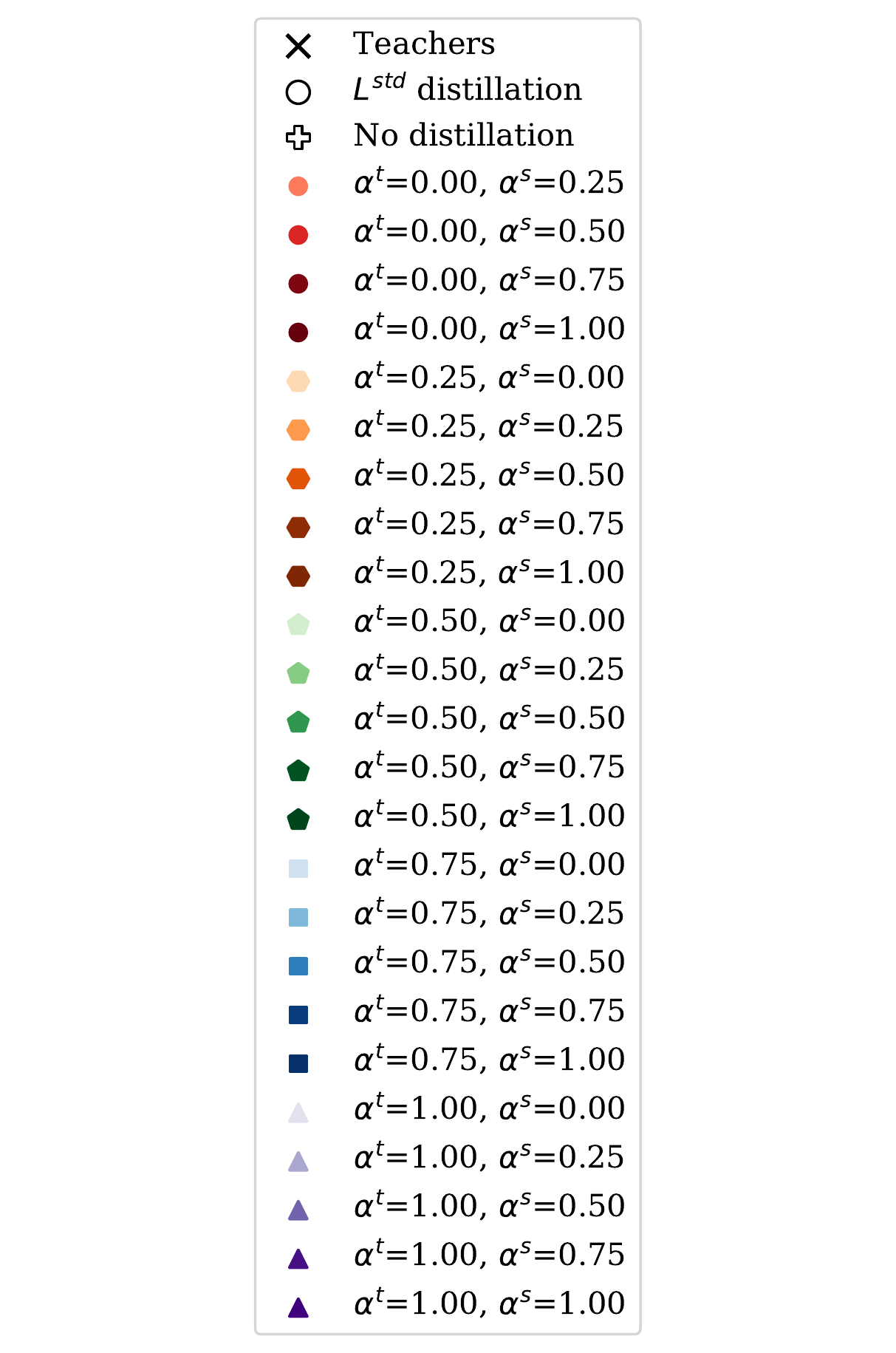} & \includegraphics[width=0.4\textwidth]{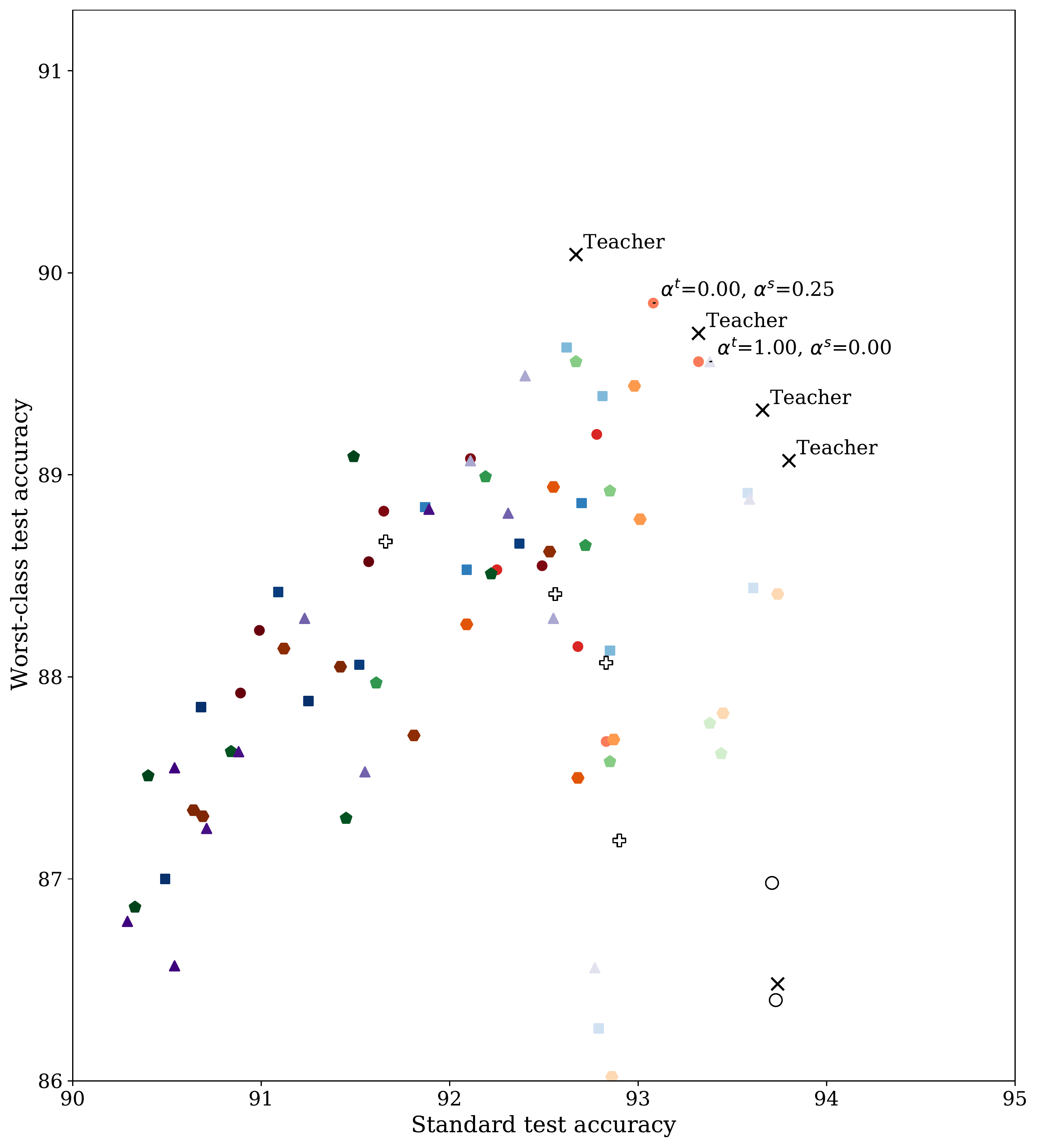} & \includegraphics[width=0.4\textwidth]{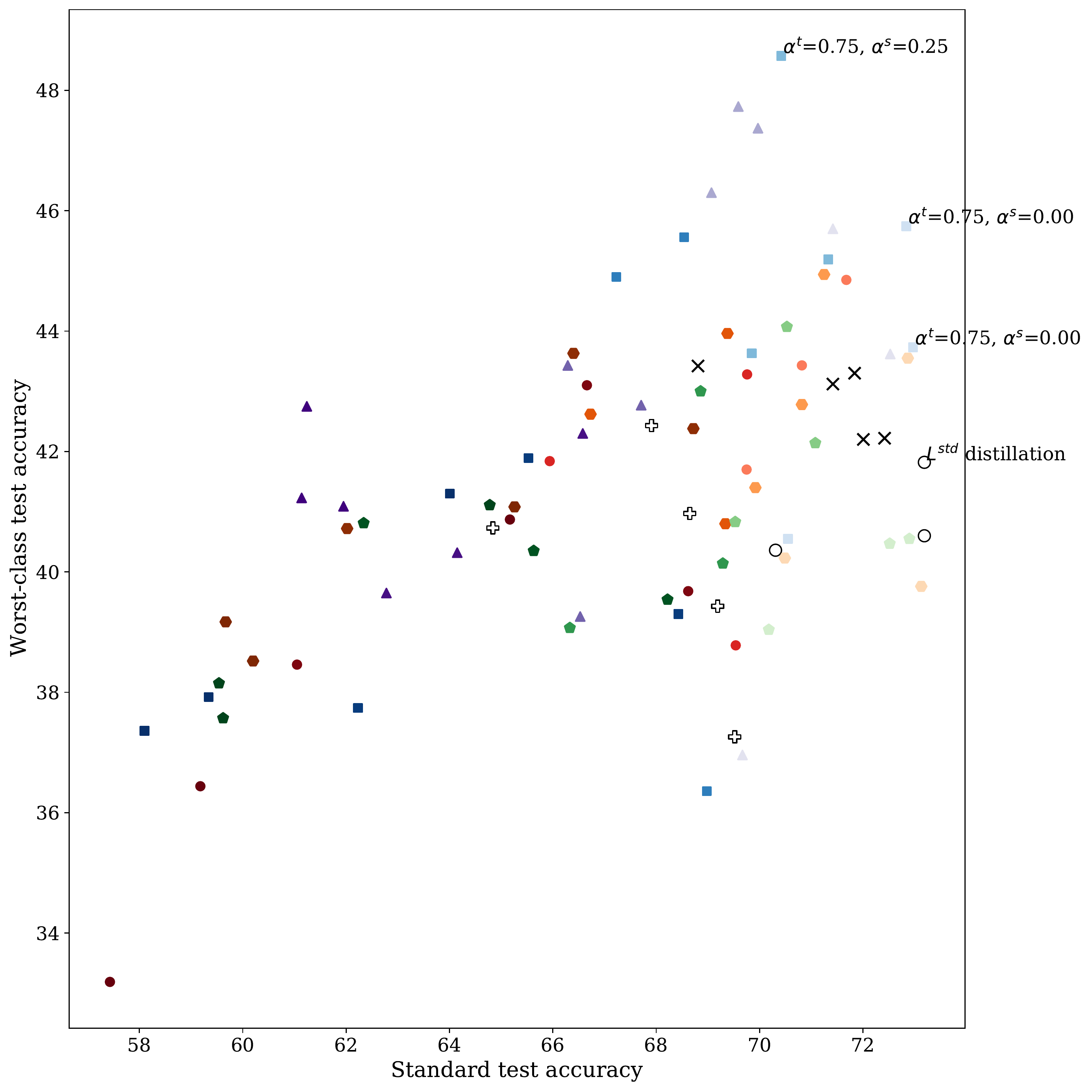}
\end{tabular}
\addtolength{\tabcolsep}{5pt} 
\caption{Tradeoffs in worst-class test accuracy vs. average test accuracy for CIFAR-10 and CIFAR-100, with students trained using architecture ResNet-32 and \textbf{one-hot} validation labels. Mean test performance over 10 runs are shown, and all temperatures are included.}
\label{fig:tradeoffs_cf_32_onehot}
\end{center}
\vskip -0.2in
\end{figure*}

\begin{figure*}[!ht]
\begin{center}
\addtolength{\tabcolsep}{-5pt} 
\begin{tabular}{ccc}
    & \tiny{CIFAR-10 (ResNet-56 $\to$ ResNet-32)} & \tiny{CIFAR-100 (ResNet-56 $\to$ ResNet-32)} \\
    \includegraphics[trim={2.8cm 0 2.8cm 0},clip,width=0.15\textwidth]{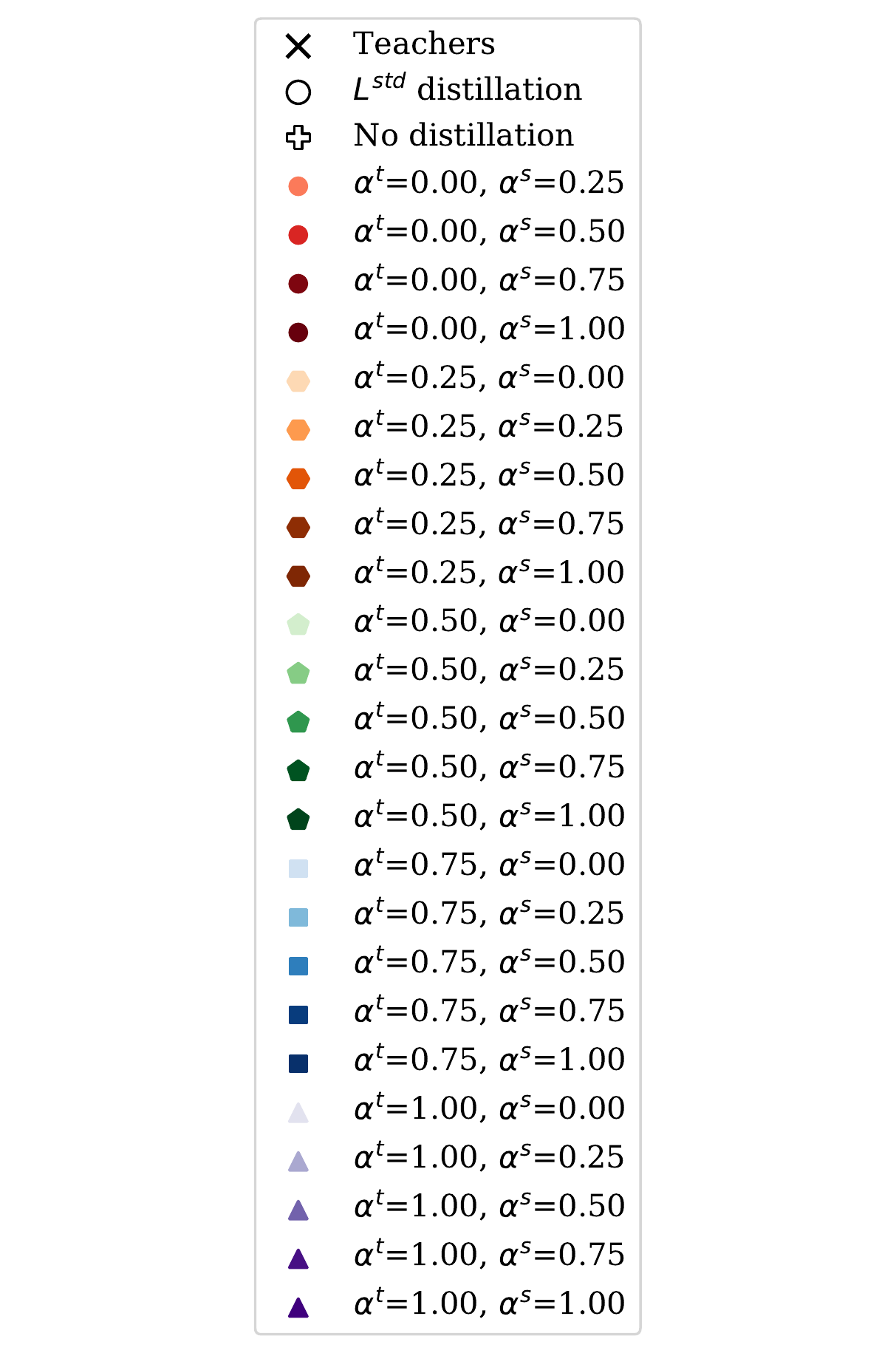} & \includegraphics[width=0.4\textwidth]{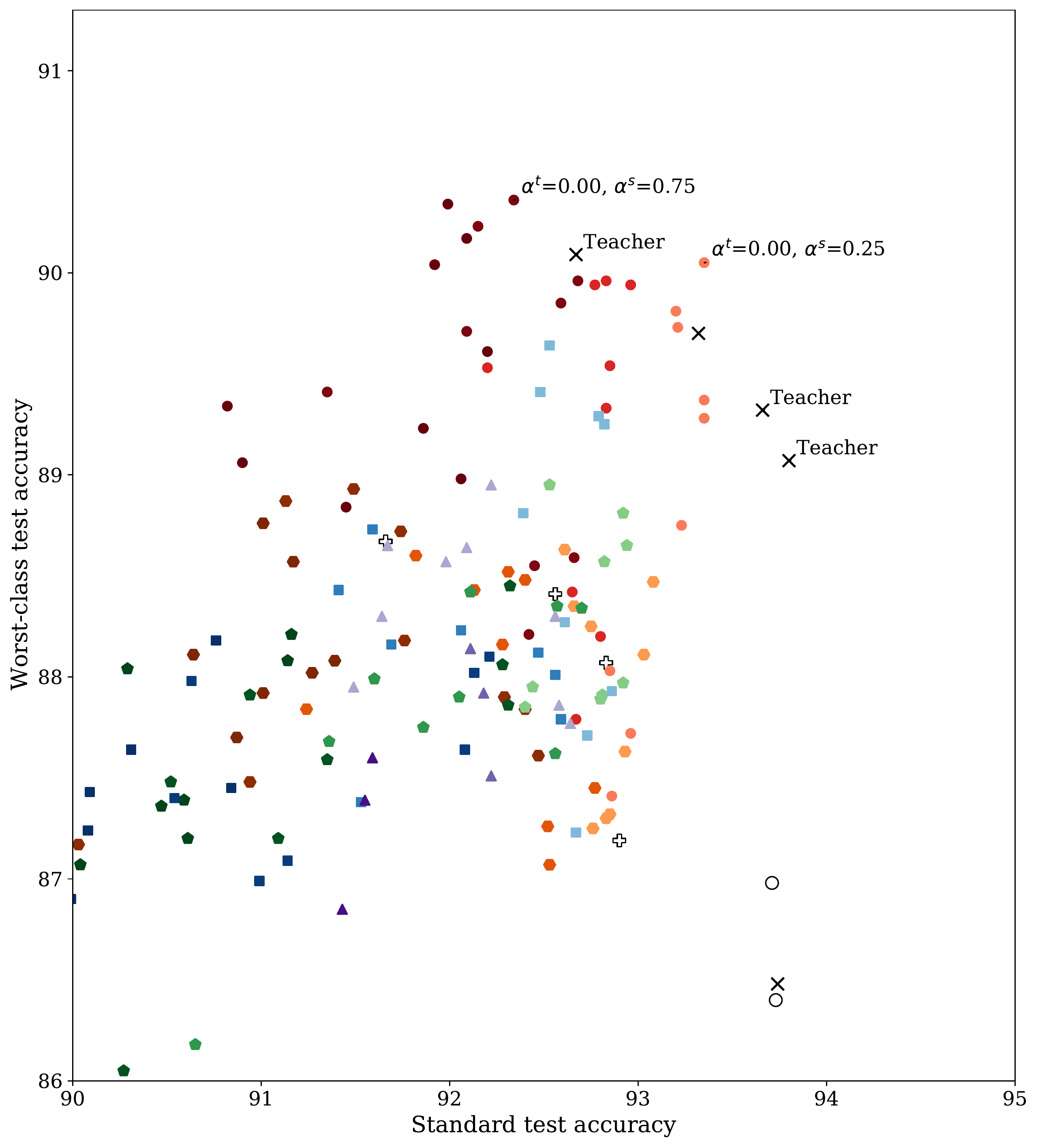} & \includegraphics[width=0.4\textwidth]{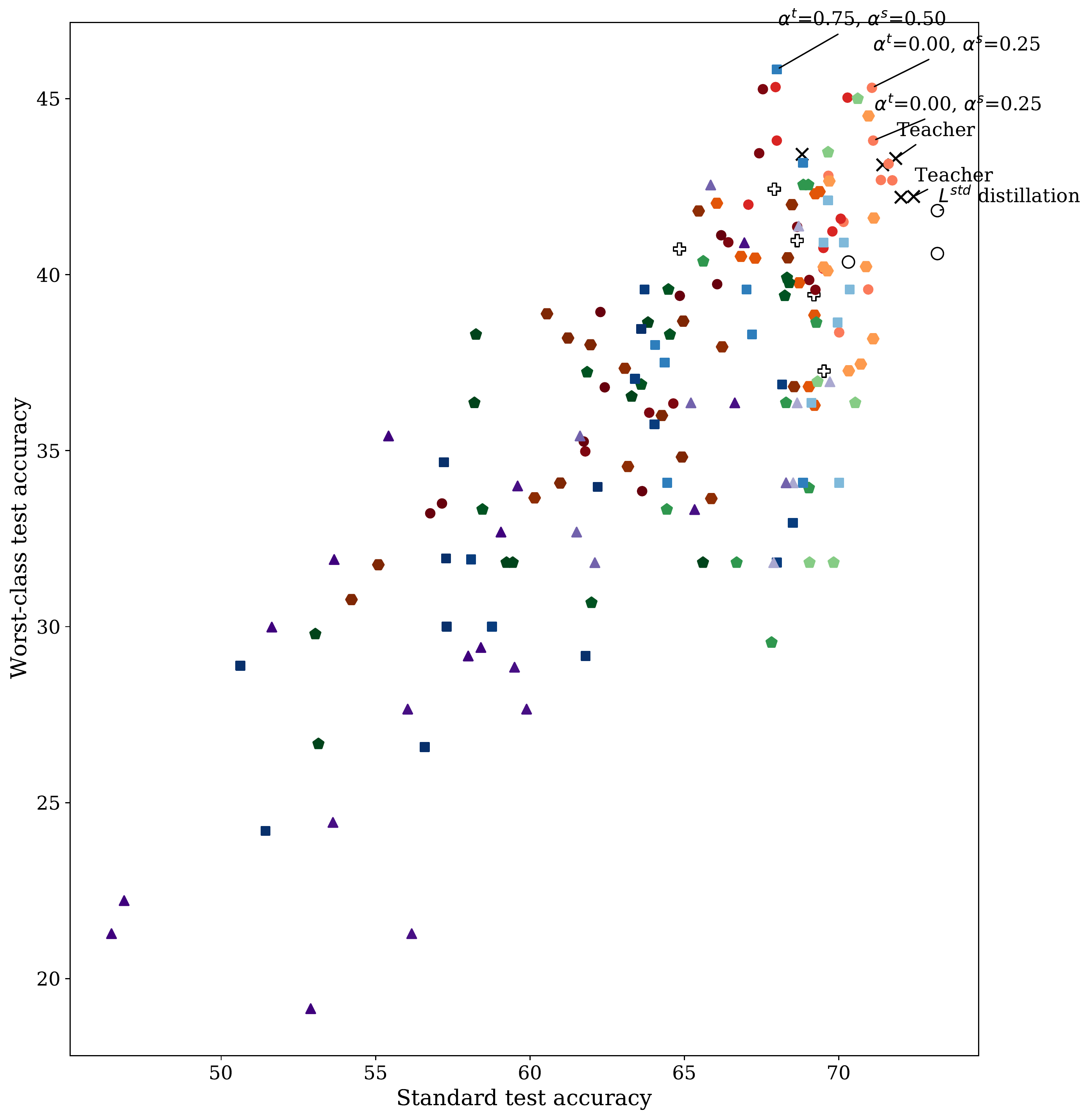}
\end{tabular}
\addtolength{\tabcolsep}{5pt} 
\caption{Tradeoffs in worst-class test accuracy vs. average test accuracy for CIFAR-10 and CIFAR-100, with students trained using architecture ResNet-32 and \textbf{teacher} validation labels. Mean test performance over 10 runs are shown, and all temperatures are included.}
\label{fig:tradeoffs_cf_32_teacher}
\end{center}
\vskip -0.2in
\end{figure*}